\documentclass[twoside]{article}

\usepackage[accepted]{aistats2024}
\usepackage[round,sort]{natbib}

\usepackage{amsthm,amsmath,amssymb,mathtools,bm,bbm,mathrsfs}

\def\eqref#1{(\ref{#1})}
\def\eps{{\epsilon}}

\newcommand{\tr}{\mathrm{tr}}

\def\ra{{\textnormal{a}}}

\def\vzero{{\bm{0}}}

\def\va{{\bm{a}}}
\def\vb{{\bm{b}}}

\def\vp{{\bm{p}}}

\def\vu{{\bm{u}}}
\def\vv{{\bm{v}}}
\def\vw{{\bm{w}}}
\def\vx{{\bm{x}}}
\def\vy{{\bm{y}}}
\def\vz{{\bm{z}}}

\def\mA{{\bm{A}}}
\def\mB{{\bm{B}}}
\def\mC{{\bm{C}}}
\def\mD{{\bm{D}}}
\def\mE{{\bm{E}}}

\def\mG{{\bm{G}}}
\def\mH{{\bm{H}}}
\def\mI{{\bm{I}}}

\def\mL{{\bm{L}}}
\def\mM{{\bm{M}}}

\def\mU{{\bm{U}}}
\def\mV{{\bm{V}}}

\DeclareMathAlphabet{\mathsfit}{\encodingdefault}{\sfdefault}{m}{sl}
\SetMathAlphabet{\mathsfit}{bold}{\encodingdefault}{\sfdefault}{bx}{n}

\def\gA{{\mathcal{A}}}
\def\gB{{\mathcal{B}}}
\def\gC{{\mathcal{C}}}

\def\gF{{\mathcal{F}}}

\def\gH{{\mathcal{H}}}

\def\gL{{\mathcal{L}}}
\def\gM{{\mathcal{M}}}

\def\gO{{\mathcal{O}}}

\def\gQ{{\mathcal{Q}}}

\def\gS{{\mathcal{S}}}

\def\gW{{\mathcal{W}}}
\def\gX{{\mathcal{X}}}

\def\gZ{{\mathcal{Z}}}

\def\sP{{\mathbb{P}}}

\def\sR{{\mathbb{R}}}

\newcommand{\E}{\mathbb{E}}

\newcommand{\KL}{\mathrm{KL}}

\DeclareMathOperator*{\argmax}{arg\,max}
\DeclareMathOperator*{\argmin}{arg\,min}

\theoremstyle{plain}
\newtheorem{assumption}{Assumption}
\newtheorem{definition}{Definition}
\newtheorem{theorem}{Theorem}
\newtheorem{proposition}{Proposition}
\newtheorem{lemma}{Lemma}

\newtheorem{remark}{Remark}
\def\1{\mathbbm{1}}

\newcommand{\diag}{\mathrm{diag}}

\newcommand{\Ber}{\mathrm{Bernoulli}}

\usepackage{graphicx, url}
\usepackage[algo2e,linesnumbered]{algorithm2e}
\usepackage[utf8]{inputenc} % allow utf-8 input
\usepackage[T1]{fontenc}    % use 8-bit T1 fonts
\usepackage{hyperref}       % hyperlinks
\usepackage{pifont}% http://ctan.org/pkg/pifont
\usepackage{subfigure}

\usepackage{multirow}
\usepackage{makecell}
\usepackage{url}            % simple URL typesetting
\usepackage{booktabs}       % professional-quality tables
\usepackage{amsfonts}       % blackboard math symbols
\usepackage{nicefrac}       % compact symbols for 1/2, etc.
\usepackage{microtype}      % microtypography
\usepackage{color,colortbl}         % colors
\usepackage[usenames,dvipsnames]{xcolor}
\usepackage{multicol}
\usepackage{footnote}
\usepackage{dsfont}

\usepackage{thm-restate}
\usepackage{enumitem}
\usepackage{commath}
\usepackage{comment}
\usepackage{tablefootnote}

\definecolor{darkblue}{rgb}{0.0,0.0,0.65}
\definecolor{darkred}{rgb}{0.65,0.0,0.0}
\definecolor{darkgreen}{rgb}{0.0,0.5,0.0}
\definecolor{tab:blue}{RGB}{31,119,180}  % 1f77b4
\definecolor{tab:red}{RGB}{214,39,40}  % d62728
\definecolor{tab:green}{RGB}{44,160,44}  % 2ca02c
\definecolor{tab:orange}{RGB}{255,127,14}  % ff7f0e
\hypersetup{
	colorlinks = true,
	citecolor  = darkblue,
	linkcolor  = darkred,
	filecolor  = darkblue,
	urlcolor   = darkblue,
}

\newcommand{\cmark}{\text{\ding{51}}}
\newcommand{\xmark}{\text{\ding{55}}}
\newcommand{\gyes}{{\color[rgb]{0,.8,0}\cmark}}
\newcommand{\rno}{\color[rgb]{.8,0,0}\xmark}

\newcommand{\Reg}{\mathrm{Reg}}

\newcommand{\vectorize}{\mathrm{vec}}

\SetKwInput{KwInput}{Input}
\SetKwInput{KwOutput}{Output}

\def\RR{{\mathbb{R}}}

\DeclareMathOperator{\EE}{\mathbb{E}} % ensures the left space is proper (e.g., 2\EE[X])

\def\ddefloop#1{\ifx\ddefloop#1\else\ddef{#1}\expandafter\ddefloop\fi}
\def\ddef#1{\expandafter\def\csname c#1\endcsname{\ensuremath{\mathcal{#1}}}}
\ddefloop ABCDEFGHIJKLMNOPQRSTUVWXYZ\ddefloop

\def\ddef#1{\expandafter\def\csname h#1\endcsname{\ensuremath{\hat{#1}}}}
\ddefloop ABCDEFGHIJKLMNOPQRSTUVWXYZabcdefghijklmnopqrsuvwxyz\ddefloop % except for 't'
\def\ddef#1{\expandafter\def\csname hc#1\endcsname{\ensuremath{\hat{\mathcal{#1}}}}}
\ddefloop ABCDEFGHIJKLMNOPQRSTUVWXYZ\ddefloop
\def\ddef#1{\expandafter\def\csname hb#1\endcsname{\ensuremath{\hat{\mathbf{#1}}}}}
\ddefloop ABCDEFGHIJKLMNOPQRSTUVWXYZ\ddefloop % 
\def\ddef#1{\expandafter\def\csname hb#1\endcsname{\ensuremath{\hat{\boldsymbol{#1}}}}}
\ddefloop abcdefghijklmnopqrstuvwxyz\ddefloop % 

\def\bth{{\boldsymbol \theta}}
\def\th{\theta}

\def\lam{\lambda}

\def\eps{\varepsilon}
\def\th{\theta}

\def\cd{\cdot}
\def\la{\langle}
\def\ra{\rangle}

\allowdisplaybreaks

\begin{document}

\runningtitle{Regret-to-Confidence-Set Conversion}


\twocolumn[

\aistatstitle{Improved Regret Bounds of (Multinomial) Logistic Bandits via Regret-to-Confidence-Set Conversion}
% 1. Improved Regret Bounds of (Multinomial) Logistic Bandits via Online Regret-to-Confidence-Set Conversion
% 2. Online Regret-to-Confidence-Set Conversion and Applications to (Multinomial) Logistic Bandits (this is reminiscent of Abbasi-Yadkori et al., AISTATS'12

\aistatsauthor{ Junghyun Lee\footnotemark[1] \And Se-Young Yun\footnotemark[1] \And Kwang-Sung Jun\footnotemark[2] }

\aistatsaddress{ \footnotemark[1]Kim Jaechul Graduate School of AI, KAIST, Seoul, Republic of Korea \\ \footnotemark[2]Department of Computer Science, University of Arizona, Tucson AZ, USA \\ \texttt{\{jh\_lee00, yunseyoung\}@kaist.ac.kr} \quad \texttt{kjun@cs.arizona.edu} }
]

\setlength{\abovedisplayskip}{3pt}
\setlength{\belowdisplayskip}{4pt}
\setlength{\abovedisplayshortskip}{3pt}
\setlength{\belowdisplayshortskip}{4pt}
\textfloatsep=1.4em

\begin{abstract}
%% ver2
Logistic bandit is a ubiquitous framework of modeling users' choices, e.g., click vs. no click for advertisement recommender system. We observe that the prior works overlook or neglect dependencies in $S \geq \lVert \bm\theta_\star \rVert_2$, where $\bm\theta_\star \in \sR^d$ is the unknown parameter vector, which is particularly problematic when $S$ is large, e.g., $S \geq d$. In this work, we improve the dependency on $S$ via a novel approach called {\it regret-to-confidence set conversion (R2CS)}, which allows us to construct a convex confidence set based on only the \textit{existence} of an online learning algorithm with a regret guarantee. Using R2CS, we obtain a strict improvement in the regret bound w.r.t. $S$ in logistic bandits while retaining computational feasibility and the dependence on other factors such as $d$ and $T$.
We apply our new confidence set to the regret analyses of logistic bandits with a new martingale concentration step that circumvents an additional factor of $S$.
We then extend this analysis to multinomial logistic bandits and obtain similar improvements in the regret, showing the efficacy of R2CS.
While we applied R2CS to the (multinomial) logistic model, R2CS is a generic approach for developing confidence sets that can be used for various models, which can be of independent interest.
\end{abstract}

\begin{table*}
    \centering
    \begin{tabular}{|c||c|c|c|} \hline 
         &  {\bf Algorithm} &  {\bf Regret Upper Bound} & {\bf \!\!Tractable?\!\!} \\ \hline 
         \multirow{5}*{\!Logistic Bandits\!} &  \makecell{\texttt{SupLogistic} \\ \citep{jun2021confidence}} & $\sqrt{dT} + d^3 \kappa(T)^2$ & \gyes \\ \cline{2-4}
         &  \makecell{\texttt{OFULog} \\ \citep{abeille2021logistic}} & $d S^{\frac{3}{2}} \sqrt{\frac{T}{\kappa_\star(T)}} + \min\left\{ d^2 S^3 \kappa_\gX(T), R_\gX(T) \right\}$ & \rno \\ \cline{2-4}
         &  \makecell{\texttt{OFULog-r} \\ \citep{abeille2021logistic}} & $d S^{\frac{5}{2}} \sqrt{\frac{T}{\kappa_\star(T)}} + \min\left\{ d^2 S^4 \kappa_\gX(T), R_\gX(T) \right\}$ & \gyes \\ \cline{2-4}
         &  \makecell{\texttt{ada-OFU-ECOLog} \\ \citep{faury2022logistic}} & $dS \sqrt{\frac{T}{\kappa_\star(T)}} + d^2 S^6 \kappa(T)$ & \gyes \\ \cline{2-4}
         &  \makecell{\texttt{OFULog+} \\ ({\bf ours, Section~\ref{sec:logistic-regret}})} & $dS \sqrt{\frac{T}{\kappa_\star(T)}} + \min\left\{ d^2 S^2 \kappa_\gX(T), R_\gX(T) \right\}$ & \gyes  \\ \hline 
         \multirow{4}*{MNL Bandits} & \makecell{\texttt{MNL-UCB} \\ \!\citep{amani2021mnl}\!} & $dK^{\frac{3}{4}} S \sqrt{\kappa(T) T}$ & \rno \\ \cline{2-4}
         &  \makecell{\texttt{Improved MNL-UCB} \\ \!\citep{amani2021mnl}\!} & $dK^{\frac{5}{4}} S^{\frac{3}{2}} \left( \sqrt{T} + dK^{\frac{5}{4}} S \kappa(T) \right)$  & \rno \\ \cline{2-4} 
         &  \makecell{\texttt{MNL-UCB+} \\ ({\bf ours})} & $d \sqrt{K S \kappa(T) T}$  & \gyes \\ \cline{2-4}
         &  \makecell{\texttt{Improved MNL-UCB+} \\ ({\bf ours})} & $d \sqrt{K} S \left( \sqrt{T} + d K^{\frac{3}{2}} \sqrt{S} \kappa(T) \right)$  & \rno \\ \hline
    \end{tabular}
    \caption{\label{tab:results} Comparison of regret upper bounds for contextual logistic and MNL bandits, w.r.t. $\kappa_\star(T)$, $\kappa_\gX(T)$, $\kappa(T)$, $d$, $T$, $K$, and $S$ (see Section~\ref{sec:logistic-setting} and \ref{sec:multinomial} for definitions). For simplicity, we omit logarithmic factors. For logistic bandits, $R_{\gX}(T)$ is an arm-set-dependent term that may be much smaller than $\kappa_\gX(T)$. ``Tractable?'' is considered in the case of a finite arm-set, i.e., when $|\gX| < \infty$.}
\end{table*}
% {\color{red} ``Tractable?'' refers to whether the algorithm can be implemented using only convex optimization oracles, given that $\gX$ is convex.}
% \kj{note that even if $\cX$ is convex, there is no guarantee that it is tractable for our algorithm. The objective function of $\max_x \max_{\th\in\cC} \la x, \th\ra$ is not convex nor concave. Even for linear bandits, the OFUL objective cannot be solved efficiently if the arm set $\cX$ is polytope (I remember ``Stochastic Linear Optimization under Bandit Feedback'' mentions something like that). So, how about we limit it to the case of $|\cX| < \infty$?}

\section{INTRODUCTION}
\label{sec:intro}
% \paragraph{Motivation.}
The bandit problem~\citep{thompson1933bandit,robbins1952bandit} provides a ubiquitous framework to model the exploration-exploitation dilemma, with various variants depending on the application domain.
Out of them, (multinomial) logistic bandits~\citep{filippi2010glm,faury2020logistic,amani2021mnl} has recently received much attention due to its power in modeling binary-valued (discrete-valued) rewards with observed covariates and contexts (respectively).
Their applications are abundant in interactive machine learning tasks including news recommendation~\citep{li2010news,li2012glm} where the rewards are (`click', `no click') or online ad placements where the rewards are one of the multiple outcomes (`click', `show me later', `never show again', `no click'). 
% taken from Amani's paper directly
% how binary rewards are very popular in bandits. mention a few (bandit) applications. now, a predominant treatment for binary rewards is to use the logistic model. e.g., the last layer of neural nets is also a logistic model.

% \kj{maybe we can define the problem right away here (without detailed assumptions)}

\begin{comment}
One popular approach is the optimistic approach (also known as ``optimism in the face of uncertainty''), which selects the next arm with the largest upper confidence bound (UCB) defined either by developing an explicit bonus term to be added to an estimated mean reward~\citep{auer2002ucb} or by constructing a confidence set for the unknown parameter $\bm\theta_\star$~\citep{abbasiyadkori2011linear,dani2008linear}.
Although the algorithms seem to differ, at the core, they both require constructing a tight confidence set $\gC_t(\delta)$ for every $t \geq 1$.
By tight, we mean that the confidence set has as small a radius as possible while satisfying $\bm\theta_\star \in \gC_t(\delta), \ \forall t \geq 1$ with probability at least $1 - \delta$; keeping the radius small is important as it directly affects the final regret bound.
% For logistic bandits, the state-of-the-art radius is $\gO(\sqrt{dS \log t})$ due to \cite{abeille2021logistic}.
\end{comment}

In logistic bandits, at every time step $t$, the learner observes a potentially infinite arm-set $\gX_t \subset \RR^d$ that can vary over time, then plays an action $\vx_t \in \gX_t$.
She then receives a reward $r_t \sim \Ber(\mu(\langle \vx_t, \bm\theta_\star \rangle))$ for some unknown $\bm\theta_\star \in \sR^d$, where $\mu(z) = (1 + e^{-z})^{-1}$ is the logistic function. 
%$\mu : \sR \rightarrow (0, 1)$ is a nonlinear {\it link} function.
%For logistic bandits, we let $\mu(z) = (1 + e^{-z})^{-1}$ be the {\it logistic function}.
%For $s \geq 1$, let $\gF_s := \sigma\left( \left\{ \vx_1, r_1, \cdots, \vx_s, r_s, \vx_{s+1} \right\} \right)$, which constitutes the so-called canonical bandit model; also see Chapter 4.6 of \cite{banditalgorithms}.
The goal of the learner is to maximize the cumulative reward, and the performance is typically measured by the (pseudo-) regret:
\begin{equation}
  \Reg^B(T) := \sum_{t=1}^T \left\{ \mu(\langle \vx_{t,\star}, \bm\theta_\star \rangle) - \mu(\langle \vx_t, \bm\theta_\star \rangle) \right\},
\end{equation}
where $\vx_{t,\star} := \argmax_{\vx \in \gX_t} \mu(\langle \vx, \bm\theta_\star \rangle)$ is the optimal action at time $t$.
The multinomial problem is defined in Section~\ref{sec:multinomial}. % for the multinomial version.

One popular bandit strategy is the optimistic approach (also known as ``optimism in the face of uncertainty''), which selects the next arm with the largest upper confidence bound (UCB).
In generalized linear models, the UCB of an arm $\bm x\in\RR^d$ is typically constructed by constructing a confidence set $\gC_t$ for the unknown parameter $\bm\theta_\star$ and then computing $\max_{\bm\th \in \gC_t} \la \bm x, \bm\th\ra$~\citep{dani2008linear,abbasiyadkori2011linear,faury2022logistic}.
For this, it is important to ensure that $\gC_t$ is a convex set since otherwise the maximization above is computationally intractable in general, and one often needs to resort to using a significantly loosened UCB (e.g.,~\citet{faury2020logistic}), which hurts the performance.

One way to construct a convex confidence set is to leverage the loss function, which first appeared in~\citet{abeille2021logistic}: % have proposed a loss-based confidence set that takes the form of
\begin{align*}
    \gC_t = \left\{ \bm \th : \|\bth\|_2 \le S, \bar\gL_t(\bth) - \bar\gL_t(\widehat\bth_t) \le \beta_t(\delta)^2 \right\}
\end{align*}
where $\bar\gL_t$ is the regularized negative log-likelihood, $\widehat\bth_t$ is the regularized MLE at time $t$, and $\beta_t$ is slowly growing in $t$.
This set $\gC_t$ is convex due to the convexity of $\bar\gL_t$.
Such a confidence set is natural as it is based on the log-likelihood ratio and leads to the state-of-the-art regret bound and numerical performance~\citep{abeille2021logistic,faury2022logistic}.
However, the tightness of the set above, specifically the radius $\beta_t(\delta)^2 = \gO(dS^3\log(t))$, is not clear, which is important given that the tightness directly affects the performance of the algorithm, both in the analysis and the numerical performance.
% \kj{actually, isn't theirs (i.e., OFULog-r) worse than this? please double check}

\paragraph{Contributions.}
In this paper, we make a number of contributions in (multinomial) logistic bandits that are enabled by a tightened loss-based confidence set.

Firstly, we propose a novel and generic confidence set construction method that we call regret-to-confidence-set conversion (R2CS).
Specifically, R2CS constructs a loss-based confidence set via an achievable regret bound in the online learning problem with the matching loss \textit{without} ever having to run the online algorithm.
Using R2CS, we provide new confidence sets for logistic loss (Theorem~\ref{thm:confidence-logistic}) and MNL loss (Theorem~\ref{thm:confidence-multinomial}) that are tighter than prior arts~\citep{abeille2021logistic,amani2021mnl,zhang2023mnl}.
Specifically, for the logistic model, our radius is $\beta_t^2 = \gO(d \log(t) + S)$ which is a significant improvement upon $\gO(dS\log(t))$ from~\citet{abeille2021logistic} when $S$ is large. 

R2CS depends on regret bounds of online learning algorithms just like similar approaches of online-to-confidence-set conversion (O2CS; \cite{abbasiyadkori2012conversion}) or online Newton step-based confidence set~\citep{Dekel10robust}.
However, R2CS is fundamentally different from them as R2CS does \textit{not} run the online learning algorithm itself, which allows us to leverage the tight regret guarantees that are currently only available via computationally intractable algorithms~\citep{foster2018logistic,mayo22scale}; see Appendix~\ref{app:o2cs} for a detailed comparison.

Secondly, we obtained improved regret bounds of contextual (multinomial) logistic bandits with our new confidence sets\footnote{After our initial submission, we became aware of a concurrent work of \cite{zhang2023mnl} that tackles the same problem. We discuss how our results compare with theirs in Section~\ref{subsec:multinomial-regret}.} as outlined in Table~\ref{tab:results}.
For logistic bandits, we improve by $\sqrt{S}$ in the leading term and $S$ for lower-order term compared to \cite{abeille2021logistic}, and we improve by $S^4$ and possibly $\kappa$ in the lower-order term compared to \cite{faury2022logistic}. 
% \kj{I would think $\kappa$ improvement can easily be made in Ada-OFU-ECOLog..?}
For MNL bandits, we improve by at least $\sqrt{S}$ and $\sqrt{K}$ for the leading terms compared to \cite{amani2021mnl,zhang2023mnl}.
This is discussed in detail in the last paragraphs of Section~\ref{subsec:logistic-regret} and \ref{subsec:multinomial-regret}.

\paragraph{Outline.}
Section~\ref{sec:logistic-setting} provides the preliminaries of logistic bandits.
Section~\ref{sec:logistic-confidence} describes in detail the core ideas of R2CS for logistic bandits, and based on the new confidence set, Section~\ref{sec:logistic-regret} discusses the resulting improved regret bound of logistic bandits.
Lastly, in Section~\ref{sec:multinomial}, we address how R2CS's applicability extends to multinomial logistic bandits.

% \begin{enumerate}
%     \item Using the well-known interpretation of KL as first-order Taylor approximation~\citep{brekelmans2020bregman,nielsen2020infogeom,infogeom}, expand the loss
    
%     \item Use self-concordant analysis to bound the loss difference by the sum of online learning regret bound, sum of KL's, and sum of second-order remainder terms, usually dependent on some martingale noise
    
%     \item Use appropriate concentration inequalities and already known (state-of-the-art) online learning regret bound to obtain a high-probability bound on the loss difference, which constitutes our new confidence set

%     \item Follow through the OFUL-type proof (existing or new) to obtain the bandit regret bound
% \end{enumerate}

\paragraph{Notations.}
$A \lesssim B$ is when we have $A \leq c B$ for some {\it universal} constant $c$ independent of any quantities we explicitly mention, up to any logarithmic factors.
For an integer $n$, let $[n] := \{1, 2, \cdots, n\}$.
$\Delta_{>0}^K$ is the interior of $(K-1)$-dimensional probability simplex.
$\gB^d(S)$ is the Euclidean $d$-ball of radius $S$, and $\gB^{K \times d}(S)$ is the ball of radius $S$ in $\sR^{K \times d}$ endowed with the Frobenius metric.
For a square matrices $\mA$ and $\mB$, $\lambda_{\min}(\mA)$ and $\lambda_{\max}(\mA)$ is the minimum and maximum eigenvalue of $\mA$, respectively.
Also, we define the Loewner ordering $\succeq$ as $\mA \succeq \mB$ if $\mA - \mB$ is positive semi-definite.
% For $\bm\mu_1, \bm\mu_2 \in \sR^K$ satisfying $\lVert \bm\mu_i \rVert_1 \leq 1$, denote $\KL(\bm\mu_1, \bm\mu_2)$ to be the KL-divergence between the two $(K+1)$-categorical distributions induced by $\bm\mu_1$ and $\bm\mu_2$.
Let $\text{Categorical}(\bm\mu)$ be the $(K+1)$-categorical distribution over $\{0,1,\ldots,K\}$ with $\bm\mu := [\mu_i]_{i\in[K]} \in [0, 1]^{K}$ where $\mu_i \in\RR$ is the mean parameter for category $i \in [K]$ and $\mu_0 = 1 - \sum_i \mu_i$.
%where $\mu_i \in\RR$ is the mean parameter for category $i \in [K]$, $\mu_0 = 1 - \sum_i \mu_i$, and we denote $\bm\mu := [\mu_i]_{i\in[K]} \in \RR^{K}$.
Denote by $\KL(\bm\mu_1, \bm\mu_2)$ the KL-divergence from $\text{Categorical}(\bm\mu_1)$ to $\text{Categorical}(\bm\mu_2)$.
% Denote by $\Reg^B$ and $\Reg^O$ the regret of the bandit and online learning algorithms, respectively.
% \kj{perhaps remove this sentence}
% \kj{how about we define $\Reg^B$ and $\Reg^O$ precisely?}

\section{PROBLEM SETTING}
\label{sec:logistic-setting}

We first consider stochastic contextual logistic bandit setting that proceeds as described in Section~\ref{sec:intro}.
For $s \geq 1$, let $\gF_s := \sigma\left( \left\{ \vx_1, r_1, \cdots, \vx_s, r_s, \vx_{s+1} \right\} \right)$, which constitutes the so-called canonical bandit model; also see Chapter 4.6 of \cite{banditalgorithms}.
\begin{comment}
  We first consider stochastic contextual logistic bandit setting that proceeds as follows:
at every round $t$, the learner observes a potentially infinite arm-set $\gX_t$ that can vary over time, then plays an action $\vx_t \in \gX_t$.
She then receives a reward of $r_t \sim \Ber(\mu(\langle \vx_t, \bm\theta_\star \rangle))$ for some unknown $\bm\theta_\star \in \sR^d$, where $\mu : \sR \rightarrow (0, 1)$ is a nonlinear {\it link} function.
For logistic bandits, we let $\mu(z) = (1 + e^{-z})^{-1}$ be the {\it logistic function}.
For $s \geq 1$, let $\gF_s := \sigma\left( \left\{ \vx_1, r_1, \cdots, \vx_s, r_s, \vx_{s+1} \right\} \right)$, which constitutes the so-called canonical bandit model; also see Chapter 4.6 of \cite{banditalgorithms}.

The bandit (pseudo-)regret for logistic bandits is defined as follows:
\begin{equation}
  \Reg^B(T) := \sum_{t=1}^T \left\{ \mu(\langle \vx_{t,\star}, \bm\theta_\star \rangle) - \mu(\langle \vx_t, \bm\theta_\star \rangle) \right\},
\end{equation}
where $\vx_{t,\star} := \argmax_{\vx \in \gX_t} \mu(\langle \vx, \bm\theta_\star \rangle)$ is the optimal action at time $t$.
\end{comment}

We consider the following standard assumptions~\citep{faury2020logistic}:
\begin{assumption}
\label{assumption:X}
    $\gX_t \subseteq \gB^d(1)$ for all $t \geq 1$.
\end{assumption}
\begin{assumption}
\label{assumption:S}
    $\bm\theta_\star \in \gB^d(S)$ with known $S > 0$.
\end{assumption}
% \junghyun{As we consider the dependency on $S$ explicitly in the regret analyses, we implicitly assume that $S$ does not vanish, i.e., $S \geq 2$ or something like that.}

We define the following problem-dependent quantities:
\begin{align*}
    \kappa_\star(T) \!:=\! \frac{1}{\frac{1}{T} \sum_{t=1}^T \dot{\mu}(\vx_{t,\star}^\intercal \bm\theta_\star)}, \ \kappa_{\gX}(T) \!:=\! \max_{t \in [T]} \max_{\vx \in \gX_t} \frac{1}{\dot{\mu}(\vx^\intercal \bm\theta_\star)},
   \\ \text{and~~} \kappa(T) := \max_{t \in [T]} \max_{\vx \in \gX_t} \max_{\bm\theta \in \gB^d(S)} \frac{1}{\dot{\mu}(\vx^\intercal \bm\theta)}~. \hspace{5em}
\end{align*}
% In the case that $\gX_t = \gX$ for all $t \geq 1$, we simply drop the dependency on $T$.
These quantities can scale exponentially in $S$ in the worst-case~\citep{faury2020logistic}.

% 개인적으로 조금 더 괜찮은 흐름은 지금처럼 뭔가 정의를 새롭게 해주고 내용이 나오는 것 보다는 증명에대한 설명을 하고 해당 정의와 lemma가 필요한 곳에서 그 필요성을 이야기하고 소개하는 것이거든

% section 3 첫 시작도 정의를 하면서 시작을 하는데 그것보다는 우리 방법은 MLE를 바탕으로 confidence set을 정의하는 것이 가장 중요한 부분이다. 이것을 먼저 선언하고 이 confidence set이 앞으로 어떻게 사용 될 것인지 간단하게 언급해주고 MLE의 구체적 적의 confidence set은 그것을 바탕으로 어떻게 정의되는지 적고

% 이 것은 Theorem 3.1과 같은 성격을 가진다라고 소개하고

% notation section에서 정의들을 바로 선언하여 모아놓는 것은 읽는데 도움이 되어 좋지만

% 본문 부분에서는 이런것들이 흐름상으로는 그냥 읽는 흐름을 방해하는 요소인 것 같거든

\section{IMPROVED CONFIDENCE SET}
\label{sec:logistic-confidence}

\paragraph{Overview and Main Theorem.}
Our R2CS approach starts by directly constructing a {\it loss-based} confidence set that contains the true parameter $\bm\theta_\star$ with probability at least $1 - \delta$.
This confidence set is centered around the norm-constrained, unregularized maximum likelihood estimator (MLE), $\widehat{\bm\theta}_t$, defined as
\begin{equation}
\label{eqn:MLE}
    \widehat{\bm\theta}_t := \argmin_{\lVert \bm\theta \rVert_2 \leq S} \sbr{\gL_t(\bm\theta) \triangleq \sum_{s=1}^{t-1} \ell_s(\bm\theta)}~,
\end{equation}
where $\ell_s$ is the logistic loss at time $s$, defined as
\begin{equation*}
    \ell_s(\bm\theta) := -r_s \log \mu(\langle \vx_s, \bm\theta \rangle) - (1 - r_s) \log (1 - \mu(\langle \vx_s, \bm\theta \rangle)).
\end{equation*}

Our loss-based confidence set is then of the form $\gL_t(\bm\theta) - \gL_t(\widehat{\bm\theta}_t) \leq \beta_t(\delta)^2$; note that as $\gL_t$ is convex, so is the resulting confidence set.
Ultimately, we want its radius $\beta_t(\delta)$ to be as small as possible while retaining the high-probability guarantee.
\begin{remark}
    The existence of $\widehat{\bm\theta}_t$ is guaranteed as $\gB^d(S)$ is compact.
    Also, as the domain and the objectives are both convex, one can use standard convex optimization algorithms, e.g., Frank-Wolfe method~\citep{frank1956optim} or interior point method~\citep{cvxbook}, to tractably compute $\widehat{\bm\theta}_t$.
\end{remark}

We now present the first main theorem characterizing our new, improved confidence set:
\begin{theorem}[Improved Confidence Set for Logistic Loss]
\label{thm:confidence-logistic}
    We have
    \begin{align*}
        \sP\left[ \forall t \geq 1, \  \bm\theta_\star \in \gC_t(\delta) \right] \geq 1 - \delta,
    \end{align*}
    where
    \begin{align*}
        \gC_t(\delta) &= \left\{ \bm\theta \in \gB^d(S) : \gL_t(\bm\theta) - \gL_t(\widehat{\bm\theta}_t) \leq \beta_t(\delta)^2 \right\},\\
        \beta_t(\delta) &= \sqrt{10d \log\left( \frac{St}{4d} + e \right) + 2((e - 2) + S) \log\frac{1}{\delta}}.
    \end{align*}
    % \kj{should be $2(e-2+S)\log(1/\delta)$ rather than $2(e-2+S)\log(t/\delta)$?}
\end{theorem}

Roughly speaking, the confidence set of \cite{abeille2021logistic} resulted in the radius of $\beta_t(\delta) = \gO(\sqrt{dS^3 \log t})$, while ours result in $\gO(\sqrt{(d + S) \log t})$.
% \kj{$\leftarrow$doublecheck this}
This separation of $d$ and $S$ leads to an overall improvement in factors of $S$.
Another important observation is that for any $\bm\theta'$, $\gL_t(\bm\theta) - \gL_t(\bm\theta') \leq \gL_t(\bm\theta) - \gL_t(\widehat{\bm\theta}_t) \leq \beta_t(\delta)^2$, and thus, even when one could find only an approximate estimate of $\gL_t(\bm\theta)$, the high-probability guarantee of $\bm\theta_\star \in \gC_t(\delta)$ still holds!
This is in contrast to the prior confidence set~\citep[Section 3.1]{abeille2021logistic}, which is geometrically centered around $\widehat{\bm\theta}_t$ and thus a biased estimate shifts the confidence set, breaking the high-probability guarantee.
% \kj{leave a remark: 
% loss-based confidence set benefit: unlike other methods, one can always find an approximate solution for $\hat\theta$, and this always results in a valid confidence set, just loose. Therefore, while the regret bound depends on the accuracy of the approximation, the practitioner is assured that the confidence set is correct. $\| \theta - \hat\theta\|_{V}^2 \le \beta_t$}

We now present the proof of Theorem~\ref{thm:confidence-logistic}, which is the essence of our R2CS approach.

\paragraph{Proof Sketch of Theorem~\ref{thm:confidence-logistic}.}
The proof has three main technical novelties, which constitute the crux of our R2CS approach and may be of independent interest to other applications.
The first novelty is the two novel decomposition lemmas for the logistic loss (Lemma~\ref{lem:decomposition1-logistic}, \ref{lem:decomposition2-logistic}) that express $\beta_t(\delta)^2$ as the sum of the regret of {\it any} online learning algorithm of our choice, a sum of martingales, and a sum of KL-divergences.
The second novelty is when bounding the sum of martingales, we derive and utilize an anytime variant of the Freedman's inequality for martingales (Lemma~\ref{lem:freedman}).
The third novelty is when bounding the sum of KL-divergences, we combine the self-concordant result of \cite{abeille2021logistic} and the information geometric interpretation of the KL-divergence (Lemma~\ref{lem:kl-bregman}).

We then use the state-of-the-art online logistic regression regret guarantee of \cite{foster2018logistic} to obtain the final confidence set (Theorem~\ref{thm:confidence-logistic}).
To use the result of \cite{foster2018logistic}, we use the norm-constrained, unregularized MLE (Eqn.~\eqref{eqn:MLE}) instead of a regularized MLE used in~\cite{abeille2021logistic}.
We emphasize here that we do not need to explicitly run the online learning algorithm of \cite{foster2018logistic}, which is quite costly; otherwise, we would have to consider its efficient variant~\citep{jezequel2020logistic}, which gives an online regret bound scaling with $S$ that gives us no improvement.

\subsection{Complete Proof of Theorem~\ref{thm:confidence-logistic}}
To use martingale concentrations, we begin by writing
\begin{equation}
\label{eqn:martingale}
    r_s = \mu(\langle \vx_s, \bm\theta_\star \rangle) + \xi_s,
\end{equation}
where $\xi_s$ is a real-valued martingale difference noise.

The following is the first decomposition lemma:
\begin{lemma}
\label{lem:decomposition1-logistic}
    For the logistic loss $\ell_s$, the following holds for any $\bm\theta$:
    \begin{equation*}
        \ell_s(\bm\theta_\star) = \ell_s(\bm\theta) + \xi_s \langle \vx_s, \bm\theta - \bm\theta_\star \rangle - \KL(\mu_s(\bm\theta_\star), \mu_s(\bm\theta)).
    \end{equation*}
\end{lemma}
\begin{proof}
    The proof follows from the first-order Taylor expansion with integral remainder and some careful rearranging of the terms (which is nontrivial); see Appendix~\ref{app:proof-decomposition1} for the full proof.
\end{proof}

We can then replace $\bm\theta$ in the above lemma with a sequence of parameters, $\{\tilde{\bm\theta}_s\}$, ``outputted'' from an online learning algorithm of our choice.
This does {\it not} imply that the algorithm of \cite{foster2018logistic} is proper, as the choice of $\tilde{\bm\theta}_s$ {\it depends} on the current given instance $\vx$; see the paragraph below Theorem~\ref{thm:foster}.

% One caveat is that some online learning algorithms, especially the improper learners, output the prediction $\hat{y}$ for the reward directly; we address this point in Remark~\ref{rmk:online}.
Stemming from this, the following is the second decomposition lemma:
\begin{lemma}
\label{lem:decomposition2-logistic}
    For the logistic loss $\ell_s$, the following holds:
    \begin{equation}
    \label{eqn:decomposition}
     \sum_{s=1}^t \ell_s(\bm\theta_\star) - \ell_s(\widehat{\bm\theta}_t) \leq \Reg^O(t) + \zeta_1(t) - \zeta_2(t),
    \end{equation}
    where $\Reg^O(t) := \sum_{s=1}^t \ell_s(\tilde{\bm\theta}_s) - \sum_{s=1}^t \ell_s(\widehat{\bm\theta}_t)$ is the regret incurred by the online learning algorithm of our choice up to time $t$, $\zeta_1(t) := \sum_{s=1}^t \xi_s \langle \vx_s, \bm\theta_\star - \tilde{\bm\theta}_s \rangle$ is a sum of martingale difference sequences, and $\zeta_2(t) := \sum_{s=1}^t \KL(\mu_s(\bm\theta_\star), \mu_s(\tilde{\bm\theta}_s))$ is a sum of KL-divergences.
%    \begin{equation*}
%        \zeta_1(t) \!:=\! \sum_{s=1}^t \xi_s \langle \vx_s, \bm\theta_\star - \tilde{\bm\theta}_s \rangle, \
%        \zeta_2(t) \!:=\! \sum_{s=1}^t \KL(\mu_s(\bm\theta_\star), \mu_s(\tilde{\bm\theta}_s)).
%    \end{equation*}
\end{lemma}
\begin{proof}
    The proof follows from Lemma~\ref{lem:decomposition1-logistic} and some rearranging; see Appendix~\ref{app:proof-decomposition2} for the full proof.
\end{proof}
% The above lemma is the crux of our R2CS approach: the regret of {\it any} online learning algorithm can be transformed into a new confidence set without ever running the online algorithm, hence our name {\it regret-to-confidence-set}.
\begin{remark}
    This decomposition is similar to the online-to-PAC conversion of \cite{lugosi2023conversion}; see Appendix~\ref{app:o2cs} for more discussions.
\end{remark}

For $\Reg^O(t)$, we use the following regret bound for online logistic regression scaling {\it logarithmically} in $S$:
\begin{theorem}[Theorem 3 of \cite{foster2018logistic}]
\label{thm:foster}
    There exists an (improper learning) algorithm for online logistic regression with the following regret:
    \begin{equation}
    \label{eqn:foster}
        \Reg^O(t) \leq 10d \log \left( e + \frac{S t}{2d} \right).
    \end{equation}
\end{theorem}

\begin{remark}
    The dependency on $S$ is tight with corresponding lower bound; see Theorem 5 of \cite{foster2018logistic} and Theorem 6 of \cite{mayo22scale}.
\end{remark}

The output of Algorithm 1 of \cite{foster2018logistic} is a sequence of $\hat{\vz}_s = (\hat{z}_0, \hat{z}_1)$, corresponding to $\vx_s$ at each time $s$.
For our purpose, we need to designate a vector $\tilde{\bm\theta}_s \in \gB^d(S)$ such that $\bm\sigma(\widehat{\vz}_s) = \bm\sigma\left( (\langle \vx_s, \tilde{\bm\theta}_s \rangle) \right)$, where $\bm\sigma : \sR^1 \rightarrow \Delta_{>0}^2$ is the softmax function defined as $\bm\sigma(z_1) = \left(\frac{1}{1 + e^{z_1}}, \frac{e^{z_1}}{1 + e^{z_1}} \right)$; see Proposition~\ref{prop:surjection} in Appendix~\ref{app:rmk-proof} for a generalization of this for $(K+1)$-classification.
% Furthermore, the analysis shows that for our purpose, it suffices to use $B = \frac{S}{2}$ in the notation of \cite{foster2018logistic}; see footnote 7 of Appendix~\ref{app:rmk-proof} for an explanation.
% \kj{actually, Foster's algorithm only predicts $y_s$ rather than outputting a parameter. but we can say that we can designate a parameter that makes the same output as $y_s$. Let's discuss this next time we meet.}

\paragraph{Upper Bounding $\zeta_1(t)$: Martingale Concentrations.}
Recall that $\gF_s = \sigma\left( \left\{ \vx_1, r_1, \cdots, \vx_s, r_s, \vx_{s+1} \right\} \right)$ is the filtration for the canonical bandit model.
We start by observing that $\vx_s$ and $\tilde{\bm\theta}_s$ are $\gF_{s-1}$-measurable, and $\xi_s$ is a martingale difference sequence w.r.t. $\gF_{s-1}$.
We also have that
\begin{align*}
    |\xi_s \langle \vx_s, \tilde{\bm\theta}_s - \bm\theta_\star \rangle| &\leq 2S, \\
    \E[\xi_s \langle \vx_s, \tilde{\bm\theta}_s - \bm\theta_\star \rangle | \gF_{s-1}] &= 0,
\end{align*}
and
\begin{equation*}
    \E[\xi_s^2 \langle \vx_s, \bm\theta_\star - \tilde{\bm\theta}_s \rangle^2 | \gF_{s-1}] = \dot{\mu}(\vx_s^\intercal \bm\theta_\star) \langle \vx_s, \tilde{\bm\theta}_s - \bm\theta_\star \rangle^2.
\end{equation*}

We now use a variant\footnote{This is a slight variant from the original inequality~\citep[Theorem 1.6]{freedman1975concentration} in that this uses any fixed estimate of the variance rather than an upper bound.} of Freedman's inequality for martingales, combined with Ville's inequality to make the concentration hold for any $t \geq 1$.
\begin{lemma}[Modification of Theorem 1 of \cite{beygelzimer2011contextual}]
\label{lem:freedman}
    Let $X_1, \cdots, X_t$ be martingale difference sequence satisfying $\max_s |X_s| \leq R$ a.s, and let $\gF_s$ be the $\sigma$-field generated by $(X_1, \cdots, X_s)$.
    Then for any $\delta \in (0, 1)$ and any $\eta \in [0, 1/R]$, the following holds with probability at least $1 - \delta$:
    \begin{equation*}
        \sum_{s=1}^t X_s \leq (e - 2) \eta \sum_{s=1}^t \E[X_s^2 | \gF_{s-1}] + \frac{1}{\eta} \log\frac{1}{\delta}, \quad \forall t \geq 1.
    \end{equation*}
    % where $V_t := \sum_{s=1}^t \E[X_s^2 | \gF_{s-1}]$ is the martingale variance term.
\end{lemma}
\begin{proof}
% \begin{proof}[Proof of Lemma~\ref{lem:freedman}]
    Define $Z_0 = 1$ and $Z_t = Z_{t-1} \cd \exp(\lam X_t - (e-2)\lam^2 \EE[X_t^2 \mid \gF_{t-1}]), \forall t\ge1$.
    The proof of Theorem 1 of \cite{beygelzimer2011contextual} shows that $(Z_t)_{t=0}^\infty$ is supermartingale and then applies Markov's inequality.
    In our proof, we apply Ville's inequality (Lemma~\ref{lem:ville} in Appendix~\ref{app:ville}), to conclude the proof.
\end{proof}

Thus, for $\eta \in \left[0, \frac{1}{2S}\right]$ to be chosen later, the following holds with probability at least $1 - \delta$: for all $t \geq 1$,
\begin{equation}
\label{eqn:zeta1}
    \zeta_1(t) \!\leq\! (e - 2)\eta\! \sum_{s=1}^t \dot{\mu}(\vx_s^\intercal \bm\theta_\star) \langle \vx_s, \bm\theta_\star \!-\! \tilde{\bm\theta}_s \rangle^2 \!+\! \frac{1}{\eta} \log\frac{1}{\delta}.
\end{equation}
% \kj{should be $\log(1/\delta)$?}

\paragraph{Lower Bounding $\zeta_2(t)$: Second-order Expansion of KL Divergence.}
We first recall the definition of Bregman divergence:
\begin{definition}
    For a given $m : \gZ \rightarrow \sR$, the {\bf Bregman divergence} $D_m(\cdot, \cdot)$ is defined as follows:
    \begin{equation*}
        D_m(\vz_1, \vz_2) = m(\vz_1) - m(\vz_2) - \nabla m(\vz_2)^\intercal (\vz_1 - \vz_2)
    \end{equation*}
\end{definition}
In our case, $\gZ = \sR$, and thus, from the first-order Taylor's expansion with integral remainder, we have that
\begin{equation}
\label{eqn:bregman}
    D_m(z_1, z_2) = \int_{z_2}^{z_1} m''(z) (z_1 - z) dz.
\end{equation}

The following lemma, which is a standard result in information geometry~\citep{infogeom,nielsen2020infogeom,brekelmans2020bregman}, relates Bernoulli KL divergence to a specific Bregman divergence; we provide the proof in Appendix~\ref{app:kl-bregman} for completeness.
% \kj{actually, this is a standard result from information geometry; we can perhaps cite Amari for it. we can still keep the proof and say 'we show a direct proof in appendix XX'.}
\begin{lemma}
\label{lem:kl-bregman}
    Let $m(z) := \log(1 + e^z)$ be the log-partition function for Bernoulli distribution and $\mu(z) = \frac{1}{1 + e^{-z}}$.
    Then, we have that $\KL(\mu(z_2), \mu(z_1)) = D_m(z_1, z_2)$.
\end{lemma}
% \begin{proof}
%     This follows from clever rearranging of the terms starting from the definition of Bregman divergence; see Appendix~\ref{app:kl-bregman} for the full proof.
% \end{proof}

Combining all of the above and the fact that $m''(z) = \dot{\mu}(z)$, we have that
\begin{align*}
    &\KL(\mu_t(\vx_s^\intercal\bm\theta_\star)), \mu(\vx_s^\intercal \tilde{\bm\theta}_s)) \\
    &= D_m(\vx_s^\intercal \tilde{\bm\theta}_s, \vx_s^\intercal\bm\theta_\star) \tag{Lemma~\ref{lem:kl-bregman}} \\
    &= \int_{\vx_s^\intercal\bm\theta_\star}^{\vx_s^\intercal \tilde{\bm\theta}_s} \dot{\mu}(z) (\vx_s^\intercal \tilde{\bm\theta}_s - z) dz \tag{Eqn.~\eqref{eqn:bregman}} \\
    &= \langle \vx_s, \bm\theta_\star - \tilde{\bm\theta}_s \rangle^2 \int_0^1 (1 - v) \dot{\mu}(\vx_s^\intercal (\tilde{\bm\theta}_s + (1 - v) \bm\theta_\star)) dv \tag{change-of-variable} \\
    &\overset{(*)}{\geq} \langle \vx_s, \bm\theta_\star - \tilde{\bm\theta}_s \rangle^2 \frac{\dot{\mu}(\vx_s^\intercal \bm\theta_\star)}{2 + |\vx_s^\intercal (\bm\theta_\star - \tilde{\bm\theta}_s)|} \\
    &\geq \langle \vx_s, \bm\theta_\star - \tilde{\bm\theta}_s \rangle^2 \frac{\dot{\mu}(\vx_s^\intercal \bm\theta_\star)}{2 + 2S} \tag{Assumption~\ref{assumption:X}, \ref{assumption:S} and triangle inequality},
\end{align*}
where $(*)$ is due to the following self-concordant result:
\begin{lemma}[Lemma 8 of \cite{abeille2021logistic}]
\label{lem:self-concordant}
    Let $f$ be any strictly increasing self-concordant function, i.e., $|\ddot{\mu}| \leq \dot{\mu}$, and let $\gZ \subset \sR$ be bounded.
    Then, the following holds for any $z_1, z_2 \in \gZ$:
    \begin{equation*}
        \int_0^1 (1 - v) \dot{f}(z_1 + v(z_2 - z_1)) dv \geq \frac{\dot{f}(z_1)}{2 + |z_1 - z_2|}.
    \end{equation*}
\end{lemma}

All in all, we have that
\begin{equation}
\label{eqn:zeta2}
    \zeta_2(t) \geq \frac{1}{2 + 2S} \sum_{s=1}^t \dot{\mu}(\vx_s^\intercal \bm\theta_\star) \langle \vx_s, \bm\theta_\star - \tilde{\bm\theta}_s \rangle^2.
\end{equation}
% We now bound the integral using self-concordant properties of $\mu$, but the analyses in \cite{abeille2021logistic} (e.g., their Lemma 8) are not tight.

\vspace{-.6em}
\paragraph{Wrapping up the proof.}
Combining Eqn.~\eqref{eqn:decomposition}, \eqref{eqn:foster}, \eqref{eqn:zeta1}, \eqref{eqn:zeta2} with $\eta = \frac{1}{2(e - 2) + 2S} < \frac{1}{2S}$ and the fact that $-\frac{1}{2 + 2S} + \frac{e - 2}{2(e - 2) + 2S} < 0$, we are done.

\section{IMPROVED REGRET}
\label{sec:logistic-regret}
\subsection{\texttt{OFULog+} and Improved Regret}
\label{subsec:logistic-regret}
\RestyleAlgo{ruled}
\begin{algorithm2e}[t]
	\SetAlgoLined
	% \KwInput{}
	\For{$t = 1, \dots, T$}{
		$\widehat{\bm\theta}_t \gets \argmin_{\lVert \bm\theta \rVert_2 \leq S} \gL_t(\bm\theta)$\;

        $(\vx_t, \bm\theta_t) \gets \argmax_{\vx \in \gX_t, \bm\theta \in \gC_t(\delta)} \mu(\langle \vx, \bm\theta \rangle)$, with $\gC_t(\delta)$ as defined in Theorem~\ref{thm:confidence-logistic}\;

        Play $\vx_t$ and observe reward $r_t$\;
	}
	\caption{\texttt{OFU-Log+}}
	\label{alg:logistic}
\end{algorithm2e}

Our new loss-based confidence set (Theorem~\ref{thm:confidence-logistic}) leads to an OFUL-type algorithm~\citep{abbasiyadkori2011linear}, which we refer to as \texttt{OFULog+}; its pseudocode is shown in Algorithm~\ref{alg:logistic}.

Note that the optimization in line 2 is tractable because $\gC_t(\delta)$ is always convex (as $\gL_t$ is convex, and the level set of any convex function is convex), and $\mu(\cdot)$ is an increasing function, meaning that line 2 can be equivalently rewritten as
\begin{equation*}
    (\vx_t, \bm\theta_t) \in \argmax_{\vx \in \gX_t, \bm\theta \in \gC_t(\delta)} \langle \vx, \bm\theta \rangle.
\end{equation*}

The existing confidence-set-based approach to logistic bandit was due to \cite{abeille2021logistic}, in which they first proposed a nonconvex confidence set, from which a loss-based confidence set was derived via convex relaxation.
As our R2CS directly constructs the loss-based confidence set, this can be elegantly ``plugged-in'' to the algorithm and proof of \cite{abeille2021logistic} with minimal change.
This is in contrast to \cite{faury2022logistic}, which requires major algorithmic innovations.

We now present the regret bound of \texttt{OFULog+} (See Theorem~\ref{thm:new-regret-logistic-full} in Appendix~\ref{app:full-statements} for the full statement, including the omitted logarithmic factors.):
\begin{theorem}[Simplified]
\label{thm:new-regret-logistic}
    \texttt{{OFULog+}} attains the following regret bound with probability at least $1 - \delta$:
    \begin{equation*}
        \Reg^B(T) \lesssim dS \sqrt{\frac{T}{\kappa_\star(T)}} + \min\left\{ d^2 S^2 \kappa_\gX(T), R_\gX(T) \right\},
    \end{equation*}
    where $R_\gX(T) := S \sum_{t=1}^T \mu(\vx_{t,\star}^\intercal \bm\theta_\star) \mathds{1}[\vx_t \in \gX_-(t)]$ and the RHS hides dependencies on $\log\frac{1}{\delta}$.
    Here, $\gX_-(t)$ is the set of detrimental arms at time $t$; see Section 4 of \cite{abeille2021logistic}.
\end{theorem}

\begin{remark}
Explicitly running the algorithm of \cite{foster2018logistic} and constructing a confidence set using techniques like \cite{abbasiyadkori2012conversion} does not yield a better guarantee, as the confidence set radius depends additively on the online regret.
Moreover, their algorithm is {\it computationally \underline{very} heavy}; our R2CS does this using only an {\it achievable} online regret bound.
\end{remark}

Extending upon Table~\ref{tab:results}, below, we discuss in detail how our bound compares to existing works\footnote{see Appendix~\ref{app:full-statements} for the omitted full statements of prior regret bounds.}:
\paragraph{Comparison to Prior Arts.}
Contextual logistic bandits, with time-varying arm-set, were first studied by \cite{faury2020logistic}, in which the authors derived the regret bounds of $\widetilde{\gO}(\sqrt{\kappa(T) T})$ and $\widetilde{\gO}(\sqrt{T} + \kappa(T))$ (corresponding to their two algorithms) based on self-concordant analyses of logistic regression~\citep{bach2010logistic}.
Although not tight, their analyses laid a stepping stone for the subsequent works on logistic bandits.
\cite{abeille2021logistic} provided the first algorithm that attains\footnote{In the original paper, the authors considered $\lambda_t = d \log \frac{t}{\delta}$, which incurred additional factors in $S$. Here, for a fair comparison, we re-tracked the $S$-dependencies with the ``optimal'' choice of $\lambda_t = \frac{d}{S} \log\frac{1 + \frac{St}{d}}{\delta}$.} a regret bound of $\widetilde{\gO}\left( dS^{\frac{3}{2}} \sqrt{\frac{T}{\kappa_\star(T)}} + \min\left\{d^2 S^3 \kappa_\gX(T), R_\gX(T) \right\} \right)$ along with near-matching minimax lower bound via an intricate local analysis.
\cite{abeille2021logistic} also proposed a tractable variant of the algorithm, \texttt{OFULog-r}, via a convex relaxation, but it incurs an extra dependency on $S$ as shown in Table~\ref{tab:results}.
\cite{faury2022logistic} provided a jointly efficient and optimal algorithm with $\widetilde{\gO}\left( dS \sqrt{\frac{T}{\kappa_\star(T)}} + d^2 S^6 \kappa(T) \right)$ regret that takes $\Omega(1)$ time complexity.
Our regret bound's leading term, $dS\sqrt{\frac{T}{\kappa_\star(T)}}$, improves upon \cite{abeille2021logistic} by a factor of $\sqrt{S}$ and matches that of \cite{faury2022logistic}, and our lower-order term, $\min\{ d^2 S^2 \kappa_\gX(T), R_\gX(T) \}$, improves upon \cite{abeille2021logistic} by a factor of $S$ and improves upon \cite{faury2022logistic} by a factor of $S^4$ and possibly $\kappa(T)$.

In Section~\ref{sec:expr}, we provide numerical results for logistic bandits, showing that our \texttt{OFULog+} obtains the state-of-the-art performance in regret over prior arts and results in a tighter confidence set.

On a slightly different approach, \cite{mason2022logistic} proposed an experimental design-based algorithm. However, the algorithm and its guarantee require the arm-set to be {\it not} time-varying, making them incomparable to ours.
Moreover, the current arm-elimination approach like \cite{mason2022logistic} is impractical as it needs a long warmup length of order at least $\gO(\kappa d^2)$.
This is in contrast to the optimism-based approach, which incurs a lower-order algorithm adaptive to the arm-set geometry in that the lower-order term may scale independently of $\kappa_\gX(T)$, given that the arm-set is sufficiently benign, e.g., unit ball~\citep[Theorem 3]{abeille2021logistic}.
\texttt{SupLogistic} of \cite{jun2021confidence} assumes that the context vectors follow a distribution and further assumes the minimum eigenvalue condition on the context covariance matrix, which is rather limiting.

\begin{remark}
Note that \cite{mason2022logistic} completely removes the factor of $S$ from the leading term in the regret bound in the fixed arm set setting.
We speculate that it is possible to construct an optimism-based algorithm that does not scale with $S$ in the leading term of the regret (up to logarithmic factors), at least for the fixed arm set setting.
We leave to future work whether it is possible to improve further the radius of the confidence set from $\gO(\sqrt{(d + S) \log t})$ to $\gO(\sqrt{d \log t})$.
% We leave this as a future work.
\end{remark}

\subsection{Proof Sketch of Theorem~\ref{thm:new-regret-logistic}}
The proof of \cite{abeille2021logistic} heavily relies on an upper bound on the Hessian-induced distance between $\bm\theta \in \gC_t(\delta)$ and $\bm\theta_\star$, $\lVert \bm\theta - \bm\theta_\star \rVert_{\mH_t(\bm\theta_\star)}$.
Here, we define a regularized Hessian $\mH_t(\bm\theta_\star)$ centered at $\bm\theta_\star$ as
\begin{equation*}
    \mH_t(\bm\theta_\star) := \sum_{s=1}^{t-1} \dot{\mu}(\vx_s^\intercal \bm\theta_\star) \vx_s \vx_s^\intercal + \lambda_t \mI_d,
\end{equation*}
where the regularization coefficient $\lambda_t > 0$ is to be chosen later.
Note that although our MLE is not regularized (Eqn.~\eqref{eqn:MLE}), the regularization ensures that $\mH_t$ is positive definite, allowing us to use the elliptical potential lemma argument w.r.t. $\mH_t^{-1}$-induced norm in the later proof.
We remark here that unlike \cite{abeille2021logistic} where $\lambda_t$ directly impacts the algorithm design, in our case, $\lambda_t$ is solely for the proof and does {\it not} impact our algorithm in any way.

Two key differences exist between our proof and that of \cite{abeille2021logistic}.
One is that we derive a new (high-probability) upper bound on $\lVert \bm\theta - \bm\theta_\star \rVert_{\mH_t(\bm\theta_\star)}$ (Lemma~\ref{lem:H-norm}).
Na\"{i}vely using Cauchy-Schwartz inequality and self-concordant controls (as done in the proof of Lemma 1 of \cite{abeille2021logistic}) gives us an extra factor of $S$.
To circumvent this, we instead use the martingale decomposition of the logistic bandit reward (Eqn.~\eqref{eqn:martingale}) and Freedman's inequality (Lemma~\ref{lem:freedman}) with an $\eps$-net argument, leading to extra factors of $S$ shaved off at the end.
Another is that we use a more refined elliptical potential count lemma argument to avoid the extra dependencies on $S$ (Lemma~\ref{lem:EPCL}, \ref{lem:EPL}; see Remark~\ref{rmk:EPCL} in Appendix~\ref{app:proof-regret-logistic}).
With these and our new confidence set (Theorem~\ref{thm:confidence-logistic}), we appropriately modify the proof of \cite{abeille2021logistic} to arrive at our new regret bound.

\subsection{Complete Proof of Theorem~\ref{thm:new-regret-logistic}}
We start with the following crucial lemma bounding the Hessian-induced distance between $\bm\theta$ and $\bm\theta_\star$:
\begin{restatable}{lemma}{hnorm}
\label{lem:H-norm}
    With $\lambda_t = \frac{1}{4S^2 (2 + 2S)}$, for any $\bm\theta \in \gC_t(\delta)$, the following holds with probability at least $1 - \delta$:
    \begin{equation*}
        \lVert \bm\theta - \bm\theta_\star \rVert_{\mH_t(\bm\theta_\star)}^2 \lesssim \gamma_t(\delta)^2 \triangleq S^2 \left( d \log\left( e + \frac{St}{d}\right) + \log\frac{1}{\delta} \right).
    \end{equation*}
\end{restatable}
\begin{proof}
    By Theorem~\ref{thm:confidence-logistic}, we have that with probability at least $1 - \delta$, $\gL_t(\bm\theta_\star) - \gL_t(\widehat{\bm\theta}_t) \leq \beta_t(\delta)^2$; throughout the proof let us assume that this event is true.
    Also, let $\bm\theta \in \gC_t(\delta)$.
    Then, by second-order Taylor expansion of $\gL_t(\bm\theta)$ around $\bm\theta_\star$,
    \begin{equation*}
        \gL_t(\bm\theta)
        = \gL_t(\bm\theta_\star) + \nabla \gL_t(\bm\theta_\star)^\intercal (\bm\theta - \bm\theta_\star) + \lVert \bm\theta - \bm\theta_\star \rVert_{\widetilde{\mG}_t(\bm\theta_\star, \bm\theta) - \lambda_t \mI}^2,
    \end{equation*}
    where $\lambda_t > 0$ is to be determined, and we define the following quantities:
    \begin{align*}
        \widetilde{\alpha}(\vx, \bm\theta_1, \bm\theta_2) &:= \int_0^1 (1 - v) \dot{\mu}\left( \vx^\intercal (\bm\theta_1 + v(\bm\theta_2 - \bm\theta_1)) \right) dv \\
        \widetilde{\mG}_t(\bm\theta_1, \bm\theta_2) &:= \sum_{s=1}^{t-1} \widetilde{\alpha}(\vx_s, \bm\theta_1, \bm\theta_2) \vx_s \vx_s^\intercal + \lambda_t \mI_d.
    \end{align*}
        
    Lemma~\ref{lem:self-concordant} implies that $\widetilde{\mG}_t(\bm\theta_1, \bm\theta_2) \succeq \frac{1}{2 + 2S} \mH_t(\bm\theta_1)$.
    Thus, we have that
    % {\footnotesize
    % \begin{align*}
    %     &\lVert \bm\theta - \bm\theta_\star \rVert_{\mH_t(\bm\theta_\star)}^2 \\
    %     % &\leq (2 + 2S) \lVert \bm\theta - \bm\theta_\star \rVert_{\widetilde{\mG}_t(\bm\theta_\star, \bm\theta)}^2 \\
    %     &\leq (2 + 2S) \left( \gL_t(\bm\theta) - \gL_t(\bm\theta_\star) + \nabla \gL_t(\bm\theta_\star)^\intercal (\bm\theta_\star - \bm\theta) + \lambda_t \lVert \bm\theta - \bm\theta_\star \rVert_2^2 \right) \\
    %     &\leq 1 + (2 + 2S) \left( \gL_t(\bm\theta) - \gL_t(\widehat{\bm\theta}_t) + \nabla \gL_t(\bm\theta_\star)^\intercal (\bm\theta_\star - \bm\theta) \right)  \tag{$\gL_t(\widehat{\bm\theta}_t) \leq \gL_t(\bm\theta_\star)$, $\lambda_t = \frac{1}{4S^2 (2 + 2S)}$} \\
    %     &\leq 1 + (2 + 2S) \beta_t(\delta)^2 + (2 + 2S) \nabla \gL_t(\bm\theta_\star)^\intercal (\bm\theta_\star- \bm\theta), \tag{$\bm\theta \in \gC_t(\delta)$}
    % \end{align*}
    % }
    \begin{align*}
        &\lVert \bm\theta - \bm\theta_\star \rVert_{\mH_t(\bm\theta_\star)}^2 \\
        &\lesssim S \left( \gL_t(\bm\theta) - \gL_t(\bm\theta_\star) + \nabla \gL_t(\bm\theta_\star)^\intercal (\bm\theta_\star - \bm\theta) + \lambda_t \lVert \bm\theta - \bm\theta_\star \rVert_2^2 \right) \\
        &\lesssim S \left( \gL_t(\bm\theta) - \gL_t(\widehat{\bm\theta}_t) + \nabla \gL_t(\bm\theta_\star)^\intercal (\bm\theta_\star - \bm\theta) \right)  \tag{$\gL_t(\widehat{\bm\theta}_t) \leq \gL_t(\bm\theta_\star)$, $\lambda_t = \frac{1}{4S^2 (2 + 2S)}$} \\
        &\lesssim S \beta_t(\delta)^2 + S \nabla \gL_t(\bm\theta_\star)^\intercal (\bm\theta_\star- \bm\theta), \tag{$\bm\theta \in \gC_t(\delta)$}
    \end{align*}
    where the last inequality holds with probability at least $1 - \delta$.
    Note that we do not need $\lambda_t$ to vary over $t$.
    
    As $\nabla \gL_t(\bm\theta_\star)^\intercal (\bm\theta_\star- \bm\theta)$ can be written as a sum of martingale difference sequences and $\bm\theta_\star- \bm\theta \in \gB^d(2S)$, the proof then concludes via a time-dependent $\eps$-net argument on $\gB^d(2S)$ with Freedman's inequality; see Appendix~\ref{app:H-norm} for the missing details.
\end{proof}
The proof of Theorem~\ref{thm:new-regret-logistic} finally concludes by tracking the regret analysis of Appendix C of \cite{abeille2021logistic}; see Appendix~\ref{app:proof-regret-logistic} for the remaining argument.

\section{EXTENSION TO MNL BANDITS}
\label{sec:multinomial}

\paragraph{Problem Setting.}
We now consider a natural extension of logistic bandits, namely, multinomial logistic (MNL) bandits, first introduced in \cite{amani2021mnl}.
At every round $t$, the learner observes a potentially infinite arm-set $\gX_t$, which can also be time-varying, and plays an action $\vx_t \in \gX$.
% She then receives a response of $\vy_t \in \gY := \left\{ \vy \in \sR_+^K : \lVert \vy \rVert_1 \leq L \right\}$, where
She then receives a reward of $r_t = \bm\rho^\intercal \vy_t$, where $\bm\rho \in \sR^{K}$ is a known reward vector, and $\vy_t = (y_{t,1}, \cdots, y_{t,K}) \in \{0, 1\}^K$ satisfies $\lVert \vy_t \rVert_1 \leq 1$.
$y_{s,k} = 1$ when $k$-th item is chosen at time $s$, and for simplicity we denote $y_{t,0} := 1 - \lVert \vy_t \rVert_1$.
Then, $(y_0, \vy_t)$ follows the multinomial logit choice model:
\begin{equation}
    \sP[\vy_t = \bm\delta_k | \vx_t] =
    \begin{cases}
        \mu_k(\vx_t, \bm\Theta_\star) &\quad k > 0, \\
        1 - \sum_{j=1}^K \mu_j(\vx_t, \bm\Theta_\star) &\quad k = 0,
    \end{cases}
\end{equation}
where $\bm\delta_k$ is the $K$-dimensional one-hot encoding for the index $k$ and $\bm\delta_0 := \vzero$.
Intuitively, $\vy_t = \bm\delta_0$ corresponds to the scenario where the user has not chosen any of the $K$ possible choices.
Here, we denote
\begin{equation}
    \mu_k(\vx_t, \bm\Theta_\star) := \frac{\exp\left( \langle \vx_t, (\bm\theta_\star^{(k)}) \rangle \right)}{1 + \sum_{j=1}^K \exp\left( \langle \vx_t, (\bm\theta_\star^{(j)}) \rangle \right)}
\end{equation}
for some unknown $\left\{ \bm\theta_\star^{(j)} \right\}_{j=1}^K \subset \sR^d$.
Here, we use $K \times d$ matrix to denote the unknown parameter, namely, $\bm\Theta_\star := [\bm\theta_\star^{(1)}, \cdots, \bm\theta_\star^{(K)}]^\intercal \in \sR^{K \times d}$ and $\bm\mu(\vx_t, \bm\Theta_\star) := [\mu_t(\bm\theta_\star^{(1)}), \cdots, \mu_t(\bm\theta_\star^{(K)})]^\intercal$.
This simplifies some parts of the analysis (e.g., avoid using Kronecker products).

The regret of MNL bandits is defined as follows:
\begin{equation}
    \Reg^B(T) := \sum_{t=1}^T \bm\rho^\intercal \left( \bm\mu(\vx_{t,\star}, \bm\Theta_\star) - \bm\mu(\vx_t, \bm\Theta) \right),
\end{equation}
where $\vx_{t,\star} := \argmax_{\vx \in \gX} \bm\rho^\intercal \bm\mu(\vx, \bm\Theta_\star)$.

We define the following quantity, which will be crucial in our overall analysis:
\begin{equation}
    \mA(\vx, \bm\Theta) := \diag(\bm\mu(\vx, \bm\Theta)) - \bm\mu(\vx, \bm\Theta) \bm\mu(\vx, \bm\Theta)^\intercal.
\end{equation}

We also have the following assumptions with problem-dependent quantities:
\begin{assumption}
    $\gX_t \subseteq \gB^d(1)$ for all $t \geq 1$.
\end{assumption}
\begin{assumption}
\label{assumption:R}
    There exist known $S, R > 0$ such that $\bm\Theta_\star \in \gB^{K \times d}(S)$ and $\bm\rho \in \gB^d(R)$.
\end{assumption}

We consider the following problem-dependent quantity~\citep{amani2021mnl}:
\begin{align*}
    \kappa(T) := \max_{t \in [T]} \max_{\vx \in \gX_t} \max_{\bm\Theta \in \gB^{K \times d}(S)} \frac{1}{\lambda_{\min}\left( \mA(\vx, \bm\Theta) \right)}.
\end{align*}

\paragraph{Improved Confidence Set.}
We proceed similarly to how we applied R2CS to logistic bandits; to make the correspondence explicit, we overload the notations used in previous sections.
We first define the norm-constrained, unregularized MLE for multiclass logistic regression as
\begin{equation}
    \widehat{\bm\Theta}_t := \argmin_{\bm\Theta \in \gB^{K \times d}(S)} \gL_t(\bm\Theta) \triangleq \sum_{s=1}^{t-1} \ell_s(\bm\Theta),
\end{equation}
where $\ell_s$ is the multiclass logistic (or {\it softmax-cross-entropy}) loss at time $s$, defined as
\begin{equation*}
    \ell_s(\bm\Theta) := - \sum_{k=0}^K y_{s,k} \log \mu_k (\vx_s, \bm\Theta),
\end{equation*}
where we denote $\mu_0(\vx_s, \bm\Theta) := 1 - \sum_{j=1}^K \mu_j(\vx_0, \bm\Theta)$.

Via similar analysis, we obtain the following new confidence set, which reduces to the logistic case when $K = 1$ up to some absolute constants:
\begin{theorem}[Improved Confidence Set for Multinomial Logistic Loss]
\label{thm:confidence-multinomial}
    We have
    \begin{align*}
        \sP\left[ \forall t \geq 1, \  \bm\Theta_\star \in \gC_t(\delta) \right] \geq 1 - \delta,
    \end{align*}
    with
    \begin{align*}
        \gC_t(\delta) &\!=\! \left\{ \bm\Theta \in \gB^{K \times d}(S) : \gL_t(\bm\Theta) - \gL_t(\widehat{\bm\Theta}_t) \leq \beta_t(\delta)^2 \right\},\\
        \beta_t(\delta)^2 &\!=\! 5dK' \log\left(e + \frac{St}{dK'}\right) + 2((e - 2) + \sqrt{6K}S) \log\frac{1}{\delta},
    \end{align*}
    where we denote $K' = K + 1$.
\end{theorem}
\begin{proof}
    We extend our previous proof of Theorem~\ref{thm:confidence-logistic} to the multinomial scenario.
    Some key differences include using generalized self-concordant control~\citep{sun2019generalized,trandinh2015generalized}, properties of the Kronecker product, and devising multinomial versions of a new martingale concentration argument (Lemma~\ref{lem:H-norm-multinomial}); see Appendix~\ref{app:confidence-multinomial} for the full proof.
\end{proof}
% Note that when $K = 1$, up to absolute constants, the radius $\beta_t(\delta)$ reduces to that of logistic loss.

\subsection{\texttt{MNL-UCB+} and Improved Regret}
\label{subsec:multinomial-regret}
\RestyleAlgo{ruled}
\begin{algorithm2e}[t]
	\SetAlgoLined
	% \KwInput{}
	\For{$t = 1, \dots, T$}{
		$\widehat{\bm\Theta}_t \gets \argmin_{\lVert \bm\Theta \rVert_2 \leq S} \gL_t(\bm\Theta)$\;

        $\vx_t \gets \argmax_{\vx \in \gX_t} \bm\rho^\intercal \bm\mu(\vx, \widehat{\bm\Theta}_t) + \epsilon_t(\vx)$, with $\epsilon_t(\vx) = \sqrt{2\kappa(T)} R L \gamma_t(\delta) \lVert \vx \rVert_{\mV_t^{-1}}$\;
        % , with $\gC_t(\delta)$ as defined in Theorem~\ref{thm:confidence-logistic}\;

        Play $\vx_t$ and observe reward $r_t$\;
	}
	% \KwOutput{$\hat{f}_{1}$ (initial estimate of the decoding function)}
	\caption{\texttt{MNL-UCB+}}
	\label{alg:multinomial1}
\end{algorithm2e}

\begin{savenotes}
\RestyleAlgo{ruled}
\begin{algorithm2e}[t]
	\SetAlgoLined
	% \KwInput{}
        $\gM_1(\bm\Theta) \gets \gB^{K \times d}(S)$\;
	\For{$t = 1, \dots, T$}{
        $\widehat{\bm\Theta}_t \gets \argmin_{\bm\Theta \in \gM_t} \gL_t(\bm\Theta)$\;
        
        $\vx_t \gets \argmax_{\vx \in \gX_t} \bm\rho^\intercal \bm\mu(\vx, \widehat{\bm\Theta}_t) + \overline{\epsilon}_t(\vx)$, with $\overline{\epsilon}_t(\vx)$ defined in Eqn.~\eqref{eqn:improved-bonus-multinomial} (Appendix~\ref{app:multinomial2})\;

        Play $\vx_t$ and observe reward $r_t$\;

        $\gM_{t+1} \gets \gM_t \cap \left\{ \bm\Theta : \exists \bm\Theta'_t \in \mathrm{min}(\gC_t(\delta)) \text{ s.t. } \mA(\vx_t, \bm\Theta) \succeq \mA(\vx_t, \bm\Theta'_t \right\}$\footnote{$\mathrm{min}(\gC_t(\delta))$ is the set of all minimal elements of the poset $\gC_t(\delta)$, endowed with the Loewner ordering w.r.t. $\mA(\vx_t, \bm\Theta)$.}\;
	}
	\caption{Improved \texttt{MNL-UCB+}}
	\label{alg:multinomial2}
\end{algorithm2e}
\end{savenotes}

\begin{figure*}[!t]
\centering
\subfigure[$S = 5$]{
\includegraphics[width=0.233\textwidth]{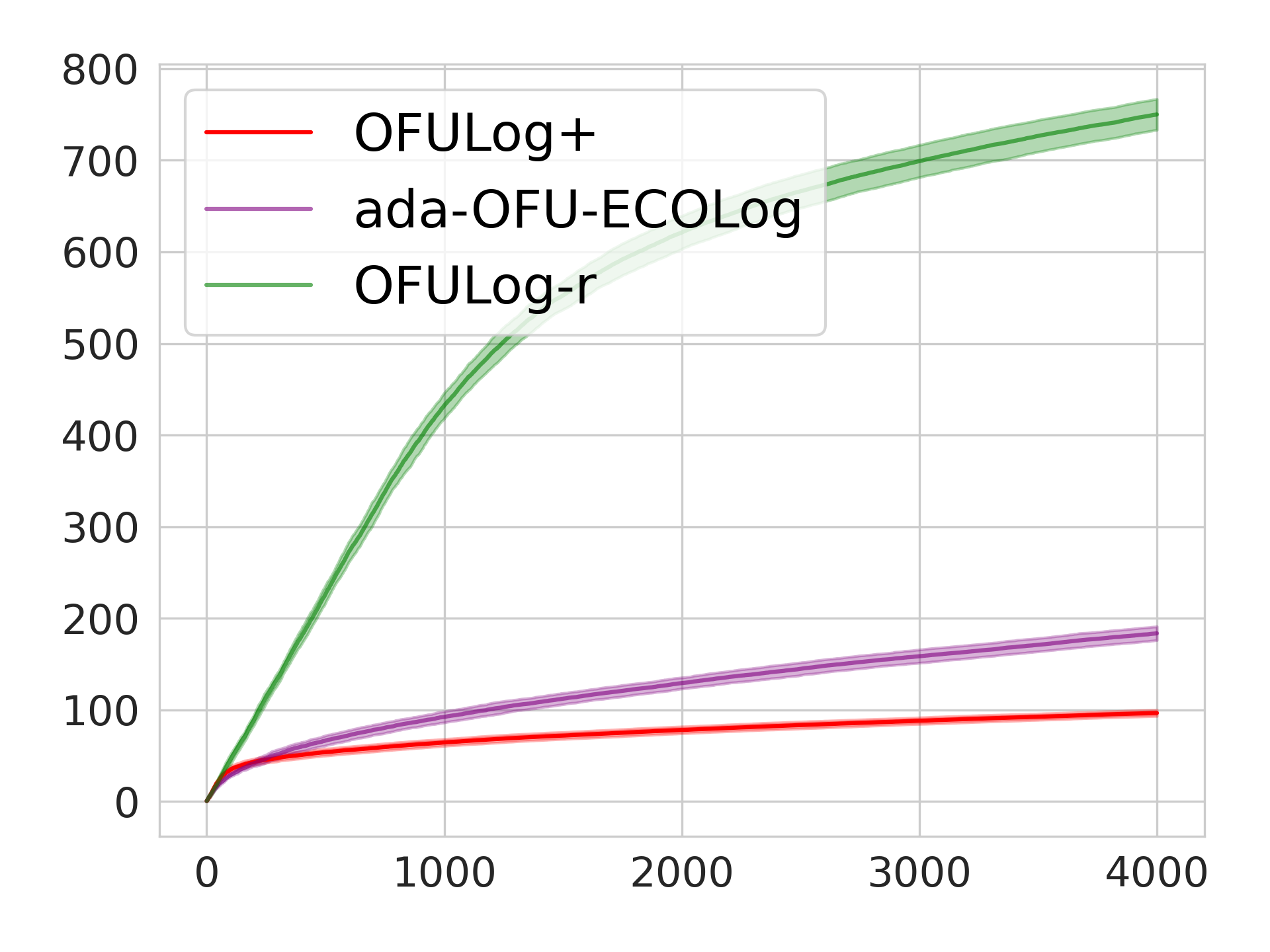}
\label{fig:logistic-a}
}
\subfigure[$S = 10$]{
\includegraphics[width=0.233\textwidth]{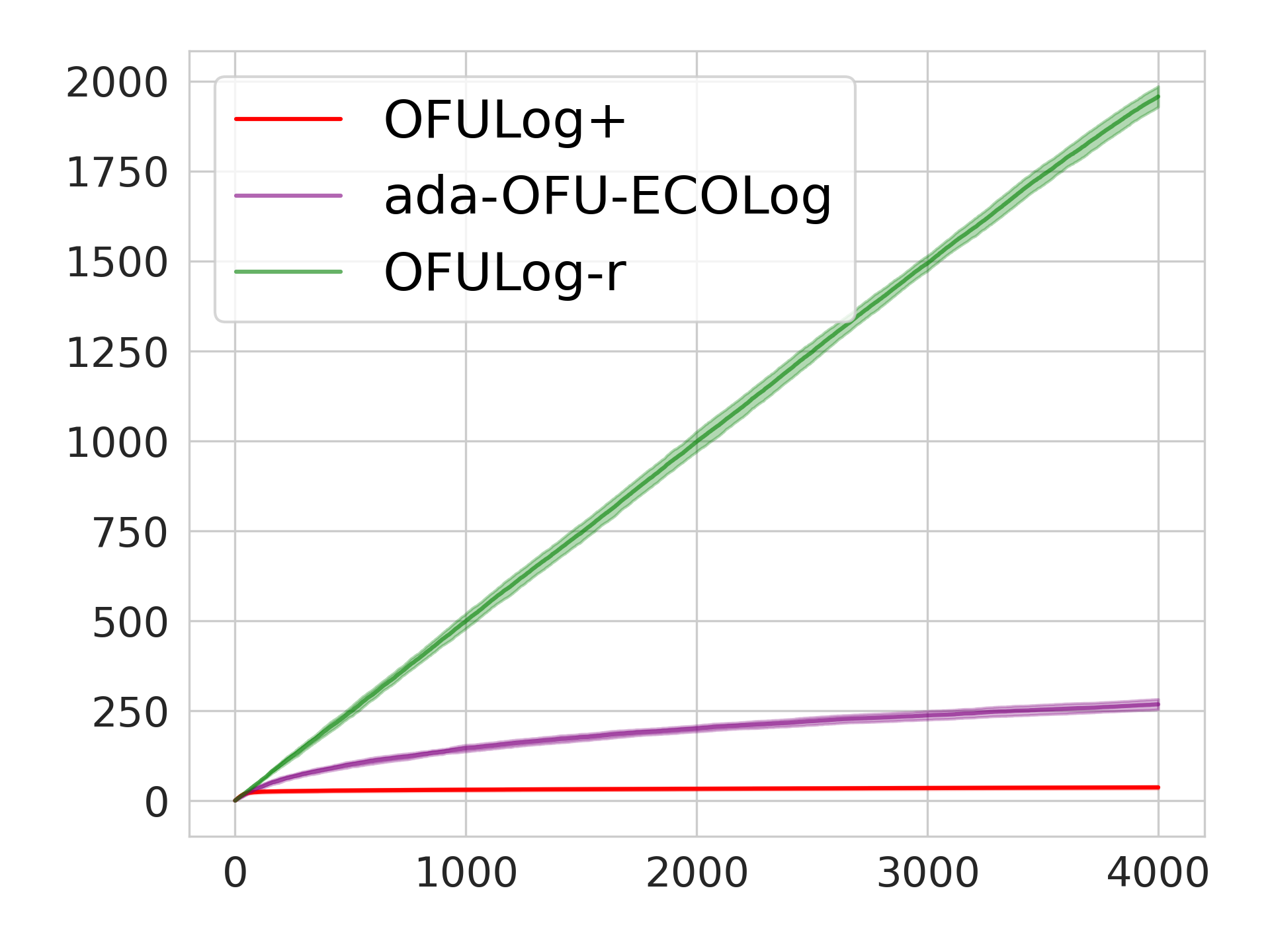}
\label{fig:logistic-b}
}
\centering
\subfigure[$S = 5$]{
\includegraphics[width=0.23\textwidth]{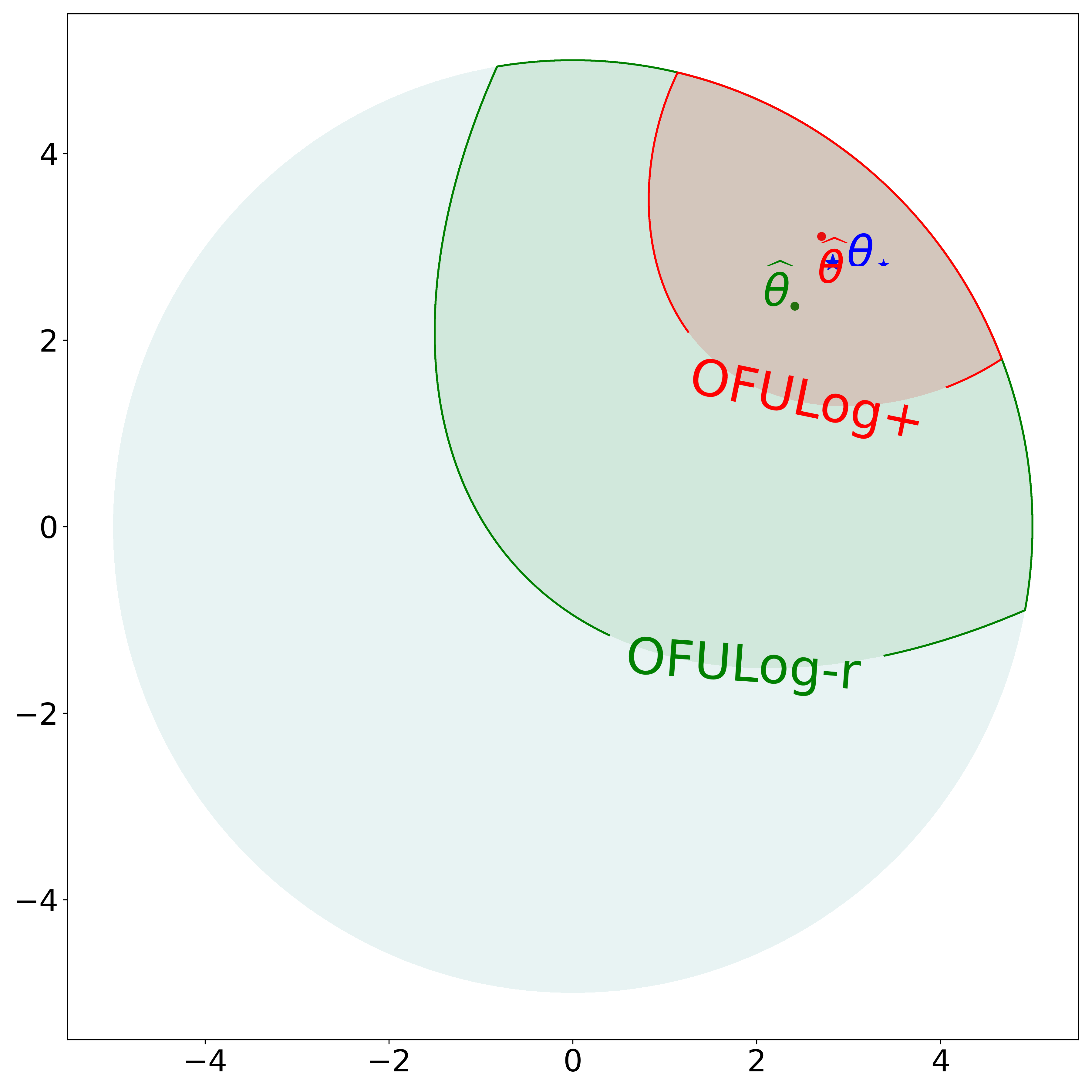}
\label{fig:logistic-c}
}
\subfigure[$S = 10$]{
\includegraphics[width=0.23\textwidth]{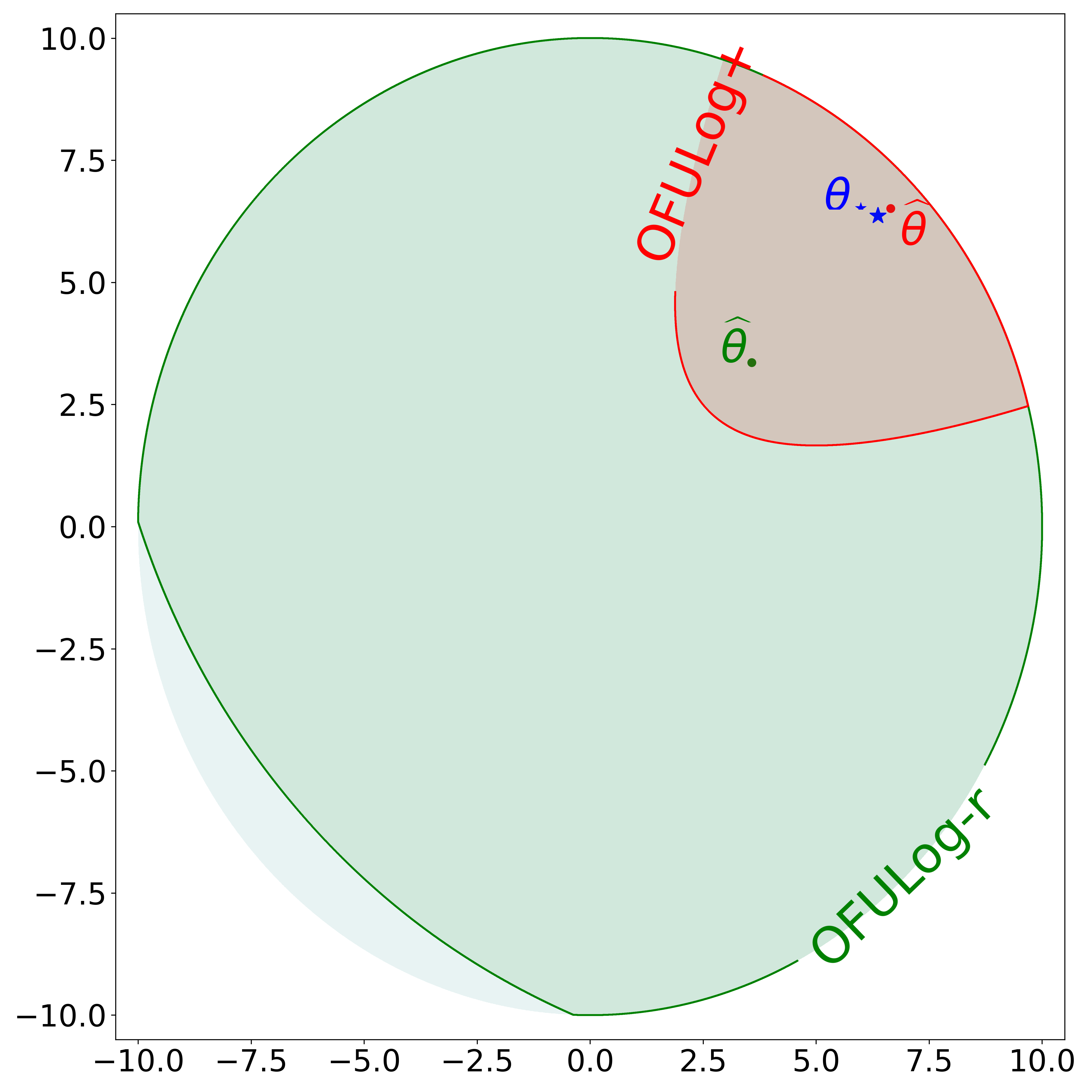}
\label{fig:logistic-d}
}
\caption{(a,b) Plot of $\Reg^B(T)$ for all considered algorithms (c,d) Confidence sets at $t = 4000$ from a single run: {\color{BrickRed} red} is from \textsc{OFULog+} and {\color{ForestGreen} green} is from \textsc{OFULog-r}.
% and {\color{Thistle} purple} is from \textsc{ada-ECO-OFULog}.
}
\label{fig:logistic}
\end{figure*}

Following \cite{amani2021mnl}, our new confidence set leads to our algorithm with an improved bonus term, \texttt{MNL-UCB+}; its pseudocode is shown in Algorithm~\ref{alg:multinomial1}.
We can improve it further with a tighter bonus term and constrained $\gC_t(\delta)$; see Algorithm~\ref{alg:multinomial2}.

For the below theorem statement, we ignore any logarithmic factors and assume that $\kappa(T)$ is very large, as it scales exponentially in $S$; see Section 3 of \cite{amani2021mnl}.
\begin{theorem}[Simplified]
\label{thm:new-regret-multinomial}
    \texttt{MNL-UCB+} and its improved version attain the following regret bounds up to logarithmic factors, respectively, w.p. $1 - \delta$:
    \begin{align*}
        \Reg^B(T) &\lesssim R d \sqrt{K S \kappa(T) T}, \\
        \Reg^B_{imp}(T) &\lesssim R d \sqrt{K} S \left( \sqrt{T} + d K^{3/2} \sqrt{S} \kappa(T) \right).
    \end{align*}
\end{theorem}
\begin{proof}
    See Theorem~\ref{thm:new-regret-multinomial-full} in Appendix~\ref{app:full-statements-multinomial} for the full statement.
    Compared to \cite{amani2021mnl}, key differences are the use of our improved confidence set and new ``multinomial'' versions of elliptical lemmas; see Appendix~\ref{app:regret-multinomial} for the full proof.
\end{proof}

\paragraph{Comparison to Prior Arts.}
Again, extending upon Table~\ref{tab:results}, we now discuss how our bound compares to existing works in detail.
To the best of our knowledge, at the time of submission, the only comparable work was \cite{amani2021mnl}.
% ; see Appendix~\ref{app:mnl} for a review of works on the combinatorial variant of MNL bandits.
There, the authors provide two bonus-based algorithms inspired by \cite{faury2020logistic}, each leading\footnote{We only consider super-logarithmic dependencies on $d, K, S, \kappa(T), T$; see Appendix~\ref{app:full-statements} for the full statement. Also, we have re-tracked the $S$-dependency with the ``optimal'' choice of $\lambda = \frac{d K^{3/2}}{S} \log\frac{1 + \frac{S T}{dK}}{\delta}$.} to the regret bounds of $\widetilde{\gO}\left( dK^{3/4} S \sqrt{\kappa(T) T} \right)$ and $\widetilde{\gO}\left( d K^{5/4} S^{3/2} \left( \sqrt{T} + d K^{5/4} S \kappa(T) \right) \right)$, respectively.
% With a closer look at the assumptions, a more realistic scenario is when $R = \sqrt{K} R'$ and $S = \sqrt{K} S'$, where $\lVert \bm\theta^{(k)} \rVert_2 \leq S'$.
We first note that even though they conjectured that $\widetilde{\gO}(dK)$ is optimal in terms of $d$ and $K$, even their regret bound with an appropriate choice of $\lambda$ results in $\widetilde{\gO}(dK^{3/4})$.
Moreover, we further improve the dependency down to $\widetilde{\gO}(d \sqrt{K})$, and even in terms of $S$, we improve by a factor of $\sqrt{S}$ in the leading term.
Recently, a concurrent work by \cite{zhang2023mnl} also made substantial improvements, both statistically and computationally.
Their \texttt{MNL-UCB+} simultaneously attains $\widetilde{\gO}\left( dS\sqrt{K \kappa(T) T} \right)$ and $\widetilde{\gO}\left( d K S \left( \sqrt{ST} + d \kappa(T) \right) \right)$, and assuming $dK \gtrsim S$ for simplicity\footnote{If $dK \lesssim S$, then we accordingly have extra $S$ dependency and less $dK$ dependency.}, their \texttt{OFUL-MLogB} simultaneously attains $\widetilde{\gO}\left( d S^{3/2} \sqrt{K \kappa(T) T} + d^2 K S^3 \kappa(T) \right)$ and $\widetilde{\gO}\left( d K S^{3/2} \sqrt{T} + d^2 K S^3 \kappa(T) \right)$, while using only $\gO(1)$ computation cost per round.
In all cases, our guarantees are strictly better by at least $\sqrt{S}$ and $\sqrt{K}$.
Still, our Algorithm~\ref{alg:multinomial2} is intractable, and we leave to future work on whether we can obtain computational efficiency while retaining our so-far optimal regret guarantees.

\section{EXPERIMENTS}
\label{sec:expr}
\paragraph{Setting.}
We consider logistic bandits and follow the experimental setting of \cite{faury2022logistic}.
We compare our \texttt{OFULog+} with \texttt{ada-OFU-ECOLog}~\citep{faury2022logistic} and \texttt{OFULog-r}~\citep{abeille2021logistic}.
% Here, \texttt{OFULog-r} refers to the tractable algorithm of \cite{abeille2021logistic} with the {\it improved} $\lambda_t = \frac{d}{S} \log\frac{St}{d\delta}$, and \texttt{OFULog-r-prev} refers to the same algorithm with the original $\lambda_t = d \log t$.
The existing implementation~\citep{faury2022logistic} utilizes only a few steps of Newton's method to approximate the MLE, which we replace with Sequential Least SQuares Programming (SLSQP) implemented in SciPy~\citep{scipy}, yielding a more precise MLE and allowing for a fairer comparison.
We also remark that their implementation does not directly reflect their theoretical algorithm, but we still use the same implementation without any modification for fairness.
Throughout the experiments, we fix $T = 4000$, $d = 2$, $|\gA| = 20$, and $\delta = 0.05$.
We use ${\color{blue}\bm\theta_\star} = \frac{S - 1}{\sqrt{d}} \bm1$ for $S \in \{5, 10\}$, and time-varying arm-set by sampling in the unit ball at random at each $t$.
For \texttt{ada-OFU-ECOLog}, we set $\lambda = 10$.
The codes are available in our \href{https://github.com/nick-jhlee/logistic_bandit}{GitHub repository}.
% \footnote{\url{https://github.com/nick-jhlee/logistic_bandit}}.

\paragraph{Results.}
The regret curves averaged over $10$ independent runs are shown in Figure~\ref{fig:logistic-a} and \ref{fig:logistic-b}, where it is clear that \texttt{OFULog+} is the best.
The confidence sets at $t = 4000$ for \texttt{OFULog-r} and \texttt{OFULog+} are shown in Figure~\ref{fig:logistic-c} and \ref{fig:logistic-d}, where we note how our MLE estimate ${\color{red} \widehat{\bm\theta}}$ is the closest to ${\color{blue}\bm\theta_\star}$, and that our confidence set is the smallest.
There are some interesting observations to be made.
First, even though \texttt{ada-OFU-ECOLog} shares the same leading term in theoretical regret as ours, numerically, \texttt{OFULog+} still outperforms by a large margin.
Second, for $S=5$ (or generally, for small $S$), \texttt{ada-OFU-ECOLog} attains better numerical regret than \texttt{OFULog+} in the {\it initial} phase, but then becomes worse in the later phase.
We believe that is due to explicit regularization of \texttt{ada-OFU-ECOLog}, which helps initially but later forces the MLE estimate to be bounded.

\vspace{-0.3cm}
\section{CONCLUSION}
In this paper, we propose regret-to-confidence-set conversion (R2CS) that converts an online learning regret guarantee to a new confidence set, without the need to run the online algorithm explicitly.
Using a novel combination of self-concordant control and information-geometric interpretation of KL-divergence as well as new martingale concentration arguments, we proved new confidence sets for logistic and MNL bandits, leading to the state-of-the-art regret bounds with improved dependencies on $S$ and $K$.

One crucial and exciting future direction is to extend our R2CS to various other settings such as improved Thompson-Sampling for logistic bandits~\citep{abeille2017thompson,faury2022logistic}, generalized linear bandits~\citep{filippi2010glm,mutny2021poisson}, norm-agnostic scenario~\citep{gales2022norm}, and even multinomial logistic MDP~\citep{hwang2023mnl}.
Another direction is to improve the Bradley-Terry model-based RLHF, which is similar to logistic bandits~\citep{wu2024preference,das2024rlhf}.
% \kj{why is it equivalent to logistic bandits? $\implies$ tone down to ``similar''}

\clearpage

\subsubsection*{Acknowledgements}
We thank the anonymous reviewers for their helpful and insightful comments.
% \kj{I am okay with not mentioning the errors since they were not critical, but it's up to you.}
J. Lee and S.-Y. Yun were supported by the Institute of Information \& Communications Technology Planning \& Evaluation (IITP) grants funded by the Korean government (MSIT) (No.2022-0-00311, Development of Goal-Oriented Reinforcement Learning Techniques for Contact-Rich Robotic Manipulation of Everyday Objects; No.2019-0-00075, Artificial Intelligence Graduate School Program (KAIST)).
K.-S. Jun was supported in part by the National Science Foundation under grant CCF-2327013.

\bibliographystyle{plainnat}
\bibliography{references}

\newpage
\onecolumn
\appendix

\newpage
\tableofcontents
\newpage

\fancypagestyle{plain}{
\fancyfoot[C]{\thepage}}
\pagestyle{plain}
\setlength{\footskip}{20pt}	
% \pagenumbering{arabic} % Arabic/Indic page numbers

\section{FURTHER RELATED WORK}

\subsection{Online-to-Something Conversion}
\label{app:o2cs}
\paragraph{Online-to-Confidence Set.}
Recently, many results have connected online learning to the concentration of measure, starting from~\cite{rakhlin17on}, followed by~\cite{jun19parameterfree,orabona21tight}, which is also closely related to the ``reduction'' framework championed by John Landford\footnote{\url{https://hunch.net/~jl/projects/reductions/reductions.html}} and later followed upon in \cite{foster2018oracle,foster2020oracle}.

For linear models, there are two main categories of techniques for building confidence sets based on online learning algorithms.
The first is to leverage the negative term $-\|\widehat{\bm\theta}_{T+1} - \bm\theta^\star\|^2_{\mV_{T}}$ from the regret bound of online Newton step (ONS)~\citep{hazan07logarithmic} where $\mV_T := \lambda \mI + \sum_{t=1}^T \vx_t \vx_t^\intercal$ and $\widehat{\bm\theta}_{T+1}$ is the parameter predicted at the time step $T+1$.
This way, one can construct a confidence set centered at $\widehat{\bm\theta}_{T+1}$ with a confidence radius that depends on the rest of the terms in the regret bound~\citep{Dekel10robust,dekel12selective,crammer13multiclass,gentile14onmultilabel,zhang16online}.
The second one, which is dubbed as {\it online-to-confidence-set conversion (O2CS)}, is to start from the regret bound $\sum_{t=1}^T \ell_t(\bm\theta_t) - \ell_t(\bm\th^\star) \le B_T$ where $\ell_t$ is a properly defined loss function (e.g., squared loss), $\bm\theta_t$ is the parameter predicted at time $t$, and $B_T$ is the regret bound of the algorithm.
We then lower bound its left-hand side with a standard concentration inequality, which results in a quadratic constraint on $\bm\th^\star$~\citep{abbasiyadkori2012conversion,jun2017conversion}.
% \kj{I removed foster et al since it does not provide a confidence set}
While this itself defines a confidence set for $\bm\th^\star$, one can further manipulate the quadratic constraint into a confidence set centered at a new estimator that regresses on the prediction $\hat y_t$'s from the online learning algorithm rather than the actual label $y_t$'s.
The benefit of O2CS over the ONS-based one is that we are not married to the particular algorithm of ONS but are open to using any online learning algorithm, and thus ``progress in constructing better algorithms for online prediction problems directly translates into tighter confidence sets''~\citep{abbasiyadkori2012conversion}; see also Chapter 23.3 of \cite{banditalgorithms}.

However, these two techniques have one fundamental difference from our proposed R2CS: they require running the online learning algorithm directly, whereas R2CS relies only on knowing an achievable regret bound without actually running it.
This means that our R2CS establishes a third category of techniques for building confidence sets based on online learning algorithms.

\paragraph{Online-to-PAC.}
Our R2CS also has a strong resemblance to the online-to-PAC conversion~\citep{lugosi2023conversion}, which shows that an achievable regret for the so-called {\it generalization game} implies a bound on the generalization error that holds for all statistical learning algorithms, uniformly, up to some martingale concentration term.
This is quite similar to our O2CS framework, except the quantity under interest for us is the confidence set of some unknown parameter or model, which is different from the generalization error.

% \textit{existence} of a regret bound, so all we need for the construction of the confidence set is the regret bound.

%Let $V_T = \lambda I + \sum_{t=1}^T x_t x_t^\T$ and $\hat\theta_T$ be the parameter predicted by the online Newton step (ONS) algorithm~\ref{hazan17logarithmic} at time step $T+1$.
%
% the negative term $-\|\hat\theta_T - \theta^*\|^2_{V_{T}}$ from the regret bound of online Newton step (ONS)~\ref{hazan17logarithmic}, which is 

\subsection{Likelihood Ratio Confidence Sets}
\label{app:likelihood}
Although our paper is focused on bandits, our ``loss-based'' confidence sets (Theorem~\ref{thm:confidence-logistic}, \ref{thm:confidence-multinomial}) are based on some likelihood ratio testing.
Despite being around for more than 50 years since the seminal work of \cite{robbins1972class}, the statistics community and especially the field of safe anytime-valid inference (SAVI) has recently revived the interest in LRCS and hypothesis testing procedures due to their elegancy and many desirable properties such as being ``universal''~\citep{wasserman2020universal} and anytime valid~\citep{ramdas2022anytime}.
The general idea is that as the sequential likelihood ratio process (SLRP) is super-martingale, one can utilize Ville's inequality~\citep{Ville1939} to obtain a time-uniform confidence sequence~\citep{wasserman2020universal,emmenegger2023likelihood}; see \cite{ramdas2022anytime} for a more detailed overview of this subject from a statistics perspective.
Recently, \cite{emmenegger2023likelihood} proposes weighted SLRP-based confidence set w.r.t. some sequence of estimators $\{\widehat{\bm\theta}_t\}$, which is chosen as outputs of the follow-the-regularized-leader (FTRL) for a naturally-derived online prediction game, and an adaptive reweighting based on the bias of the estimator $\widehat{\bm\theta}_t$ to reduce the variance.
They then instantiate their confidence set for generalized linear models, and by leveraging a deep connection between the Bregman divergence geometry and Bregman information gain~\citep[Theorem 3]{chowdhury2023bregman}, they quantitatively analyzed the geometry of their confidence set.

Indeed, there is a strong resemblance between our R2CS and \cite{emmenegger2023likelihood}.
\cite{emmenegger2023likelihood} sticks to the {\it sequential} likelihood ratio testing (SRLT), $\gL_t(\bm\theta) - \gL_t(\{\widehat{\bm\theta}_s\}_{s=1}^t)$, where $\gL_t(\cdot)$ is some log-likelihood.
Our R2CS also starts with SLRT, but we then convert it to a form of batched likelihood ratio testing, $\gL_t(\bm\theta) - \gL_t(\widehat{\bm\theta}_t)$ by leveraging some online learning regret.
Investigating further into the deep connections between R2CS and \cite{emmenegger2023likelihood} and even aforementioned related works on SAVI is an exciting future direction.
% \kj{I tend to think that we know the relationship -- we are actually starting off with the sequential likelihood ratio testing, but then make it a batch likelihood ratio via a regret bound. }

\subsection{Optimism-based Approaches to Linear Bandits}
\label{app:linear}
We briefly review the optimism-based approaches to linear bandits and some recent advances.
``Optimism in the face of uncertainty'' (OFU) is a powerful principle in sequential decision-making that operates by choosing actions in the most optimistic way possible while being sufficiently plausible.
For bandits especially, this amounts to constructing an anytime-valid confidence sequence of some models, where the radius of each confidence set corresponds to the amount of uncertainty at a given time.
The seminal work by \cite{abbasiyadkori2011linear} shows that for linear bandits, one can construct such sequence using the celebrated self-normalized martingale concentrations~\citep{delepena2004self} and that the regret can be bounded as roughly the sum of radii of the confidence sets over all the timesteps.
There has been much effort to improve the confidence set of linear bandits; most recently, \cite{flynn2023linearbandit} proposed a set of confidence sequences that can be constructed via adaptive martingale mixtures.
Equally as important, a misspecified choice of confidence set radius can be catastrophic.
Recently, there has been some work on tackling this issue as well.
\cite{gales2022norm} considered the norm-agnostic scenario, and \cite{kim2022variance,jun2024variance} considered the variance-agnostic scenario.

Recall that most logistic and MNL bandits literature, including ours, is OFU-style. One notable distinction of our R2CS framework is that we do not utilize self-normalized martingale concentrations, which had to be modified for logistic and MNL losses~\citep{faury2020logistic,amani2021mnl}.
It would be interesting to take the recent advancements in linear bandits mentioned above and extend them to logistic and MNL bandits.

\subsection{Multinomial Logistic (MNL) Bandits}
\label{app:mnl}
There are two lines of work in multinomial logistic (MNL) bandits.
One line of work, closely related to ours and which we have discussed extensively in the main text, considers $K + 1$ outcomes modeled by the multinomial logit model, a multinomial extension of \cite{faury2020logistic}.
There are only two relevant works in this line so far.
\cite{amani2021mnl} proposed two UCB-based algorithms, one of which is intractable due to the complex nature of its confidence set.
\cite{zhang2023mnl} then proposed two algorithms, one that is UCB-based with improved confidence set, and another that is jointly efficient and regret-effective in the style of \cite{faury2022logistic} with better computation cost, $\gO(1)$ per round.
Notably, they also use an online-to-confidence-set conversion type argument, with some appropriate modifications.
Another line of work considers a combinatorial bandit-type extension for assortment selection problem from choice model theory~\citep{agrawal2023mnl,oh2021mnl}.
Here, their considered setting fundamentally differs from ours in that the learner chooses an assortment (a subset of indices) $\gQ_t$, from which the reward follows the multinomial logit distribution over $\gQ_t$.

\subsection{Generalized Linear Bandits}
Generalized linear (GL) bandit, which is a generalization of the logistic bandits by replacing the logistic link with a general exponential family link, was introduced by the seminal work of \cite{filippi2010glm}, in which they also proposed an optimistic algorithm.
Other than the advances in logistic bandits, as surveyed in the main text, there were also significant advances in the GL bandits.
Inspired by online Newton step~\citep{hazan07logarithmic}, several works have proposed efficient and online algorithms for generalized linear bandits~\citep{li2017glm,jun2017conversion}.
Thompson sampling-style algorithms~\citep{russo2018thompson} have also been studied extensively for logistic bandits and generalized linear bandits~\citep{kveton2020glm,abeille2017thompson,kim2023thompson}.
\cite{kazerouni2021glm} studied the problem of best arm identification for the generalized linear bandits.
\cite{russac2021glm} considers a (piecewise) non-stationary GL bandit and proposes an algorithm with forgetting.
\cite{oh2021sparse,li2022sparse} considered a high-dimensional variant of GL bandits with sparsity.
\cite{kang2022glmbilinear} recently extended the GL bandit setting to generalized low-rank matrix bandits in which the arm-set becomes the low-rank matrix manifold.

\newpage
\section{MISSING RESULTS}
\label{app:results}
In this section, we provide two missing results from the main text.

\subsection{Ville's Inequality}
\label{app:ville}
We used a martingale version of Markov's inequality in the proof of Lemma~\ref{lem:freedman}, known as Ville's inequality.
Here's the full statement:
\begin{lemma}[Th\'{e}or\`{e}me 1 of pg. 84 of \cite{Ville1939}]\label{lem:ville}
    Let $X_n$ be a nonnegative supermartingale.
    Then, for any $\lambda > 0$, $\sP\left[ \sup_{n \geq 0} X_n \geq \lambda \right] \leq \frac{\E[X_0]}{\lambda}$.
    % \begin{equation*}
    %     \sP\left[ \sup_{n \geq 0} X_n \geq \lambda \right] \leq \frac{\E[X_0]}{\lambda}.
    % \end{equation*}
\end{lemma}
This is known to be essentially tight; see \cite{howard2020timeuniform} for further discussions.

A fun historical note: this is also commonly known as the {\it Doob's maximal inequality}, but historically, Jean Ville was the first to report this in literature in his 1939 thesis~\citep{Ville1939}.
Interestingly, despite finding Ville's writing style lacking in his review of the book~\citep{Doob1939}, Joseph L. Doob recognized the significance of the result it presented, as evidenced by his later work~\citep{Doob1940}.

\subsection{``Outputs'' from Algorithm 1 of \cite{foster2018logistic}}
\label{app:rmk-proof}
The following proposition justifies using the improper learning algorithm of \cite{foster2018logistic} for our purpose (specifically, the existence of $\tilde{\bm\theta}_s$ for logistic bandits and $\widetilde{\bm\Theta}_s$ for multinomial logistic bandits):
\begin{proposition}
\label{prop:surjection}
    Consider a softmax function $\bm\sigma : \sR^K \rightarrow \Delta_{>0}^{K+1}$ defined as $\bm\sigma(\vz)_k = \frac{e^{z_k}}{1 + \sum_{k' \in [K]} e^{z_{k'}}}$ for $k \in [K]$ and $\bm\sigma(\vz)_0 = \frac{1}{1 + \sum_{k' \in [K]} e^{z_{k'}}}$.
    Then, for any $\vx \in \gB^d(1)$ and $\hat{\vz} \in \sR^{K+1}$ outputted from Algorithm 1 of \cite{foster2018logistic} (see their line 4), there exists $\bm\Theta = [\bm\theta^{(1)} | \cdots | \bm\theta^{(K)}]^\intercal \in \gB^{K \times d}(\sqrt{K}S)$ s.t. $\bm\sigma(\hat{\vz}) = \bm\sigma\left( (\langle \vx, \bm\theta^{(1)} \rangle, \cdots, \langle \vx, \bm\theta^{(K)} \rangle) \right)$.
\end{proposition}
\begin{proof}
    From line 4 of Algorithm 1 of \cite{foster2018logistic} with $\mu = 0$, we have that for some distribution $P_t$ whose support is $\gS := \left(\gB^d(S)\right)^{\otimes K}$ (set of $K \times d$ matrices where the norm of each row is bounded by $S$),\footnote{The softmax considered in \cite{foster2018logistic} is actually of the form $\bm\sigma(\vz)_{k'} = \frac{e^{z_{k'}}}{\sum_{l \in \{0\} \cup [K]} e^{z_l}}$ for $k' \in \{0\} \cup [K+1]$. By dividing the denominator and numerator by $e^{z_0}$ and recalling that $z_k = \langle \vx, \bm\theta^{(k)} \rangle$, by triangle inequality, it can be seen that our parameter space, $\gS$, and the parameter space of \cite{foster2018logistic} with $B = S/2$, $\left(\gB^d(S/2)\right)^{\otimes (K+1)}$, are equivalent. In the notation of \cite{foster2018logistic}, we set $B = S/2$.}
    \begin{equation*}
        \bm\sigma(\hat{\vz}) = \E_{\bm\Theta \sim P_t}\left[ \bm\sigma(\bm\Theta \vx) \right].
    \end{equation*}
    
    Define $F : \gS \rightarrow \Delta^{K+1}_{>0}$ to be $F(\bm\Theta) = \bm\sigma(\bm\Theta \vx)$, which is continuous.
    We have the following two lemmas:
    \begin{lemma}
    \label{lem:expectation}
        Let $(\gX, P)$ be a probability space with the usual Borel $\sigma$-algebra, $Y \subset \gH$ be a compact, convex subset of a separable, Hilbert space $\gH$, and $F : \gX \rightarrow Y$ be (Bochner) measurable.
        Then, for any random variable $X$ on $\gX$, we have that $\E[F(X)] \in Y$.
    \end{lemma}
    \begin{lemma}
    \label{lem:convex}
        $\mathrm{conv}\left(F(\gS)\right) \subseteq F(\gB^{K \times d}(\sqrt{K} S))$, where $\mathrm{conv}(\cdot)$ is the convex hull operator.
    \end{lemma}
    % there exists a $\bm\Theta \in \gB^{K \times d}(S)$ such that $\bm\sigma(\hat{\vz}) = \bm\sigma(\bm\Theta \vx)$.

    The proof then concludes as the following:
    by the above two lemmas, we have that $\bm\sigma(\hat{\vz}) = \E[F(\bm\Theta)] \in F(\gB^{K \times d}(\sqrt{K}S))$, i.e., there exists $\bm\Theta \in \gB^{K \times d}(\sqrt{K} S)$ such that $\bm\sigma(\hat{\vz}) = F(\bm\Theta)$.
\end{proof}

\subsubsection{Proof of Lemma~\ref{lem:expectation}}
(The proof here is inspired by an old \href{https://math.stackexchange.com/questions/1576097/integral-0f-a-function-with-range-in-convex-set-is-in-the-convex-set}{StackExchange post}. Also, see e.g., \cite{functional} for the necessary background on functional analysis.)

It is clear that $\E[F(X)]$ exists.
The proof now proceeds via {\it reductio ad absurdum}, i.e., suppose that $e \triangleq \E[F(X)] \not\in Y$.
Then, as $\{e\}$ and $Y$ are disjoint, compact, and convex sets in a separable Hilbert space, by the Hahn-Banach Separation Theorem and Riesz Representation Theorem, there exists a $v \in \gH$ such that
\begin{equation*}
    \langle v, F(x) \rangle < \langle v, e \rangle, \quad \forall x \in \gX.
\end{equation*}
Then, we have that
\begin{equation*}
    \int_\gX \langle v, F(x) \rangle dP(x)
    = \left\langle v, \int_\gX F(x) dP(x) \right \rangle
    = \langle v, e \rangle
    < \langle v, e \rangle,
\end{equation*}
a contradiction.
\hfill\qedsymbol

\subsubsection{Proof of Lemma~\ref{lem:convex}}
Let $\bm\Theta_1, \bm\Theta_2 \in \gS$ and $\lambda \in [0, 1]$.
We will show that $\lambda F(\bm\Theta_1) + (1 - \lambda) F(\bm\Theta_2) \in F(\gB^{K \times d}(\sqrt{K} S))$.

First, for some given $\vp = (p_1, \cdots, p_K)^\intercal$, we show that there exists $\bm\Theta = [\bm\theta^{(1)} | \cdots | \bm\theta^{(K)}]^\intercal$ that satisfies the following system of equations: for each $k \in [K]$,
\begin{equation*}
    \frac{\exp\left( \langle \vx, \bm\theta^{(k)} \rangle \right)}{1 + \sum_{k' \in [K]} \exp\left( \langle \vx, \bm\theta^{(k')} \rangle \right)} = p_k.
\end{equation*}
Denoting $\alpha_k := \exp\left( \langle \vx, \bm\theta^{(k)} \rangle \right)$, above can be rearranged to the following system of equations:
\begin{equation*}
    \underbrace{\begin{bmatrix}
        1 - p_1 & -p_1 & \cdots & -p_1 \\
        -p_2 & 1 - p_2 & \cdots & -p_2 \\
        \vdots & \vdots & \vdots & \vdots \\
        -p_K & -p_K & \cdots & 1 - p_K
    \end{bmatrix}}_{\triangleq \mC_K}
    \begin{bmatrix}
        \alpha_1 \\ \alpha_2 \\ \vdots \\ \alpha_K
    \end{bmatrix}
    =
    \begin{bmatrix}
        p_1 \\ p_2 \\ \vdots \\ p_K
    \end{bmatrix}.
\end{equation*}
From simple computation, one can easily see that
\begin{equation*}
    \mC_K^{-1}
    = \frac{1}{p_0} \vp \bm1^\intercal + \mI_K,
\end{equation*}
% \kj{does this mean we have a rank one matrix on the RHS, that is equal to the inverse of a $K$ by $K$ matrix?}
where we recall that $p_0 = 1 - \sum_{k=1}^K p_k$.
This gives a unique solution
\begin{equation*}
    \alpha_k^* = \frac{p_k}{p_0} > 0.
\end{equation*}
Then, we arrive at another system of linear equations: $\vx^\intercal \bm\theta^{(k)} = \log \alpha_k^*$ for each $k \in [K]$.
One can easily see that $\bm\theta^{(k)} = \frac{\log \alpha_k^*}{\lVert \vx \rVert_2} \vx$ satisfies the system.

All in all, we showed that there exists a $\bm\Theta$ such that $\lambda F(\bm\Theta_1) + (1 - \lambda) F(\bm\Theta_2) = F(\bm\Theta)$ and
\begin{equation*}
    \lVert \bm\Theta \rVert_F^2 = \sum_{k \in [K]} \left( \log\alpha_k^* \right)^2,
\end{equation*}
where in our case,
\begin{equation*}
    p_k = \lambda \frac{\exp\left( \langle \vx, \bm\theta_1^{(k)} \rangle \right)}{1 + \sum_{k' \in [K]} \exp\left( \langle \vx, \bm\theta_1^{(k')} \rangle \right)} + (1 - \lambda) \frac{\exp\left( \langle \vx, \bm\theta_2^{(k)} \rangle \right)}{1 + \sum_{k' \in [K]} \exp\left( \langle \vx, \bm\theta_2^{(k')} \rangle \right)}.
\end{equation*}
Then,
\begin{align*}
    \frac{p_k}{p_0} &= \frac{\lambda \frac{\exp\left( \langle \vx, \bm\theta_1^{(k)} \rangle \right)}{1 + \sum_{k' \in [K]} \exp\left( \langle \vx, \bm\theta_1^{(k')} \rangle \right)} + (1 - \lambda) \frac{\exp\left( \langle \vx, \bm\theta_2^{(k)} \rangle \right)}{1 + \sum_{k' \in [K]} \exp\left( \langle \vx, \bm\theta_2^{(k')} \rangle \right)}}{\lambda \frac{1}{1 + \sum_{k' \in [K]} \exp\left( \langle \vx, \bm\theta_1^{(k')} \rangle \right)} + (1 - \lambda) \frac{1}{1 + \sum_{k' \in [K]} \exp\left( \langle \vx, \bm\theta_2^{(k')} \rangle \right)}} \\
    &\leq \frac{\lambda \frac{e^S}{1 + \sum_{k' \in [K]} \exp\left( \langle \vx, \bm\theta_1^{(k')} \rangle \right)} + (1 - \lambda) \frac{e^S}{1 + \sum_{k' \in [K]} \exp\left( \langle \vx, \bm\theta_2^{(k')} \rangle \right)}}{\lambda \frac{1}{1 + \sum_{k' \in [K]} \exp\left( \langle \vx, \bm\theta_1^{(k')} \rangle \right)} + (1 - \lambda) \frac{1}{1 + \sum_{k' \in [K]} \exp\left( \langle \vx, \bm\theta_2^{(k')} \rangle \right)}} \tag{$\bm\Theta_i \in \gS$, i.e., $\left\lVert \bm\theta_i^{(k)} \right\rVert_2 \leq S$ for each $k \in [K]$} \\
    &= e^S,
\end{align*}
and thus,
\begin{align*}
    \lVert \bm\Theta \rVert_F^2 &\leq K S^2.
\end{align*}
\hfill\qedsymbol
\newpage
\section{PROOFS - LOGISTIC BANDITS}
\subsection{Notations}
Recall from the main text that $\gL_t(\bm\theta) := \sum_{s=1}^t \ell_s(\bm\theta)$ is the cumulative {\it unregularized} logistic loss up to time $t$.
We also consider the following quantities~\citep{abeille2021logistic}:
\begin{align}
    % \alpha(\vx, \bm\theta_1, \bm\theta_2) &:= \int_0^1 \dot{\mu}\left( \vx^\intercal (\bm\theta_1 + v(\bm\theta_2 - \bm\theta_1)) \right) dv \\
    % \mG_t(\bm\theta_1, \bm\theta_2) &:= \sum_{s=1}^{t-1} \alpha(\vx_s, \bm\theta_1, \bm\theta_2) \vx_s \vx_s^\intercal + \lambda_t \mI_d \\
    \widetilde{\alpha}(\vx, \bm\theta_1, \bm\theta_2) &:= \int_0^1 (1 - v) \dot{\mu}\left( \vx^\intercal (\bm\theta_1 + v(\bm\theta_2 - \bm\theta_1)) \right) dv \\
    \widetilde{\mG}_t(\bm\theta_1, \bm\theta_2) &:= \sum_{s=1}^{t-1} \widetilde{\alpha}(\vx_s, \bm\theta_1, \bm\theta_2) \vx_s \vx_s^\intercal + \lambda_t \mI_d \\
    \mH_t(\bm\theta) &:= \sum_{s=1}^{t-1} \dot{\mu}(\vx_s^\intercal \bm\theta) \vx_s \vx_s^\intercal + \lambda_t \mI_d,
\end{align}
where $\lambda_t > 0$ is to be determined, and the following problem-dependent constants:
\begin{equation}
    \kappa_\star(T) := \frac{1}{\frac{1}{T} \sum_{t=1}^T \dot{\mu}(\vx_{t,\star}^\intercal \bm\theta_\star)},
    \ \kappa_{\gX}(T) := \max_{t \in [T]} \max_{\vx \in \gX_t} \frac{1}{\dot{\mu}(\vx^\intercal \bm\theta_\star)},
    \ \kappa(T) := \max_{t \in [T]} \max_{\vx \in \gX_t} \max_{\bm\theta \in \gB^d(S)} \frac{1}{\dot{\mu}(\vx^\intercal \bm\theta)},
\end{equation}
where $\vx_{t,\star} = \argmax_{\vx \in \gX_t} \mu(\langle \vx, \bm\theta_\star \rangle)$ is the optimal action at time $t$.
Also, we overload the notation and define $A \lesssim B$ to be when we have $A \leq c B$ for some {\it universal} constant $c$, not ignoring logarithmic factors.

\subsection{Full Theorem Statements for Regret Bounds}
\label{app:full-statements}
We provide full theorem statements for our regret and prior arts for logistic bandits.
We start by providing the regret bound of our \texttt{OFULog+}:
\begin{theorem}
\label{thm:new-regret-logistic-full}
    \texttt{OFULog+} attains the following regret bound:
    \begin{equation*}
        \Reg^B(T) \lesssim R_{leading}(T) + R_{log}(T) + R_{detr}(T),
    \end{equation*}
    where w.p. at least $1 - \delta$,
    \begin{align*}
        R_{leading}(T) &:= \left( dS \log\left(e + \frac{ST}{d}\right) + \sqrt{d}S \log\frac{1}{\delta} \right) \sqrt{\frac{T}{\kappa_\star(T)}}, \\
        R_{log}(T) &:= d^2 S^2 \left( \log\left(e + \frac{ST}{d}\right) \right)^2 + d S^2 \left( \log \frac{1}{\delta} \right)^2, \\
        R_{detr}(T) &:= \min\left\{ \kappa_\gX(T) R_{log}(T), S \sum_{t=1}^T \mu(\vx_{t,\star}^\intercal \bm\theta_\star) \mathds{1}[\vx_t \in \gX_-(t)] \right\},
    \end{align*}
    where $\gX_-(t)$ is the set of detrimental arms at time $t$ as defined in \cite{abeille2021logistic}.
\end{theorem}

We now provide the prior state-of-the-art regret bounds that we compare ourselves to:
\begin{theorem}[Theorem 1 of \cite{abeille2021logistic}]
\label{thm:old-regret-logistic-full1}
    \texttt{OFULog} with $\lambda_t = \frac{d}{S} \log\frac{St}{d \delta}$ attains the following regret bound:
    \begin{equation*}
        \Reg^B(T) \leq R_{leading}(T) + R_{log}(T) + R_{detr}(T),
    \end{equation*}
    where w.p. at least $1 - \delta$,
    \begin{align*}
        R_{leading}(T) &\lesssim d S^{\frac{3}{2}} (\log T) \left( \log\left(1 + \frac{ST}{d}\right) + \log\frac{1}{\delta} \right) \sqrt{\frac{T}{\kappa_\star(T)}}, \\
        R_{log}(T) &\lesssim d^2 S^3 \left( \log T \right)^2 \left( \log\left(1 + \frac{ST}{d}\right) + \log\frac{1}{\delta} \right)^2, \\
        R_{detr}(T) &\lesssim \min\left\{ \kappa_\gX(T) R_{log}(T), S \sum_{t=1}^T \mu(\vx_{t,\star}^\intercal \bm\theta_\star) \mathds{1}[\vx_t \in \gX_-(t)] \right\}.
    \end{align*}
\end{theorem}
\begin{theorem}[Theorem 2 of \cite{abeille2021logistic}]
\label{thm:old-regret-logistic-full2}
    \texttt{OFULog-r} with $\lambda_t = \frac{d}{S} \log\frac{St}{d \delta}$ attains the following regret bound:
    \begin{equation*}
        \Reg^B(T) \leq R_{leading}(T) + R_{log}(T) + R_{detr}(T),
    \end{equation*}
    where w.p. at least $1 - \delta$,
    \begin{align*}
        R_{leading}(T) &\lesssim d S^{\frac{5}{2}} (\log T) \left( \log\left(1 + \frac{ST}{d}\right) + \log\frac{1}{\delta} \right) \sqrt{\frac{T}{\kappa_\star(T)}}, \\
        R_{log}(T) &\lesssim d^2 S^4 \left( \log T \right)^2 \left( \log\left(1 + \frac{ST}{d}\right) + \log\frac{1}{\delta} \right)^2, \\
        R_{detr}(T) &\lesssim \min\left\{ \kappa_\gX(T) R_{log}(T), S \sum_{t=1}^T \mu(\vx_{t,\star}^\intercal \bm\theta_\star) \mathds{1}[\vx_t \in \gX_-(t)] \right\}.
    \end{align*}
\end{theorem}
\begin{theorem}[Theorem 2 of \cite{faury2022logistic}]
\label{thm:old-regret-logistic-full3}
    \texttt{ada-OFU-ECOLog} attains the following w.p. $1 - \delta$:
    \begin{equation*}
        \Reg^B(T) \lesssim dS \log\frac{1}{\delta} \sqrt{\frac{T}{\kappa_\star(T)}} + d^2 S^6 \kappa \left( \log\frac{1}{\delta} \right)^2.
    \end{equation*}
\end{theorem}
% \kj{We kind of showed a bunch of existing results without much remarks here. do we want to highlight a few interesting things about those results above?}

Lastly, although incomparable to our setting, for completeness, we provide the regret bound as provided in \cite{mason2022logistic} for fixed arm-set setting:
\begin{theorem}[Theorem 2 and Corollary 3 of \cite{mason2022logistic}]
\label{thm:old-regret-logistic-full4}
    \texttt{HOMER} with the naive warmup attains the following w.p. $1 - \delta$:
    \begin{equation*}
        \Reg^B(T) \lesssim \min\left\{ \sqrt{d \frac{T}{\kappa_\star} \log\frac{|\gX|}{\delta}}, \frac{d}{\kappa_\star \Delta} \log\frac{|\gX|}{\delta} \right\} + d^2 \kappa \log\frac{|\gX|}{\delta},
    \end{equation*}
    where $\Delta := \min_{\vx \in \gX \setminus \{\vx_\star\}} \mu(\vx_\star^\intercal \bm\theta_\star) - \mu(\vx^\intercal \bm\theta_\star)$ is the instance-dependent reward gap.
    Here, doubly logarithmic terms are omitted.
\end{theorem}

\subsection{Proof of Theorem~\ref{thm:new-regret-logistic-full} -- Regret Bound of \texttt{OFULog+}}
\label{app:proof-regret-logistic}
We follow the arguments presented in Appendix C.1 of \cite{abeille2021logistic}, but there are two key differences.
One is that we have a new confidence set (Theorem~\ref{thm:confidence-logistic}).
Another is that we use elliptical potential {\it count} lemma to control the additional dependencies on $S$, which we present here:
\begin{lemma}[Elliptical Potential Count Lemma\footnote{This is a generalization of Exercise 19.3 of \cite{banditalgorithms}, presented (in parallel) at Lemma 7 of \cite{gales2022norm} and Lemma 4 of \cite{kim2022variance}.}]
\label{lem:EPCL}
    Let $\vx_1, \cdots, \vx_T \in \gB^d(1)$ be a sequence of vectors, $\mV_t := \lambda \mI + \sum_{s=1}^{t-1} \vx_s \vx_s^\intercal$, and let us define the following: $\gH_T := \left\{ t \in [T] : \lVert \vx_t \rVert_{\mV_t^{-1}}^2 > 1 \right\}$.
    Then, we have that
    \begin{equation}
        \left| \gH_T \right| \leq \frac{2d}{\log(2)} \log\left( 1 + \frac{1}{\lambda \log(2)} \right).
    \end{equation}
\end{lemma}

We also recall the classical elliptical potential lemma:
\begin{lemma}[Lemma 11 of \cite{abbasiyadkori2011linear}]
\label{lem:EPL}
   Let $\vx_1, \cdots, \vx_T \in \gB^d(1)$ be a sequence of vectors and $\mV_t := \lambda \mI + \sum_{s=1}^{t-1} \vx_s \vx_s^\intercal$.
    Then, we have that
    \begin{equation}
        \sum_{t=1}^T \min\left\{ 1, \lVert \vx_t \rVert^2_{\mV_t^{-1}} \right\} \leq 2 d \log\left( 1 + \frac{T}{d\lambda} \right).
    \end{equation}
\end{lemma}

\begin{remark}
\label{rmk:EPCL}
    The ``classical'' Elliptical Potential Lemma~\citep{abbasiyadkori2011linear} ``forces'' $\lVert \vx_t \rVert_{\mV_t^{-1}}^2$ to be always bounded by $1$ via rescaling by $\max\left(1, \frac{1}{\lambda} \right)$, which is in our case of order $S^3$.
    Elliptical Potential Count Lemma helps us alleviate such additional $S$-dependency.
\end{remark}

Recall the following:
\begin{equation*}
    \Reg^B(T) = \underbrace{\sum_{t=1}^T \dot{\mu}(\vx_t^\intercal \bm\theta_\star) (\vx_{t,\star} - \vx_t)^\intercal \bm\theta_\star}_{\triangleq R_1(T)} + \underbrace{\sum_{t=1}^T \tilde{\vartheta}_t \left\{ (\vx_{t,\star} - \vx_t)^\intercal \bm\theta_\star \right\}^2}_{\triangleq R_2(T)},
\end{equation*}
where
\begin{equation*}
    \vartheta_t = \int_0^1 (1 - v) \ddot{\mu}\left( \vx_t^\intercal \bm\theta_\star + v (\vx_{t,\star} - \vx_t)^\intercal \bm\theta_\star \right) dv.
\end{equation*}

We bound $R_1(T)$ first.
To do that, we first recall the crucial lemma:
\hnorm*

Let us define $\tilde{\vx}_t := \sqrt{\dot{\mu}(\vx_t^\intercal \bm\theta_\star)} \vx_t$ and $\gH_T := \left\{ t \in [T] : \lVert \tilde{\vx}_t \rVert_{\widetilde{\mV}_{t-1}^{-1}}^2 > 1 \right\}$.
Note that $\mH_t(\bm\theta_\star) = \lambda \mI + \sum_{s=1}^{t-1} \tilde{\vx}_s \tilde{\vx}_s^\intercal$.
Then, the following holds w.p. at least $1 - \delta$:
\begin{align*}
    R_1(T) &= \sum_{t \in \gH_T} \dot{\mu}(\vx_t^\intercal \bm\theta_\star) (\vx_{t,\star} - \vx_t)^\intercal \bm\theta_\star + \sum_{t \not\in \gH_T} \dot{\mu}(\vx_t^\intercal \bm\theta_\star) (\vx_{t,\star} - \vx_t)^\intercal \bm\theta_\star \\
    &\leq 2S |\gH_T| + \sum_{t \not\in \gH_T} \dot{\mu}(\vx_t^\intercal \bm\theta_\star) (\vx_{t,\star} - \vx_t)^\intercal \bm\theta_\star \\
    &\lesssim dS \log S + \sum_{t \not\in \gH_T} \dot{\mu}(\vx_t^\intercal \bm\theta_\star) \lVert \vx_t \rVert_{\mH_t^{-1}(\bm\theta_\star)} \lVert \bm\theta_t - \bm\theta_\star \rVert_{\mH_t(\bm\theta_\star)} \tag{Lemma~\ref{lem:EPCL}} \\
    &\lesssim dS\log S + \sum_{t \not\in \gH_T} \gamma_t(\delta) \sqrt{\dot{\mu}(\vx_t^\intercal \bm\theta_\star)} \lVert \tilde{\vx}_t \rVert_{\mH_t^{-1}(\bm\theta_\star)} \tag{Lemma~\ref{lem:H-norm}} \\
    &\leq dS\log S + \gamma_T(\delta) \sqrt{\sum_{t \not\in \gH_T} \dot{\mu}(\vx_t^\intercal \bm\theta_\star)} \sqrt{\sum_{t \not\in \gH_T} \lVert \tilde{\vx}_t \rVert_{\mH_t^{-1}(\bm\theta_\star)}^2} \\
    &\leq dS\log S + \gamma_T(\delta) \sqrt{\sum_{t=1}^T \dot{\mu}(\vx_t^\intercal \bm\theta_\star)} \sqrt{\sum_{t=1}^T \min\left\{ 1, \lVert \tilde{\vx}_t \rVert_{\mH_t^{-1}(\bm\theta_\star)}^2 \right\}} \\
    &\lesssim dS\log S + \gamma_T(\delta) \sqrt{\sum_{t=1}^T \dot{\mu}(\vx_t^\intercal \bm\theta_\star)} \sqrt{d \log\left(1 + \frac{ST}{d}\right)} \tag{Lemma~\ref{lem:EPL}} \\
    &\lesssim dS\log S + \gamma_T(\delta) \left( \sqrt{\frac{T}{\kappa_\star(T)}} + \sqrt{\Reg^B(T)} \right) \sqrt{d \log\left(1 + \frac{ST}{d}\right)} \tag{Appendix C.1 of \cite{abeille2021logistic}} \\
    &\lesssim \sqrt{d} S \left( \sqrt{d} \log\left(e + \frac{ST}{d}\right) + \log\frac{1}{\delta} \right) \left( \sqrt{\frac{T}{\kappa_\star(T)}} + \sqrt{\Reg^B(T)} \right).
\end{align*}

Similarly to above, by altering the proof of Appendix C.1 of \cite{abeille2021logistic} by using Lemma~\ref{lem:H-norm}, \ref{lem:EPCL} and \ref{lem:EPL}, we now bound $R_2(T)$ via two different proof processes.
\begin{equation*}
    R_2(T) \lesssim \underbrace{d S^2 \left( d \left( \log\left(e + \frac{ST}{d}\right) \right)^2 + \left( \log\frac{1}{\delta} \right)^2 \right)}_{\triangleq R_{log}(T)} \kappa_\gX(T)
\end{equation*}
and
\begin{equation*}
    R_2(T) \lesssim S \sum_{t=1}^T \mu(\vx_{t,\star}^\intercal \bm\theta_\star) \mathds{1}[\vx_t \in \gX_-(t)] + R_{log}(T).
\end{equation*}

We recall the following polynomial inequality:
\begin{lemma}[Proposition 7 of \cite{abeille2021logistic}]
\label{lem:poly}
    For $b, c \geq 0$ and $x \in \sR$, $x^2 \leq bx + c$ implies $x^2 \leq 2(b^2 + c)$.
\end{lemma}
All in all, we have
\begin{align*}
    \Reg^B(T) &\lesssim 
    \sqrt{d} S \left( \sqrt{d} \log\left(e + \frac{ST}{d}\right) + \log\frac{1}{\delta} \right) \sqrt{\Reg^B(T)}
    + \underbrace{\sqrt{d} S \left( \sqrt{d} \log\left(e + \frac{ST}{d}\right) + \log\frac{1}{\delta} \right) \sqrt{\frac{T}{\kappa_\star(T)}}}_{\triangleq R_{leading}(T)} \\
    &\quad + R_{log}(T) + \underbrace{\min\left\{ \kappa_\gX(T) R_{log}(T), S \sum_{t=1}^T \mu(\vx_{t,\star}^\intercal \bm\theta_\star) \mathds{1}[\vx_t \in \gX_-(t)] \right\}}_{\triangleq R_{detr}(T)},
\end{align*}
and thus from Lemma~\ref{lem:poly}, we have that $\Reg^B(T) \lesssim R_{leading}(T) + R_{log}(T) + R_{detr}(T).$
\hfill\qedsymbol

\subsection{Proof of Supporting Lemmas}
\subsubsection{Proof of Lemma~\ref{lem:decomposition1-logistic}}
\label{app:proof-decomposition1}
We overload the notation and let $\ell_s(\mu) := -r_s \log \mu - (1 - r_s) \log (1 - \mu)$.
In this case, we have the following:
\begin{equation*}
    \ell_s'(\mu) = - \frac{r_s}{\mu} + \frac{1 - r_s}{1 - \mu}, \quad \ell_s''(\mu) = \frac{r_s}{\mu^2} + \frac{1 - r_s}{(1 - \mu)^2}.
\end{equation*}
By Taylor's theorem with the integral form of the remainder,
\begin{align*}
    \ell_s(\mu_s) - \ell_s(\mu^\star) &= \ell_s'(\mu^\star) (\mu_s - \mu^\star) + \int_{\mu^\star}^{\mu_s} \ell_s''(z) (\mu_s - z) dz \\
    &= \frac{\mu^\star - r_s}{\mu^\star (1 - \mu^\star)} (\mu_s - \mu^\star) + \int_{\mu^\star}^{\mu_s} \left( \frac{r_s}{z^2} + \frac{1 - r_s}{(1 - z)^2} \right) (\mu_s - z) dz \\
    &= -\xi_s \frac{\mu_s - \mu^\star}{\mu^\star (1 - \mu^\star)} + \int_{\mu^\star}^{\mu_s} \left( \frac{r_s}{z^2} + \frac{1 - r_s}{(1 - z)^2} \right) (\mu_s - z) dz,
\end{align*}
where we recall that $\mu^\star - r_s = -\xi_s$.
Let us simplify the integral on the RHS:
\begin{align*}
    & \int_{\mu^\star}^{\mu_s} \left( \frac{r_s}{z^2} + \frac{1 - r_s}{(1 - z)^2} \right) (\mu_s - z) dz \\
    &= r_s\left\{ \frac{\mu_s}{\mu^\star} - 1 - \log\frac{\mu_s}{\mu^\star} \right\} + (1 - r_s) \left\{ \frac{1 - \mu_s}{1 - \mu^\star} - 1 - \log\frac{1 - \mu_s}{1 - \mu^\star} \right\} \\
    &= - 1 + \left\{ r_s \frac{\mu_s}{\mu^\star} + (1 - r_s) \frac{1 - \mu_s}{1 - \mu^\star} \right\} - \left\{ r_s \log\frac{\mu_s}{\mu^\star} + (1 - r_s) \log\frac{1 - \mu_s}{1 - \mu^\star} \right\} \\
    &\overset{(*)}{=} -1 + \left\{ \mu_s + \xi_s \frac{\mu_s}{\mu^\star} + (1 - \mu_s) - \xi_s \frac{1 - \mu_s}{1 - \mu^\star} \right\} - \left\{ \mu^\star \log\frac{\mu_s}{\mu^\star} + (1 - \mu^\star) \log\frac{1 - \mu_s}{1 - \mu^\star} + \xi_s \log\frac{\mu_s}{\mu^\star} - \xi_s \log\frac{1 - \mu_s}{1 - \mu^\star} \right\} \\
    &= \xi_s \frac{\mu_s - \mu^\star}{\mu^\star (1 - \mu^\star)} + \KL(\mu^\star, \mu_s) + \xi_s \left( \log\frac{\mu^\star}{1 - \mu^\star} - \log\frac{\mu_s}{1 - \mu_s} \right),
\end{align*}
where $(*)$ follows from the fact that $r_s = \mu^\star + \xi_s$.
Plugging this back into the original expression and recalling the definition of $\mu_s$ and $\mu^\star$, we have that
\begin{align*}
    \ell_s(\mu_s) - \ell_s(\mu^\star) &= \KL(\mu^\star, \mu_s) + \xi_s \left( \langle \vx_s, \bm\theta^\star \rangle - \langle \vx_s, \bm\theta_s \rangle \right) \\
    &= \KL(\mu^\star, \mu_s) + \xi_s \langle \vx_s, \bm\theta^\star - \bm\theta_s \rangle.
\end{align*}
\hfill\qedsymbol

\subsubsection{Proof of Lemma~\ref{lem:decomposition2-logistic}}
\label{app:proof-decomposition2}
By Lemma~\ref{lem:decomposition1-logistic}, we have the following:
\begin{align*}
    0 &= \sum_{s=1}^t \left\{ \ell_s(\tilde{\bm\theta}_s) - \ell_s(\bm\theta_\star) -\KL(\mu_s(\bm\theta_\star), \mu_s(\tilde{\bm\theta}_s) ) + \xi_s \langle \vx_s, \tilde{\bm\theta}_s - \bm\theta_\star \rangle \right\} \\
    &= \sum_{s=1}^t \left\{ \ell_s(\tilde{\bm\theta}_s) - \ell_s(\widehat{\bm\theta}_t) + \ell_s(\widehat{\bm\theta}_t) - \ell_s(\bm\theta^\star) -\KL(\mu_s(\bm\theta^\star), \mu_s(\tilde{\bm\theta}_s)) + \xi_s \langle \vx_s, \tilde{\bm\theta}_s - \bm\theta_\star \rangle \right\} \\
    &= \sum_{s=1}^t \left\{ \ell_s(\widehat{\bm\theta}_t) - \ell_s(\bm\theta^\star) - \KL(\mu_s(\bm\theta^\star), \mu_s(\tilde{\bm\theta}_s)) + \xi_s \langle \vx_s, \tilde{\bm\theta}_s - \bm\theta_\star \rangle \right\} + \Reg^O(T).
\end{align*}
Rearranging gives the desired result.
\hfill\qedsymbol

\subsubsection{Proof of Lemma~\ref{lem:kl-bregman}}
\label{app:kl-bregman}
This follows from direct computation:
\begin{align*}
    D_m(z_1, z_2) &= m(z_1) - m(z_2) - m'(z_2) (z_1 - z_2) \\
    &= \log (1 + e^{z_1}) - \log (1 + e^{z_2}) - \frac{e^{z_2}}{1 + e^{z_2}} (z_1 - z_2) \\
    &= \log \frac{e^{z_2}}{1 + e^{z_2}} - \log \frac{e^{z_1}}{1 + e^{z_1}} + \left( 1 - \frac{e^{z_2}}{1 + e^{z_2}} \right) (z_1 - z_2) \\
    &= \log \mu_2 - \log \mu_1 + (1 - \mu_2) \log \frac{\mu_1 (1 - \mu_2)}{\mu_2 (1 - \mu_1)} \\
    &= \mu_2 \log\frac{\mu_2}{\mu_1} + (1 - \mu_2) \log\frac{1 - \mu_2}{1 - \mu_1}
    = \KL(\mu_2, \mu_1).
\end{align*}
\hfill\qedsymbol

\subsubsection{Proof of Lemma~\ref{lem:H-norm}}
\label{app:H-norm}
% We first recall important self-concordant results:
% \begin{lemma}[Lemma 7 and 8 of \cite{abeille2021logistic}]
% \label{lem:self-concordant}
%     For any $\bm\theta_1, \bm\theta_2$,
%     \begin{align}
%         % \mG_t(\bm\theta_1, \bm\theta_2) &\succeq \frac{1}{1 + 2S} \mH_t(\bm\theta), \quad \bm\theta \in \{\bm\theta_1, \bm\theta_2\}, \\
%         \widetilde{\mG}_t(\bm\theta_1, \bm\theta_2) &\succeq \frac{1}{2 + 2S} \mH_t(\bm\theta_1).
%     \end{align}
% \end{lemma}

% Let $\bm\theta \in \gC_t(\delta)$.
By Theorem~\ref{thm:confidence-logistic}, we have that with probability at least $1 - \delta$, $\gL_t(\bm\theta_\star) - \gL_t(\widehat{\bm\theta}_t) \leq \beta_t(\delta)^2$; throughout the proof let us assume that this event is true.
Also, let $\bm\theta \in \gC_t(\delta)$.
Then, by second-order Taylor expansion of $\gL_t(\bm\theta)$ around $\bm\theta_\star$,
\begin{align*}
    \gL_t(\bm\theta) &= \gL_t(\bm\theta_\star) + \nabla \gL_t(\bm\theta_\star)^\intercal (\bm\theta - \bm\theta_\star) + \lVert \bm\theta - \bm\theta_\star \rVert_{\widetilde{\mG}(\bm\theta_\star, \bm\theta) - \lambda_t \mI}^2 \\
    &= \gL_t(\bm\theta_\star) + \nabla \gL_t(\bm\theta_\star)^\intercal (\bm\theta - \bm\theta_\star) + \lVert \bm\theta - \bm\theta_\star \rVert_{\widetilde{\mG}(\bm\theta_\star, \bm\theta)}^2 - \lambda_t \lVert \bm\theta - \bm\theta_\star \rVert_2^2.
\end{align*}
Lemma~\ref{lem:self-concordant} implies that $\widetilde{\mG}_t(\bm\theta_1, \bm\theta_2) \succeq \frac{1}{2 + 2S} \mH_t(\bm\theta_1)$.
Thus, we have that
\begin{align}
    \lVert \bm\theta - \bm\theta_\star \rVert_{\mH_t(\bm\theta_\star)}^2 &\leq (2 + 2S) \lVert \bm\theta - \bm\theta_\star \rVert_{\widetilde{\mG}_t(\bm\theta_\star, \bm\theta)}^2 \nonumber \\
    &= (2 + 2S) \left( \gL_t(\bm\theta) - \gL_t(\bm\theta_\star) + \nabla \gL_t(\bm\theta_\star)^\intercal (\bm\theta_\star - \bm\theta) + \lambda_t \lVert \bm\theta - \bm\theta_\star \rVert_2^2 \right) \nonumber \\
    &\leq (2 + 2S) \left( \gL_t(\bm\theta) - \gL_t(\widehat{\bm\theta}_t) + \nabla \gL_t(\bm\theta_\star)^\intercal (\bm\theta_\star - \bm\theta) + \lambda_t \lVert \bm\theta - \bm\theta_\star \rVert_2^2 \right)  \tag{$\gL_t(\widehat{\bm\theta}_t) \leq \gL_t(\bm\theta_\star)$} \nonumber \\
    &\leq 1 + (2 + 2S) \beta_t(\delta)^2 + (2 + 2S) \nabla \gL_t(\bm\theta_\star)^\intercal (\bm\theta_\star- \bm\theta), \quad \text{w.p. at least $1 - \delta$}, \label{eqn:norm-H}
\end{align}
where we chose $\lambda_t = \frac{1}{4S^2 (2 + 2S)}$.
Here, there is no need to consider time-varying regularization as unlike \cite{abeille2021logistic}, and we do {\it not} explicitly use the regularization by $\lambda_t$ in our algorithm.

Thus, it remains to bound $\nabla \gL_t(\bm\theta_\star)^\intercal (\bm\theta_\star- \bm\theta)$, which is done via a new concentration-type argument.
Let $\gB_d(2S)$ be a $d$-ball of radius $2S$ and $\vv \in \gB_d(2S)$.
% $\vw := \bm\theta - \widehat{\bm\theta_t} \in \gB_d(2S)$.

First note that
\begin{equation*}
    \nabla \gL_t(\bm\theta_\star)^\intercal \vv = \sum_{s=1}^t (\mu(\vx_s^\intercal \bm\theta_\star) - r_s) \vx_s^\intercal \vv = \sum_{s=1}^t \xi_s \vx_s^\intercal \vv,
\end{equation*}
where here we overload the notation and denote $\xi_s := \mu(\vx_s^\intercal \bm\theta_\star) - r_s$.
Still, $\xi_s$ is a martingale difference sequence w.r.t. $\gF_{s-1} = \sigma\left( \{ \vx_1, r_1, \cdots, \vx_{s-1}, r_{s-1}, \vx_s \} \right)$, and thus so is $\xi_s \vx_s^\intercal \vv$.

As $|\xi_s \vx_s^\intercal \vv| \leq 2S$ and $\E[(\xi_s \vx_s^\intercal \vv)^2 | \gF_{s-1}] = \dot{\mu}(\vx_s^\intercal \bm\theta_\star) (\vx_s^\intercal \vv)^2$, by Freedman's inequality (Lemma~\ref{lem:freedman}), for any $\eta \in \left[ 0, \frac{1}{2S} \right]$, the following holds:
\begin{equation}
\label{eqn:freedman2}
    \sP\left[ \sum_{s=1}^t \xi_s \vx_s^\intercal \vv \leq (e - 2)\eta \sum_{s=1}^t \dot{\mu}(\vx_s^\intercal \bm\theta_\star) (\vx_s^\intercal \vv)^2 + \frac{1}{\eta} \log\frac{1}{\delta} \right] \geq 1 - \delta.
\end{equation}

Now for $\eps_t \in (0, 1)$ to be chosen later satisfying $\eps_t < \eps_{t+1}$, let $\hcB_{\eps_t}$ be an $\eps_t$-cover of $\gB_d(2S)$ (endowed with the usual Euclidean topology), i.e.,
\begin{equation*}
    \forall \vv \in \gB_d(2S), \ \exists \vw(\vv) \in \hcB_{\eps_t} : \lVert \vv - \vw(\vv) \rVert_2 \leq \eps_t.
\end{equation*}
By Corollary 4.2.13 of \cite{vershynin2018high}, we have that $|\hcB_{\eps_t}| \leq \left( \frac{5S}{\eps_t} \right)^d$.
With this, we apply union bound for Eqn.~\eqref{eqn:freedman2} to both $t \geq 1$ and $\vv \in \hcB_{\eps_t}$: with the choice of $\delta_t = \left( \frac{\eps_t}{5S} \right)^d \frac{\delta}{t}$ and applying the union bound, for any $\eta \in [0, 2S]$, the following holds with probability at least $1 - \delta$:
\begin{equation*}
    \sum_{s=1}^t \xi_s \vx_s^\intercal \vv \leq (e - 2)\eta \sum_{s=1}^t \dot{\mu}(\vx_s^\intercal \bm\theta_\star) (\vx_s^\intercal \vv)^2 + \frac{d}{\eta} \log\frac{5S}{\eps_t} + \frac{1}{\eta} \log\frac{1}{\delta}, \quad \forall \vv \in \hcB(\eps_t), \ \forall t \geq 1.
\end{equation*}

Let $\vv_t \in \hcB_{\eps_t}$ be s.t. $\lVert (\bm\theta_\star - \bm\theta) - \vv_t \rVert_2 \leq \eps_t$.
Then,
\begin{align*}
    &\nabla \gL_t(\bm\theta_\star)^\intercal (\bm\theta_\star - \bm\theta) \\
    &= \sum_{s=1}^t \xi_s \vx_s^\intercal \vv_t +  \sum_{s=1}^t \xi_s \vx_s^\intercal \left( (\bm\theta_\star - \bm\theta) - \vv_t \right) \\
    &\leq (e - 2)\eta \sum_{s=1}^t \dot{\mu}(\vx_s^\intercal \bm\theta_\star) (\vx_s^\intercal \vv_t)^2 + \frac{d}{\eta} \log\frac{5S}{\eps_t} + \frac{1}{\eta} \log\frac{1}{\delta} + \eps_t t \tag{w.p. at least $1 - \delta$} \\
    &= (e - 2)\eta \sum_{s=1}^t \dot{\mu}(\vx_s^\intercal \bm\theta_\star) (\vx_s^\intercal (\bm\theta_\star - \bm\theta))^2 + (e - 2)\eta \sum_{s=1}^t \dot{\mu}(\vx_s^\intercal \bm\theta_\star) \left( (\vx_s^\intercal \vv_t)^2 - (\vx_s^\intercal (\bm\theta_\star - \bm\theta))^2 \right) + \frac{d}{\eta} \log\frac{5S}{\eps_t} + \frac{1}{\eta} \log\frac{1}{\delta} + \eps_t t \\
    &\overset{(*)}{\leq} (e - 2)\eta \sum_{s=1}^t \dot{\mu}(\vx_s^\intercal \bm\theta_\star) (\vx_s^\intercal (\bm\theta_\star - \bm\theta))^2 + \frac{(e - 2) \eta}{4} (4S \eps_t + \eps_t^2) t + \frac{d}{\eta} \log\frac{5S}{\eps_t} + \frac{1}{\eta} \log\frac{1}{\delta} + \eps_t t \\
    &= (e - 2)\eta \lVert \bm\theta_\star - \bm\theta \rVert_{\mH_t(\bm\theta_\star)}^2 + \frac{d}{\eta} \log\frac{5S}{\eps_t} + \frac{1}{\eta} \log\frac{1}{\delta} + \left( \frac{(e - 2)}{4} \left( 4S\eta + \eps_t \eta \right) + 1 \right) \eps_t t.
\end{align*}
where $(*)$ follows from $\dot{\mu} \leq \frac{1}{4}$ and
\begin{equation*}
    (\vx_s^\intercal \va)^2 - (\vx_s^\intercal \vb)^2 = 2 \vx_s^\intercal \vb \vx_s^\intercal (\vb - \va) + (\vx_s^\intercal (\va - \vb))^2 
    \leq 4S \eps_t + \eps_t^2
\end{equation*}
for any $\va, \vb \in \hcB_{\eps_t}$.

Choosing $\eta = \frac{1}{2 (e - 2) (2 + 2S)} < \frac{1}{2S}$, $\eps_t = \frac{d}{t}$, and rearranging Eqn.~\eqref{eqn:norm-H} with Theorem~\ref{thm:confidence-logistic}, we finally have that
\begin{align*}
    \lVert \bm\theta - \bm\theta_\star \rVert_{\mH_t(\bm\theta_\star)}^2 \lesssim dS^2 \log\left(e + \frac{St}{d}\right) + S^2 \log\frac{1}{\delta}.
\end{align*}
\hfill\qedsymbol

\newpage
\section{PROOFS - MULTINOMIAL LOGISTIC BANDITS}

\subsection{Notations}
Recall that $\mA(\vx, \bm\Theta) := \diag(\bm\mu(\vx, \bm\Theta)) - \bm\mu(\vx, \bm\Theta) \bm\mu(\vx, \bm\Theta)^\intercal$.
We now define the following quantities:
\begin{align}
    \mH_t(\bm\Theta) &:= \lambda\mI_{Kd} + \sum_{s=1}^{t-1} \mA(\vx_s, \bm\Theta) \otimes \vx_s \vx_s^\intercal \\
    \mB(\vx, \bm\Theta_1, \bm\Theta_2) &:= \int_0^1 \mA(\vx, \bm\Theta_1 + v(\bm\Theta_2 - \bm\Theta_1)) dv, \\
    \mG_t(\bm\Theta_1, \bm\Theta_2) &:=\lambda \mI_{Kd} + \sum_{s=1}^{t-1} \mB(\vx_s, \bm\Theta_1, \bm\Theta_2) \otimes \vx_s \vx_s^\intercal, \\
    \widetilde{\mB}(\vx, \bm\Theta_1, \bm\Theta_2) &:= \int_0^1 (1 - v) \mA(\vx, \bm\Theta_1 + v(\bm\Theta_2 - \bm\Theta_1)) dv, \\
    \widetilde{\mG}_t(\bm\Theta_1, \bm\Theta_2) &:=\lambda \mI_{Kd} + \sum_{s=1}^{t-1} \widetilde{\mB}(\vx_s, \bm\Theta_1, \bm\Theta_2) \otimes \vx_s \vx_s^\intercal, \\
    \mV_t &:= 2 \kappa(T) \lambda \mI_d + \sum_{s=1}^{t-1} \vx_s \vx_s^\intercal,
\end{align}
where $\lambda > 0$ is to be chosen later.

We also recall all problem-dependent quantities as introduced in \cite{amani2021mnl}, which we extend to time-varying arm-set:
\begin{align}
    \kappa(T) &= \max_{t \in [T]} \max_{\vx \in \gX_t} \max_{\bm\Theta \in \gB^{K \times d}(S)} \frac{1}{\lambda_{\min}\left( \mA(\vx, \bm\Theta) \right)}, \\
    L_T &= \max_{t \in [T]} \max_{\vx \in \gX_t} \max_{\bm\Theta \in \gB^{K \times d}(S)} \lambda_{\max}\left( \mA(\vx, \bm\Theta) \right), \label{eqn:L} \\
    M_T &\geq \max_{t \in [T]} \max_{\vx \in \gX_t} \max_{\bm\Theta \in \gB^{K \times d}(S)} \max_{k \in [K]} \left| \lambda_{\max}\left( \nabla^2 \mu_k(\vx, \bm\Theta) \right) \right|, \\
    M_T' &\geq \max_{t \in [T]} \max_{\vx \in \gX_t} \max_{\bm\Theta \in \gB^{K \times d}(S)} \max_{k, k' \in [K]} \left| \lambda_{\max}\left( \nabla [\mA(\vx, \bm\Theta)_{k,k'}] \right) \right|.
\end{align}

Also, we overload the notation and define $A \lesssim B$ to be when we have $A \leq c B$ for some {\it universal} constant $c$, not ignoring logarithmic factors.

\subsection{Proof of Theorem~\ref{thm:confidence-multinomial} -- MNL Loss-based Confidence Set}
\label{app:confidence-multinomial}
We can write
\begin{equation}
    \vy_s = \bm\mu(\vx_s, \bm\Theta_\star) + \bm\xi_s,
\end{equation}
where $\bm\xi_s$ is some vector-valued martingale noise and $\vy_s = (y_{s,1}, \cdots, y_{s,K}) \in \{0, 1\}^K$.
% \kj{is $\vy_s\in \mathbb{R}^{K+1}$ or $\in \mathbb{R}^{K}$?}

We first establish an extension of Lemma~\ref{lem:decomposition1-logistic} to the multiclass case, whose proof is provided in Appendix~\ref{app:lem-decomposition1}:
\begin{lemma}
\label{lem:decomposition1-multinomial}
    The following holds for any $\bm\Theta, \bm\Theta_\star \in \sR^{K \times d}$:
    \begin{equation}
        \ell_s(\bm\Theta_\star) = \ell_s(\bm\Theta) - \KL(\bm\mu(\vx_s, \bm\Theta_\star), \bm\mu(\vx_s, \bm\Theta)) + \bm\xi^\intercal (\bm\Theta - \bm\Theta_\star) \vx_s.
    \end{equation}
\end{lemma}
% From hereon, let us universally denote $\bm\theta \in \sR^{Kd \times 1}$ to be the vectorized parameter vector, i.e., $\bm\theta = \vectorize(\bm\Theta^\intercal)$.

Let $\{\widetilde{\bm\Theta}_s\} \subset \gB^{K \times d}(\sqrt{K} S)$ be the output from an online learning algorithm of our choice (for Algorithm 1 of \cite{foster2018logistic}, this is guaranteed by Proposition~\ref{prop:surjection}).
The following lemma, whose proof is immediate from the above lemma (and is the same as that of Lemma~\ref{lem:decomposition2-logistic}), provides the necessary connection:
\begin{lemma}
\label{lem:decomposition2-multinomial}
    \begin{equation}
    \label{eqn:decomposition-multinomial}
     \sum_{s=1}^t \ell_s(\bm\Theta_\star) - \ell_s(\widehat{\bm\Theta}_t) \leq \Reg^O(t) + \zeta_1(t) - \zeta_2(t),
    \end{equation}
    where
    \begin{equation*}
        \zeta_1(t) := \sum_{s=1}^t \bm\xi_s^\intercal (\widetilde{\bm\Theta}_s - \bm\Theta_\star) \vx_s, \quad
        \zeta_2(t) := \sum_{s=1}^t \KL(\bm\mu(\vx_s, \bm\Theta_\star), \bm\mu(\vx_s, \widetilde{\bm\Theta}_s)).
    \end{equation*}
\end{lemma}

For bounding $\Reg^O(T)$, we again consider the algorithm of \cite{foster2018logistic}, which is also valid for online multiclass logistic regression:
\begin{theorem}[Theorem 3 of \cite{foster2018logistic}]
    There exists an (improper learning) algorithm for online multiclass logistic regression with the following regret:
    \begin{equation}
    \label{eqn:foster-multinomial}
        \Reg^O(t) \leq 5 d (K + 1) \log \left( e + \frac{St}{d(K + 1)} \right).
    \end{equation}
\end{theorem}
\begin{remark}
    Again, if one were to use the classical {\it O2CS} approach, then to take computational efficiency into account, one would have to use efficient variants of online multiclass logistic regression algorithm~\citep{jezequel2021logistic,agarwal2022logistic}.
    These, however, incur an online regret that scales in $S$, again, which leads to no improvement in the final regret.
\end{remark}

\subsubsection{Upper Bounding $\zeta_1(t)$: Martingale Concentrations}
Again, let $\gF_{s-1}$ be the $\sigma$-field generated by $(\vx_1, \vy_1, \cdots, \vx_{s-1}, \vy_{s-1}, \vx_s)$.
Then, $\vx_s$ and $\widetilde{\bm\Theta}_s$ are $\gF_{s-1}$-measurable, and $\bm\xi_s^\intercal (\widetilde{\bm\Theta}_s - \bm\Theta_\star) \vx_s$ is martingale difference w.r.t. $\gF_{s-1}$.
We also have that $|\bm\xi_s^\intercal (\widetilde{\bm\Theta}_s - \bm\Theta_\star) \vx_s| \leq 2 \sqrt{K}S$
and
\begin{align*}
    \E\left[ \left( \bm\xi_s^\intercal (\widetilde{\bm\Theta}_s - \bm\Theta_\star) \vx_s \right)^2 | \gF_{s-1} \right] &= \vx_s^\intercal (\widetilde{\bm\Theta}_s - \bm\Theta_\star)^\intercal \E[ \bm\xi_s \bm\xi_s^\intercal | \gF_{s-1}] (\widetilde{\bm\Theta}_s - \bm\Theta_\star) \vx_s \\
    &= \vx_s^\intercal (\widetilde{\bm\Theta}_s - \bm\Theta_\star)^\intercal \underbrace{\left( \diag(\{ \mu_k((\bm\theta_\star^{(k)})^\intercal \vx_s) \}_{k=1}^K) - \bm\mu_s \bm\mu_s^\intercal \right)}_{\triangleq \mA(\vx_s, \bm\Theta_\star)} (\widetilde{\bm\Theta}_s - \bm\Theta_\star) \vx_s \triangleq \sigma_s^2.
\end{align*}

By Freedman's concentration inequality (Lemma~\ref{lem:freedman}), the following holds for any $\eta \in \left[ 0, \frac{1}{2\sqrt{K}S} \right]$:
\begin{equation}
\label{eqn:zeta1-multinomial}
    \sP\left[ \zeta_1(t) = \sum_{s=1}^t \bm\xi_s^\intercal (\widetilde{\bm\Theta}_s - \bm\Theta_\star) \vx_s
    \leq (e - 2) \eta\sum_{s=1}^t \sigma_s^2 + \frac{1}{\eta} \log\frac{1}{\delta}, \quad \forall t \geq 1 \right] \geq 1 - \delta.
\end{equation}

\subsubsection{Lower bounding $\zeta_2(t)$: Multivariate second-order expansion of the KL Divergence}
The following lemmas are multivariate version of Lemma~\ref{lem:kl-bregman} and \ref{lem:self-concordant}
\begin{lemma}
\label{lem:kl-bregman-multinomial}
    Let $m(\vz) := \log\left( 1 + \sum_{k=1}^K e^{z_k} \right)$ be the log-exp-sum function (which is known to be the log-partition function for Categorical distribution), and $\bm\mu(\vz) = (\mu_1, \cdots, \mu_K)$ with $\mu_k := \frac{e^{z_k}}{1 + \sum_{k=1}^K e^{z_i}}$.
    Then we have that
    $\KL(\bm\mu(\vz^{(2)}), \bm\mu(\vz^{(1)})) = D_m(\vz^{(1)}, \vz^{(2)})$.
\end{lemma}
\begin{proof}
    See Appendix~\ref{app:lem-bregman}.
\end{proof}

\begin{definition}[Definition 1 of \cite{trandinh2015generalized}; Definition 2 of \cite{sun2019generalized}]
    For a given function $f : \sR^d \rightarrow \sR$, define $\varphi_{\vx, \vu}(t) := f(\vx + t\vu)$ for $\vx \in \mathrm{dom}(f)$ and $\vu \in \sR^d$.
    Then, we say that $f$ is {\bf $M_f$-generalized self-concordant} if the following is true for any $\vx, \vu$:
    \begin{equation*}
        |\varphi_{\vx,\vu}'''(t)| \leq M_f \varphi_{\vx,\vu}''(t) \lVert \vu \rVert_2, \quad \forall t \in \sR, M_f > 0.
    \end{equation*}
\end{definition}

\begin{lemma}
\label{lem:generalized-self-concordant1}
Suppose $f : \sR^d \rightarrow \sR$ is $M_f$-generalized self-concordant, and let $\gZ \subset \sR^d$ be bounded.
Then, the following holds for any $\vz_1, \vz_2 \in \gZ$:
\begin{equation}
    \int_0^1 (1 - v) \nabla^2 f(\vz_1 + v(\vz_2 - \vz_1)) dv \succeq \frac{1}{2 + M_f \lVert \vz_1 - \vz_2 \rVert_2} \nabla^2 f(\vz_1).
\end{equation}
\end{lemma}
\begin{proof}
    See Appendix~\ref{app:lem-generalized}.
\end{proof}

By Lemma 4 of \cite{trandinh2015generalized}, the log-exp-sum function $m$ is $\sqrt{6}$-generalized self-concordant.
Via a similar second-order expansion argument and the above lemma, we have that
\begin{align*}
    \KL(\bm\mu_\star, \tilde{\bm\mu}) &= D_m(\widetilde{\bm\Theta}_s \vx_s, \bm\Theta_\star \vx_s) \\
    &= \vx_s^\intercal (\widetilde{\bm\Theta}_s - \bm\Theta_\star)^\intercal \left\{ \int_0^1 (1 - v) \nabla^2 m(\bm\Theta_\star \vx_s + v(\widetilde{\bm\Theta}_s \vx_s - \bm\Theta_\star \vx_s)) dv \right\} (\widetilde{\bm\Theta}_s - \bm\Theta_\star) \vx_s \\
    &\geq \frac{1}{2 + \sqrt{6} \lVert (\widetilde{\bm\Theta}_s - \bm\Theta_\star) \vx_s \rVert_2} \vx_s^\intercal (\widetilde{\bm\Theta}_s - \bm\Theta_\star)^\intercal \nabla^2 m(\bm\Theta_\star \vx_s) (\widetilde{\bm\Theta}_s - \bm\Theta_\star) \vx_s \\
    &\geq \frac{1}{2 + 2\sqrt{6K} S} \sigma_s^2,
\end{align*}
and thus,
\begin{equation}
\label{eqn:zeta2-multinomial}
    \zeta_2(t) \geq \frac{1}{2 + 2\sqrt{6K}S} \sum_{s=1}^t \sigma_s^2.
\end{equation}

\begin{proof}[Proof of Theorem~\ref{thm:confidence-multinomial}]
Combining Eqn.~\eqref{eqn:decomposition-multinomial}, \eqref{eqn:foster-multinomial}, \eqref{eqn:zeta1-multinomial}, \eqref{eqn:zeta2-multinomial} with the choice of $\eta = \frac{1}{2(e - 2) + 2\sqrt{6K}S} < \frac{1}{2\sqrt{K}S}$ and the fact that $-\frac{1}{2 + 2\sqrt{6K}S} + \frac{e - 2}{2(e - 2) + 2\sqrt{6K}S} < 0$, we have the desired result.
\end{proof}
% \begin{proof}[Proof of Theorem~\ref{thm:confidence-logistic}]
% Combining Eqn.~\eqref{eqn:decomposition}\eqref{eqn:foster}\eqref{eqn:zeta1}\eqref{eqn:zeta2} with the choice of $\eta = \frac{1}{2S}$ and the fact that $-\frac{1}{2 + 2S} + \frac{e - 2}{2S} < 0$, we have the desired result.
% \end{proof}

\subsection{Full Theorem Statements for Regret Bounds}
\label{app:full-statements-multinomial}
We provide full theorem statements for our regret and prior arts for multinomial logistic bandits.
We start by providing the regret bound of our \texttt{MNL-UCB+}:
\begin{theorem}
\label{thm:new-regret-multinomial-full}
    \texttt{MNL-UCB+} and its improved version attain the following regret bounds, respectively, w.p. at least $1 - \delta$:
    \begin{align}
        \Reg^B(T) &\lesssim L_T R \sqrt{d \sqrt{K} S} \left( \sqrt{d \sqrt{K} \log\left(e + \frac{ST}{dK}\right)} + \sqrt{\log\frac{1}{\delta}} \right) \sqrt{\kappa(T) T \log\left( 1 + \frac{ST}{d K \kappa(T)} \right)} \nonumber \\ 
        &= \widetilde{\gO}\left( R d \sqrt{K S \kappa(T) T} \right), \\
        \Reg_{imp}^B(T) &\lesssim R \sqrt{d} K^{\frac{1}{4}} S \sqrt{\left( d\sqrt{K} \log\left(e + \frac{ST}{dK} \right) + \log\frac{1}{\delta} \right) \log \left( 1 + \frac{S^2 T}{d} \right) T} \nonumber \\
        &\quad + R d K^{\frac{3}{2}} S^{\frac{3}{2}} \max\{M_T, M_T'\} \left( d\sqrt{K} \log\left(e + \frac{ST}{dK} \right) + \log\frac{1}{\delta} \right) \log\left(1 + \frac{ST}{dK \kappa(T)} \right) \kappa(T) \nonumber \\
        &= \widetilde{\gO}\left( R d S \sqrt{KT} + R d^2 K^2 S^{\frac{3}{2}} \kappa(T) \right).
    \end{align}
\end{theorem}

We now provide the previous state-of-the-art regret bounds that we compare ourselves to:
\begin{theorem}[Theorem 2, 3 of \cite{amani2021mnl}]
\label{thm:old-regret-multinomial-full1}
    \texttt{MNL-UCB} and its improved version with $\lambda = \frac{d K^{\frac{3}{2}}}{S} \left( \log \left( 1 + \frac{S T}{dK} \right) + \log\frac{1}{\delta} \right)$ attain the following regret bounds, respectively, w.p. $1 - \delta$:
    \begin{align}
        \Reg^B(T) &\lesssim L_T R d K^{\frac{3}{4}} S \left( \log\left(1 + \frac{ST}{dK}\right) + \log\frac{1}{\delta} \right) \sqrt{\max\left( \frac{S}{dK^{\frac{3}{4}} \log\frac{ST}{dK\delta}}, \kappa(T) \right) T} \nonumber \\
        &= \widetilde{\gO}\left( R d K^{\frac{3}{4}} S \sqrt{\kappa(T) T} \right), \\
        \Reg_{imp}^B(T) &\lesssim R d K^{\frac{5}{4}} S^{\frac{3}{2}} \left( \log\left(1 + \frac{ST}{dK}\right) + \log\frac{1}{\delta} \right) \sqrt{T} \nonumber \\
        &\ + R d^2 K^2 S^2 (M_T' \sqrt{K S} + M_T) \left( \left(\log\left(1 + \frac{ST}{dK}\right)\right)^2 + \left( \log\frac{1}{\delta} \right)^2 \right) \max\left( \frac{S}{dK^{\frac{3}{2}} \log\frac{ST}{dK\delta}}, \kappa(T) \right) \nonumber \\
        &= \widetilde{\gO}\left( R d K^{\frac{5}{4}} S^{\frac{3}{2}} \sqrt{T} + Rd^2K^{\frac{5}{2}}S^{\frac{5}{2}} \kappa(T) \right).
    \end{align}
\end{theorem}

\begin{theorem}[Theorem 2 of \cite{zhang2023mnl}]
\label{thm:old-regret-multinomial-full2}
    \texttt{MNL-UCB+}\footnote{Coincidentally, their algorithm's name is the same as the one proposed here, but as this is the only place where the names may get confused, we use their name here.} with $\lambda = \frac{dK}{S} \log\frac{1}{\delta}$ attain the following regret bounds, respectively, w.p. $1 - \delta$:
    \begin{align}
    	\Reg^B(T) &\lesssim R \min\left\{ d S \sqrt{K \kappa(T) T}, dK S^{\frac{3}{2}} \sqrt{T} + d^2 K S \kappa(T) \sqrt{\log\frac{1 + \frac{T}{d}}{\delta} \log\left(1 + \frac{T}{d}\right)} \right\} \sqrt{\log\frac{1 + \frac{T}{d}}{\delta} \log\left(1 + \frac{T}{d}\right)} \nonumber \\
    	&= \widetilde{\gO}\left( R \min\left\{ dS \sqrt{K \kappa(T) T}, dKS^{\frac{3}{2}}\sqrt{T} + d^2 K S \kappa(T) \right\} \right).
    \end{align}
\end{theorem}

\begin{theorem}[Theorem 4 of \cite{zhang2023mnl}]
\label{thm:old-regret-multinomial-full3}
    For simplicity, suppose that $dK \gtrsim S$.
    Then, \texttt{OFUL-MLogB} with $\lambda = dKS$ attain the following regret bounds (simultaneously), respectively, w.p. $1 - \delta$:
    \begin{equation}
        \Reg^B(T) \lesssim R d K S^{\frac{3}{2}} \left( \sqrt{T \log\left( 1 + \frac{T \tilde{L}}{d K S} \right)} + d S^{\frac{3}{2}} \kappa(T) \log\left( 1 + \frac{T}{d K S \kappa(T)} \right) \right),
    \end{equation}
    with $\tilde{L}$ being the smoothness parameter of the logistic loss, or
    \begin{equation}
        \Reg^B(T) \lesssim R d S^{\frac{3}{2}} \sqrt{K \kappa(T) T \log\left( 1 + \frac{T}{d^2 K S \kappa(T)} \right)} + R d^2 K S^3 \kappa(T) \log\left( 1 + \frac{T}{d K S \kappa(T)} \right),
    \end{equation}
    where here only, for simplicity, we've omitted $\log\frac{1}{\delta}$ terms.
\end{theorem}
\begin{remark}
    If $dK \gtrsim S$, then we accordingly have extra $S$ dependency and less $dK$ dependency.
\end{remark}

\subsection{Proof of Theorem~\ref{thm:new-regret-multinomial-full} -- Regret Bound of (Improved) MNL-UCB+}
\label{app:regret-multinomial}
From hereon and forth, we vectorize everything and use $\bm\Theta$ and $\bm\theta$ interchangeably.
We denote $\bm\theta = \vectorize(\bm\Theta^\intercal) = \begin{bmatrix} \bm\theta^{(1)} \\ \vdots \\ \bm\theta^{(K)} \end{bmatrix} \in \sR^{Kd \times 1}$ for $\bm\theta^{(k)} \in \sR^d$, and the $k$-th row of $\bm\Theta$ is $\left( \bm\theta^{(k)} \right)^\intercal$.

Again, we start with the following crucial lemma, whose proof is provided in Appendix~\ref{app:H-norm-multinomial}:
\begin{lemma}
\label{lem:H-norm-multinomial}
    With $\lambda = \frac{K}{4S^2}$, for any $\bm\Theta \in \gC_t(\delta)$, the following holds with probability at least $1 - \delta$:
    \begin{equation}
        \lVert \bm\theta - \bm\theta_\star \rVert_{\widetilde{\mG}_t(\bm\Theta_\star, \bm\Theta)}^2 \lesssim \gamma_t(\delta)^2 \triangleq d K S \log\left(e + \frac{St}{dK}\right) + \sqrt{K}S \log\frac{1}{\delta} + dKL_T,
    \end{equation}
\end{lemma}
For simplicity, we assume that the last term, $dKL_T$, is negligible.

Now, assume that we have some bonus term $\epsilon_t(\vx)$ s.t. the following holds w.h.p. for each $\vx \in \gX_t$ and $t \in [T]$:
\begin{equation}
    \Delta(\vx, \bm\Theta_t) := \left| \bm\rho^\intercal \bm\mu(\vx, \bm\Theta_\star) - \bm\rho^\intercal \bm\mu(\vx, \bm\Theta_t) \right| \leq \epsilon_t(\vx),
\end{equation}
and assume that the learner follows the following UCB algorithm:
\begin{equation}
    \vx_t = \argmax_{\vx \in \gX_t} \bm\rho^\intercal \bm\mu(\vx, \bm\Theta_t) + \epsilon_t(\vx).
\end{equation}

Then, we have that
\begin{align*}
    \Reg^B(T) &= \sum_{t=1}^T \left\{ \bm\rho^\intercal \bm\mu(\vx_{t,\star}, \bm\Theta_\star) - \bm\rho^\intercal \bm\mu(\vx_t, \bm\Theta_\star) \right\} \\
    &\leq \sum_{t=1}^T \left\{ \bm\rho^\intercal \bm\mu(\vx_{t,\star}, \bm\Theta_t) + \epsilon_t(\vx_{t,\star}) - \bm\rho^\intercal \bm\mu(\vx_t, \bm\Theta_\star) \right\} \\
    &\leq \sum_{t=1}^T \left\{ \bm\rho^\intercal \bm\mu(\vx_t, \bm\Theta_t) + \epsilon_t(\vx_t) - \bm\rho^\intercal \bm\mu(\vx_t, \bm\Theta_\star) \right\} \\
    &\leq 2\sum_{t=1}^T \epsilon_t(\vx_t).
\end{align*}

% For the regret analyses, we use the classical elliptical potential lemma:
% \begin{lemma}[Lemma 11 of \cite{abbasiyadkori2011linear}]
%     Let $\vx_1, \cdots, \vx_T \in \sR^d$ be s.t. $\max_{t \in [T]} \lVert \vx_t \rVert_2 \leq 1$, and define $\mV_t := \sum_{s=1}^t \vx_s \vx_s^\intercal + \lambda \mI_d$ for some $\lambda > 0$.
%     Then, we have that
%     \begin{equation}
%         \sum_{t=1}^T \min\left\{ 1, \lVert \vx_t \rVert^2_{\mV_{t-1}^{-1}} \right\} \leq 2d \log\left(1 + \frac{T}{d \lambda} \right).
%     \end{equation}
% \end{lemma}
% \begin{lemma}[Lemma 7 of \cite{gales2022norm}]
%     \begin{equation}
%         \sum_{t=1}^T \1\left[ \lVert \vx_t \rVert^2_{\mV_{t-1}^{-1}} > 1 \right] \leq \frac{2d}{\log 2} \log\left( 1 + \frac{1}{\lambda \log 2} \right).
%     \end{equation}
% \end{lemma}

We also recall a simple technical lemma:
\begin{lemma}[Lemma 10 of \cite{amani2021mnl}]
\label{lem:10}
    \begin{equation}
        \bm\mu(\vx, \bm\Theta_1) - \bm\mu(\vx, \bm\Theta_1)  = \left[ \mB(\vx, \bm\Theta_1, \bm\Theta_2) \otimes \vx^\intercal \right] (\bm\theta_1 - \bm\theta_2).
    \end{equation}
\end{lemma}

\subsubsection{$\sqrt{\kappa T}$-type regret -- Algorithm~\ref{alg:multinomial1}}
Here, we follow the proof provided in Appendix B of \cite{amani2021mnl}, with appropriate modifications as done in our logistic bandit regret proof.
We start with the following lemma, 
\begin{lemma}[Improved Lemma 1 of \cite{amani2021mnl}]
\label{lem:1}
    For $\bm\Theta \in \gC_t(\delta)$ and $\vx \in \gX_t$, the following holds with probability at least $1 - \delta$:
    \begin{equation}
        \Delta(\vx, \bm\Theta) \leq \epsilon_t(\vx) := \sqrt{2\kappa(T)} R L_T \gamma_t(\delta) \lVert \vx \rVert_{\mV_t^{-1}}.
    \end{equation}
\end{lemma}
\begin{proof}
    We have that
    \begin{align*}
        \Delta(\vx, \bm\Theta) &\leq R \left\lVert \left[ \mB(\vx, \bm\Theta_\star, \bm\Theta) \otimes \vx^\intercal \right] (\bm\theta_\star - \bm\theta) \right\rVert_2 \tag{Assumption~\ref{assumption:R}, CS, Lemma~\ref{lem:10}} \\
        &\leq R \left\lVert \left[ \mB(\vx, \bm\Theta_\star, \bm\Theta) \otimes \vx^\intercal \right] \widetilde{\mG}_t(\bm\Theta_\star, \bm\Theta)^{-1/2} \right\rVert_2 \left\lVert \bm\theta_\star - \bm\theta \right\rVert_{\widetilde{\mG}_t(\bm\Theta_\star, \bm\Theta)} \tag{CS} \\
        &\overset{(*)}{\leq} R L_T \sqrt{\lambda_{\max}\left( \left[ \mI_K \otimes \vx^\intercal \right] \widetilde{\mG}_t(\bm\Theta_\star, \bm\Theta)^{-1} \left[ \mI_K \otimes \vx \right] \right)} \left\lVert \bm\theta_\star - \bm\theta \right\rVert_{\widetilde{\mG}_t(\bm\Theta_\star, \bm\Theta)} \\
        &\leq R L_T \sqrt{2\kappa \lambda_{\max}\left( \left[ \mI_K \otimes \vx^\intercal \right] \left[ \mI_K \otimes \mV_t^{-1} \right] \left[ \mI_K \otimes \vx \right] \right)} \left\lVert \bm\theta_\star - \bm\theta \right\rVert_{\widetilde{\mG}_t(\bm\Theta_\star, \bm\Theta)} \tag{Lemma~\ref{lem:generalized-self-concordant1}} \\
        &= \sqrt{2\kappa(T)} R L_T \gamma_t(\delta) \lVert \vx \rVert_{\mV_t^{-1}} \tag{$\bm\theta \in \gC_t(\delta)$, Theorem~\ref{thm:confidence-multinomial}},
    \end{align*}
    where CS refers to Cauchy-Schwartz inequality and $(*)$ is when the hidden computations are precisely the same as done in the chain of inequalities in Appendix B.2 of \cite{amani2021mnl}.
\end{proof}

Again, instead of na\"{i}vely applying the Elliptical Potential Lemma (Lemma~\ref{lem:EPL}), we utilize the Elliptical Potential Count Lemma (Lemma~\ref{lem:EPCL}).
Letting $\gH_T := \left\{ t \in [T] : \lVert \vx_t \rVert_{\mV_t^{-1}}^2 > 1 \right\}$,
\begin{align*}
    \Reg^B(T) &\lesssim \sum_{t=1}^T \epsilon_t(\vx_t) \\
    &\leq \sqrt{\kappa(T)} R L_T \gamma_T(\delta) \left( \frac{|\gH_T|}{\lambda \kappa(T)} + \sum_{t \not\in \gH_T} \lVert \vx_t \rVert_{\mV_t^{-1}} \right) \\
    &\lesssim \sqrt{\kappa(T)} R L_T \gamma_T(\delta) \left( \frac{S^2}{\kappa(T) K} \log\frac{S^2}{\kappa(T) K} + \sqrt{(T - |\gH_T|) \sum_{t \not\in \gH_T} \lVert \vx_t \rVert_{\mV_t^{-1}}^2} \right) \tag{CS, Lemma~\ref{lem:EPCL}} \\
    &\overset{(*)}{\lesssim} \sqrt{\kappa(T)} R L_T \gamma_T(\delta) \sqrt{T \sum_{t \in [T]} \min\left\{ 1, \lVert \vx_t \rVert_{\mV_t^{-1}}^2 \right\}} \\
    &\lesssim R L_T \gamma_T(\delta) \sqrt{d \kappa(T) T \log \left( 1 + \frac{ST}{d K \kappa(T)} \right)} \tag{Lemma~\ref{lem:EPL}},
\end{align*}
where $(*)$ follows from the fact that $\kappa(T)$ scales exponentially in $S$ and thus the first term in the parentheses, which has no dependency on $T$, can be ignored; see Eqn.~(26) of \cite{amani2021mnl}.

Plugging in the definition of $\gamma_T(\delta)$, we have the following regret bound:
\begin{align*}
    \Reg^B(T) &\lesssim L_T R \sqrt{d KS \log\left(e + \frac{ST}{dK}\right) + \sqrt{K} S \log\frac{1}{\delta}} \sqrt{d \kappa(T) T \log\left( 1 + \frac{ST}{d K \kappa(T)} \right)} \\
    &\lesssim L_T R \sqrt{d \sqrt{K} S} \left( \sqrt{d \sqrt{K} \log\left(e + \frac{ST}{dK}\right)} + \sqrt{\log\frac{1}{\delta}} \right) \sqrt{\kappa(T) T \log\left( 1 + \frac{ST}{d K \kappa(T)} \right)}.
\end{align*}

\subsubsection{$\sqrt{T} + \kappa $-type regret -- Algorithm~\ref{alg:multinomial2}}
\label{app:multinomial2}
Here, we follow the proof provided in Appendix D of \cite{amani2021mnl}, with appropriate modifications as done in our logistic bandit regret proof.
To do that, we first recall some notions.

For each $t \in [T]$, $\vx \in \gX_t$, let $\left\{ \bm\Theta_{t, h} \right\}_{h \in [N_t]} \subset \gC_t(\delta) \cap \gB^{K \times d}(S)$ be the set of minimal elements w.r.t. Loewner ordering, i.e., $N_t$ is the number of minimal elements of $\gC_t(\delta) \cap \gB^{K \times d}(S)$.
Also, define
\begin{equation}
    \gM_t(\delta) := \left\{ \bm\Theta \in \sR^{K \times d} : \forall s \in [t - 1] \ \exists i(s) \in [N_s] \ \text{s.t.} \ \mA(\vx_s, \bm\Theta) \succeq \mA(\vx_s, \bm\Theta_{s,i(s)}) \right\},
\end{equation}
and define $\gW_t(\delta) := \gC_t(\delta) \cap \gM_t(\delta)$ to be the new feasible set of estimators.

With similar reasoning as previous, we first have the following:
\begin{lemma}[Improved Lemma 17 of \cite{amani2021mnl}]
\label{lem:17}
    For any $\bm\Theta_1, \bm\Theta_2 \in \gW_t(\delta)$ and any $t \in [T]$, with probability at least $1 - \delta$ we have that
    \begin{equation}
        \bm\mu(\vx, \bm\Theta_1) - \bm\mu(\vx, \bm\Theta_2) \leq \left[ \mA(\vx, \bm\Theta_2) \otimes \vx^\intercal \right] (\bm\theta_1 - \bm\theta_2) + 2 \kappa(T) M_T \gamma_t(\delta)^2 \lVert \vx \rVert^2_{\mV_t^{-1}} \bm1,
    \end{equation}
    where $\leq$ holds elementwise.
\end{lemma}
\begin{proof}
    In their chain of inequalities for their proof of Lemma 17 in their Appendix D~\citep{amani2021mnl}, we alternatively proceed as follows:
    \begin{align*}
        M_T \left\lVert \left[ \mI_K \otimes \vx^\intercal \right] (\bm\theta_1 - \bm\theta_2) \right\rVert_2^2 &\leq M_T \left\lVert \left[ \mI_K \otimes \vx^\intercal \right] \widetilde{\mG}_t(\bm\Theta_1, \bm\Theta_2)^{-1/2} \right\lVert_2^2 \left\lVert \bm\theta_1 - \bm\theta_2 \right\rVert_{\widetilde{\mG}(\bm\Theta_1, \bm\Theta_2)}^2 \tag{CS} \\
        &\leq M_T \left\lVert \left[ \mI_K \otimes \vx^\intercal \right] \widetilde{\mG}_t(\bm\Theta_1, \bm\Theta_2)^{-1/2} \right\lVert_2^2 \gamma_t(\delta)^2 \tag{Lemma~\ref{lem:H-norm-multinomial}} \\
        &\overset{(*)}{\leq} 2\kappa(T) M_T \gamma_t(\delta)^2 \lVert \vx \rVert_{\mV_t^{-1}}^2
    \end{align*}
    where CS refers to Cauchy-Schwartz inequality w.r.t. $\widetilde{\mG}_t$ instead of $\mG_t$, and $(*)$ is when the hidden computations are precisely the same as done in the chain of inequalities in Appendix D of \cite{amani2021mnl}.
    The rest of the proof is then the same.
\end{proof}

\begin{lemma}[Improved Lemma 18 of \cite{amani2021mnl}]
\label{lem:18}
    \begin{equation}
    \label{eqn:improved-bonus-multinomial}
        \Delta(\vx, \bm\Theta_t) \leq \overline{\epsilon}_t(\vx, \bm\Theta_t) := R \sqrt{2 + 2\sqrt{6}S} \gamma_t(\delta) \left\lVert \left[ \mA(\vx, \bm\Theta_t) \otimes \vx^\intercal \right] \mH_t(\bm\Theta_t)^{-1/2} \right\rVert_2 + 2 \kappa(T) M_T \left( \sum_{k=1}^K \rho_k \right) \gamma_t(\delta)^2 \lVert \vx \rVert_{\mV_t^{-1}}^2.
    \end{equation}
\end{lemma}
\begin{proof}
    In their chain of inequalities for their proof of Lemma 18 in their Appendix D~\citep{amani2021mnl}, we alternatively proceed as follows:
    \begin{align*}
        \Delta(\vx, \bm\Theta_t) &\leq R \left\lVert \left[ \mA(\vx, \bm\Theta_t) \otimes \vx^\intercal \right] (\bm\theta_\star - \bm\theta_t) \right\rVert_2 + 2\kappa(T)M_T \left( \sum_{k=1}^K \rho_k \right) \gamma_t(\delta)^2 \lVert \vx \rVert_{\mV_t^{-1}}^2 \\
        &\leq R \left\lVert \left[ \mA(\vx, \bm\Theta_t) \otimes \vx^\intercal \right] \widetilde{\mG}_t(\bm\Theta_\star, \bm\Theta_t)^{-1/2} \right\lVert_2 \left\lVert \bm\theta_\star - \bm\theta_t \right\rVert_{\widetilde{\mG}(\bm\Theta_\star, \bm\Theta_t)} + 2\kappa(T)M_T \left( \sum_{k=1}^K \rho_k \right) \gamma_t(\delta)^2 \lVert \vx \rVert_{\mV_t^{-1}}^2 \tag{CS} \\
        &\leq R \gamma_t(\delta) \left\lVert \left[ \mA(\vx, \bm\Theta) \otimes \vx^\intercal \right] \widetilde{\mG}_t(\bm\Theta_\star, \bm\Theta_t)^{-1/2} \right\lVert_2 + 2\kappa(T)M_T \left( \sum_{k=1}^K \rho_k \right) \gamma_t(\delta)^2 \lVert \vx \rVert_{\mV_t^{-1}}^2 \tag{Lemma~\ref{lem:H-norm-multinomial}} \\
        &\leq R \sqrt{2 + 2\sqrt{6} S} \gamma_t(\delta) \left\lVert \left[ \mA(\vx, \bm\Theta_t) \otimes \vx^\intercal \right] \mH_t(\bm\Theta_t)^{-1/2} \right\lVert_2 + 2\kappa(T)M_T \left( \sum_{k=1}^K \rho_k \right) \gamma_t(\delta)^2 \lVert \vx \rVert_{\mV_t^{-1}}^2, \tag{Lemma~\ref{lem:generalized-self-concordant1}} \\
        &\leq R \sqrt{2 + 2\sqrt{6} S} \gamma_t(\delta) \left\lVert \left[ \mA(\vx, \bm\Theta_t) \otimes \vx^\intercal \right] \mH_t(\bm\Theta_t)^{-1/2} \right\lVert_2 + 2\kappa(T)M_T \gamma_t(\delta)^2 \sqrt{RK} \lVert \vx \rVert_{\mV_t^{-1}}^2,
    \end{align*}
    where CS refers to Cauchy-Schwartz inequality w.r.t. $\widetilde{\mG}_t$ instead of $\mG_t$.
\end{proof}

We now follow through with proof of Theorem 3 of \cite{amani2021mnl} as shown in their Appendix D, with some key differences.
One is that we use Cauchy-Schwartz inequality w.r.t. $\widetilde{\mG}_t$ instead of $\mG_t$, and another is that we utilize the Elliptical Potential Count Lemma-type argument.

By first-order Taylor expansion, we have that $\mA(\vx, \bm\Theta_t) = \mA(\vx, \bm\Theta_{t,h}) + \mU(\vx, \bm\Theta_t, \bm\Theta_{t,h})$, where $\mU(\vx, \bm\Theta_t, \bm\Theta_{t,h}) \in \sR^{K \times K}$ is defined as
\begin{equation}
    \mU(\vx, \bm\Theta_t, \bm\Theta_{t,h})_{ij} := \vx^\intercal (\bm\Theta_t - \bm\Theta_{t,h}) \int_0^1 \nabla \left[\mA \left( \vx, v \bm\Theta_t + (1 - v) \bm\Theta_{t,h} \right)_{ij} \right] dv, \quad \forall i, j \in [K]
\end{equation}

Following a similar line of reasoning as done in our Lemma~\ref{lem:18} and in \cite{amani2021mnl}, we have
\begin{equation*}
    \lambda_{\max}\left( \mU(\vx, \bm\Theta_t, \bm\Theta_{t,h}) \right) \leq M_T' K \sqrt{2\kappa(T)} \gamma_t(\delta) \lVert \vx \rVert_{\mV_t^{-1}}.
\end{equation*}

Thus,
\begin{align*}
    &\left\lVert \left[ \mA(\vx, \bm\Theta_t) \otimes \vx^\intercal \right] \mH_t(\bm\Theta_t)^{-1/2} \right\rVert_2 \\
    &\leq \left\lVert \left[ \mA(\vx, \bm\Theta_{t,h}) \otimes \vx^\intercal \right] \mH_t(\bm\Theta_t)^{-1/2} \right\rVert_2 + \left\lVert \left[ \mU(\vx, \bm\Theta_t, \bm\Theta_{t,h}) \otimes \vx^\intercal \right] \mH_t(\bm\Theta_t)^{-1/2} \right\rVert_2 \\
    &\leq \left\lVert \left[ \mA(\vx, \bm\Theta_{t,h}) \otimes \vx^\intercal \right] \mH_t(\bm\Theta_t)^{-1/2} \right\rVert_2 + \left\lVert \mU(\vx, \bm\Theta_t, \bm\Theta_{t,h}) \right\rVert_2 \left\lVert \left[ \mI_K \otimes \vx^\intercal \right] \mH_t(\bm\Theta_t)^{-1/2} \right\rVert_2 \\
    &\leq \left\lVert \left[ \mA(\vx, \bm\Theta_{t,h}) \otimes \vx^\intercal \right] \mH_t(\bm\Theta_t)^{-1/2} \right\rVert_2 + \sqrt{2\kappa(T)} M_T' K \gamma_t(\delta) \left\lVert \left[ \mI_K \otimes \vx^\intercal \right] \mH_t(\bm\Theta_t)^{-1/2} \right\rVert_2 \lVert \vx \rVert_{\mV_t^{-1}} \\
    &\overset{(*)}{\leq} \left\lVert \left[ \mA(\vx, \bm\Theta_{t,h}) \otimes \vx^\intercal \right] \mH_t(\bm\Theta_t)^{-1/2} \right\rVert_2 + 2 M_T' K \kappa(T) \gamma_t(\delta) \left\lVert \left[ \mI_K \otimes \vx^\intercal \right] (\mI_K \otimes \mV_t)^{-1/2} \right\rVert_2 \lVert \vx \rVert_{\mV_t^{-1}} \\
    &\leq \left\lVert \left[ \mA(\vx, \bm\Theta_{t,h}) \otimes \vx^\intercal \right] \mH_t(\bm\Theta_t)^{-1/2} \right\rVert_2 + 2 M_T' K \kappa(T) \gamma_t(\delta) \lVert \vx \rVert_{\mV_t^{-1}}^2,
\end{align*}
where $(*)$ follows from the fact that $\mI_K \otimes \mV_t \preceq 2 \kappa(T) \mG_t(\bm\Theta_1, \bm\Theta_2)$.

Recall that for each $t \in [T]$ and for each $s \in [t - 1]$, let $i(s) \in [N_s]$ be the index such that $\mA(\vx_s, \bm\Theta_t) \succeq \mA(\vx_s, \bm\Theta_{s,i(s)})$.
By Eqn. (86) of \cite{amani2021mnl}, we have
\begin{equation}
    \mH_t(\bm\Theta_t) \succeq \mL_t := \lambda \mI_{Kd} + \sum_{s=1}^{t-1} \sum_{k=1}^K \tilde{\vx}_{s, k} \tilde{\vx}_{s, k}^\intercal,
\end{equation}
where $\tilde{\vx}_{s,k} := \mA(\vx_s, \bm\Theta_{s,i(s)})_k \otimes \vx_s \in \sR^{Kd \times 1}$ satisfies $\lVert \tilde{\vx}_{s,k} \rVert_2 = \left\lVert \mA(\vx_s, \bm\Theta_{s,i(s)})_k \right\rVert_2 \lVert \vx_s \rVert_2 \leq \lVert \mA(\vx_s, \bm\Theta_{s,i(s)}) \rVert_2 \leq 1$.

We then observe that
\begin{align*}
    \sum_{t=1}^T \left\lVert \left[ \mA(\vx, \bm\Theta_{t,i(t)}) \otimes \vx^\intercal \right] \mH_t(\bm\Theta_t)^{-1/2} \right\rVert_2 &= \sum_{t=1}^T \sqrt{\sum_{k=1}^K \left\lVert \mA(\vx_t, \bm\Theta_{t, i(t)})_k \otimes \vx_t \right\rVert_{\mH_t^{-1}(\bm\theta_\star)}^2 } \\
    &= \sum_{t=1}^T \sqrt{\sum_{k=1}^K \left\lVert \tilde{\vx}_{t, k} \right\rVert_{\mH_t^{-1}(\bm\theta_\star)}^2 } \\
    &\leq \sum_{t=1}^T \sqrt{\sum_{k=1}^K \left\lVert \tilde{\vx}_{t, k} \right\rVert_{\mL_t^{-1}(\bm\theta_\star)}^2 }.
\end{align*}

Distinct from our logistic bandits proof, we extend the previous elliptical lemmas (Lemma~\ref{lem:EPCL} and \ref{lem:EPL}) to more general, ``multinomial'' versions, which we present here:
\begin{lemma}[Generalized Elliptical Potential Count Lemma]
\label{lem:EPCL-generalized}
    Let $\left\{ \vx_{t,k} \right\}_{t \in [T], k \in [K]} \subset \gB^d(1)$ be a sequence of vectors, $\mV_t := \lambda \mI + \sum_{s=1}^{t-1} \sum_{k=1}^K \vx_{s,k} \vx_{s,k}^\intercal$, and let us define the following: $\gH_T := \left\{ t \in [T] : \sum_{k=1}^K \lVert \vx_{t,k} \rVert_{\mV_t^{-1}}^2 > 1 \right\}$.
    Then, we have that
    \begin{equation}
        |\gH_T| \leq \frac{2d}{\log(2)} \log\left( 1 + \frac{K}{\lambda \log(2)} \right).
    \end{equation}
\end{lemma}
\begin{proof}
    Although the proof is similar to \cite{gales2022norm}, there are some subtle differences; we provide the full proof in Appendix~\ref{app:lem-generalized-elliptical1}.
\end{proof}

\begin{lemma}[Generalized Elliptical Potential Lemma]
\label{lem:EPL-generalized}
    Let $\left\{ \vx_{t,k} \right\}_{t \in [T], k \in [K]} \subset \gB^d(1)$ be a sequence of vectors, $\mV_t := \lambda \mI + \sum_{s=1}^{t-1} \sum_{k=1}^K \vx_{s,k} \vx_{s,k}^\intercal$.
    Then, we have that
    \begin{equation}
        \sum_{t=1}^T \min\left\{1, \sum_{k=1}^K \left\lVert \vx_{t, k} \right\rVert_{\mV_t^{-1}}^2 \right\} \leq 2 d \log\left( 1 + \frac{KT}{d \lambda} \right)
    \end{equation}
\end{lemma}
\begin{proof}
    The proof is similar to \cite{abbasiyadkori2011linear}, except we use some matrix determinant inequalities.
    We provide the full proof in Appendix~\ref{app:lem-generalized-elliptical2}.
\end{proof}

Using these new elliptical lemmas, we have:
\begin{align*}
    \sum_{t=1}^T \sqrt{\sum_{k=1}^K \left\lVert \tilde{\vx}_{t, k} \right\rVert_{\mL_t^{-1}(\bm\theta_\star)}^2 } &= \sum_{t \in \gH_T} \sqrt{\sum_{k=1}^K \left\lVert \tilde{\vx}_{t, k} \right\rVert_{\mL_t^{-1}(\bm\theta_\star)}^2 } + \sum_{t \not\in \gH_T} \sqrt{\sum_{k=1}^K \left\lVert \tilde{\vx}_{t, k} \right\rVert_{\mL_t^{-1}(\bm\theta_\star)}^2 } \\
    &\lesssim d S \log (1 + S^2) + 
    \sqrt{ T \sum_{t=1}^T \min\left\{ 1, \sum_{k=1}^K \left\lVert \tilde{\vx}_{t, k} \right\rVert_{\mL_t^{-1}(\bm\theta_\star)}^2  \right\} } \tag{CS, $\lambda = \frac{K}{S^2}$, Lemma~\ref{lem:EPCL-generalized}} \\
    &\lesssim d S \log (1 + S^2) + 
    \sqrt{ d T \log\left( 1 + \frac{S^2 T}{d} \right) } \tag{Lemma~\ref{lem:EPL-generalized}} \\
\end{align*}

\begin{remark}
    Note that by decoupling $S$ and $\sqrt{T}$, we have significantly improved upon \cite{amani2021mnl}, which relies on a matrix determinant-norm lemma~\citep[Lemma 12]{abbasiyadkori2011linear}.
\end{remark}

All in all, we have the following regret bound:
\begin{align*}
    \Reg^B(T) &\lesssim \sum_{t=1}^T \overline{\epsilon}_t(\vx_t, \bm\Theta_t) \\
    &\lesssim \sqrt{RK} \kappa(T) \gamma_T(\delta)^2 \left( M_T + M_T' \sqrt{RKS} \right) \sum_{t=1}^T \lVert \vx_t \rVert_{\mV_t^{-1}}^2
    + R \sqrt{S} \gamma_T(\delta) \sqrt{ d T \log\left( 1 + \frac{S^2 T}{d} \right) } \\
    &\lesssim \sqrt{RK} \kappa(T) \gamma_T(\delta)^2 \left( M_T + M_T' \sqrt{RKS} \right) \left( \frac{dS^2}{K \kappa(T)} \log\left(1 + \frac{S}{K \kappa(T)} \right) + d \log\left(1 + \frac{ST}{dK \kappa(T)} \right) \right) \\
    &\quad + R \sqrt{S} \gamma_T(\delta) \sqrt{ d T \log\left( 1 + \frac{S^2 T}{d} \right) } \tag{Lemma~\ref{lem:EPCL}, \ref{lem:EPL}} \\
    &\overset{(*)}{\lesssim} R \sqrt{d} K^{\frac{1}{4}} S \sqrt{\left( d\sqrt{K} \log\left(e + \frac{ST}{dK} \right) + \log\frac{1}{\delta} \right) \log\left( 1 + \frac{S^2 T}{d} \right) T} \\
    &\quad + R d K^{\frac{3}{2}} S^{\frac{3}{2}} \max\{M_T, M_T'\} \left( d\sqrt{K} \log\left(e + \frac{ST}{dK} \right) + \log\frac{1}{\delta} \right) \log\left(1 + \frac{ST}{dK \kappa(T)} \right) \kappa(T),
\end{align*}
where at $(*)$, we recall that $\kappa(T) = \Theta(e^S)$ (Section 3 of \cite{amani2021mnl}) and thus the first term in the parentheses is ignorable.
\hfill\qedsymbol

\subsection{Proof of Supporting Lemmas}
\subsubsection{Proof of Lemma~\ref{lem:decomposition1-multinomial}}
\label{app:lem-decomposition1}
We overload the notation and let $\ell(\bm\mu) = - y_0 \log \left( 1 - \sum_{k=1}^K \mu_k \right) - \sum_{k=1}^K y_k \log \mu_k$, where $\bm\mu = (\mu_1, \cdots, \mu_K)$.
For simplicity denote $\mu_0(\bm\mu) = \mu_0 = 1 - \sum_{k=1}^K \mu_k$ and $\mu_0^\star = \mu_0(\bm\mu^\star)$.
Then we first have that for $k \neq k' \in [K]$,
\begin{equation*}
    \partial_k \ell(\bm\mu) = \frac{y_0}{\mu_0} - \frac{y_k}{\mu_k}, \quad
    \partial_{kk} \ell(\bm\mu) = \frac{y_0}{\mu_0^2} + \frac{y_k}{\mu_k^2}, \quad
    \partial_{k k'} \ell(\bm\mu) = \frac{y_0}{\mu_0^2}.
\end{equation*}

Let $\alpha$ be multi-index.
By multivariate Taylor's theorem with the integral form of the remainder,
\begin{align*}
    \ell(\bm\mu) - \ell(\bm\mu^\star) &= \nabla \ell(\bm\mu^\star)^\intercal (\bm\mu - \bm\mu^\star) + 2 \sum_{|\alpha| = 2} \frac{(\bm\mu - \bm\mu^\star)^\alpha}{\alpha!} \int_0^1 (1 - t) \partial^\alpha \ell(\bm\mu^\star + t(\bm\mu - \bm\mu^\star)) dt \\
    &= \nabla \ell(\bm\mu^\star)^\intercal (\bm\mu - \bm\mu^\star)
    + \sum_{k=1}^K (\mu_k - \mu^\star_k)^2 \int_0^1 (1 - t) \left\{ \frac{y_0}{(\mu^\star_0 + t(\mu_0 - \mu^\star_0))^2} + \frac{y_k}{(\mu^\star_k + t(\mu_k - \mu^\star_k))^2} \right\} dt \\
    &\quad + 2 \sum_{1 \leq k < k' \leq K} (\mu_k - \mu^\star_k) (\mu_{k'} - \mu^\star_{k'}) \int_0^1 (1 - t) \frac{y_0}{(\mu^\star_0 + t(\mu_0 - \mu^\star_0))^2} dt \\
    &= \nabla \ell(\bm\mu^\star)^\intercal (\bm\mu - \bm\mu^\star) + \sum_{k=1}^K (\mu_k - \mu^\star_k)^2 \int_0^1 (1 - t) \frac{y_k}{(\mu^\star_k + t(\mu_k - \mu^\star_k))^2} dt \\
    &\quad + \left( \sum_{k=1}^K (\mu_k - \mu^\star_k) \right)^2 \int_0^1 (1 - t) \frac{y_0}{(\mu^\star_0 + t(\mu_0 - \mu^\star_0))^2} dt \\
    &= \underbrace{\nabla \ell(\bm\mu^\star)^\intercal (\bm\mu - \bm\mu^\star)}_{(a)}
    + \underbrace{\sum_{k=0}^K (\mu_k - \mu^\star_k)^2 \int_0^1 (1 - t) \frac{y_k}{(\mu^\star_k + t(\mu_k - \mu^\star_k))^2} dt}_{(b)}.
\end{align*}

\paragraph{(a)}
\begin{align*}
    \nabla \ell(\bm\mu^\star)^\intercal (\bm\mu - \bm\mu^\star) &= \sum_{k=1}^K \left( \frac{y_0}{\mu_0^\star} - \frac{y_k}{\mu_k^\star} \right) (\mu_k - \mu^\star_k) \\
    &= \sum_{k=1}^K \left( \frac{y_0}{\mu_0^\star} (\mu_k - \mu_k^\star) - \frac{y_k}{\mu_k^\star}\mu_k + y_k \right).
\end{align*}

\paragraph{(b)}
\begin{align*}
    \sum_{k=0}^K (\mu_k - \mu^\star_k)^2 \int_0^1 (1 - t) \frac{y_k}{(\mu^\star_k + t(\mu_k - \mu^\star_k))^2} dt &= \sum_{k=0}^K (\mu_k - \mu^\star_k)^2 \int_{\mu_k^\star}^{\mu_k} \left( 1 - \frac{v - \mu^\star_k}{\mu_k - \mu^\star_k} \right) \frac{y_k}{v^2} \frac{1}{\mu_k - \mu^\star_k} dv \\
    &= \sum_{k=0}^K y_k \int_{\mu_k^\star}^{\mu_k} \frac{\mu_k - v}{v^2} dv \\
    &= \sum_{k=0}^K y_k \left\{ \frac{\mu_k}{\mu_k^\star} - 1 - \log\frac{\mu_k}{\mu_k^\star} \right\}.
\end{align*}

Recall that $\sum_{k=0}^K y_k = \sum_{k=0}^K \mu_k = \sum_{k=0}^K \mu_k^\star = 1$ and $y_k = \mu_k^\star + \xi_k$ for $k \in [K]$.
Denoting $\xi_0 = - \sum_{k=1}^K \xi_k$, we then also have that $y_0 = \mu_0^\star + \xi_0$.
Then, we have that 
\begin{align*}
    \ell(\bm\mu) - \ell(\bm\mu^\star) &= y_0 \left\{ \frac{\mu_0}{\mu_0^\star} - 1 - \log\frac{\mu_0}{\mu_0^\star} \right\} + \sum_{k=1}^K \left\{ \frac{y_0}{\mu_0^\star} (\mu_k - \mu_k^\star) - y_k \log\frac{\mu_k}{\mu_k^\star} \right\} \\
    &= \frac{y_0}{\mu_0^\star} \sum_{k=0}^K \mu_k - y_0 + y_0\log\frac{\mu_0^\star}{\mu_0} + \sum_{k=1}^K \left\{ - \frac{y_0}{\mu_0^\star} \mu_k^\star + y_k \log\frac{\mu_k^\star}{\mu_k} \right\} \\
    &= \frac{y_0}{\mu_0^\star} - \frac{y_0}{\mu_0^\star} \sum_{k=1}^K \mu_k^\star - y_0
    + \sum_{k=0}^K y_k \log\frac{\mu_k^\star}{\mu_k} \\
    &= \sum_{k=0}^K \mu_k^\star \log\frac{\mu_k^\star}{\mu_k} + \sum_{k=0}^K \xi_k \log\frac{\mu_k^\star}{\mu_k} \\
    &= \sum_{k=0}^K \mu_k^\star \log\frac{\mu_k^\star}{\mu_k} + \sum_{k=0}^K \xi_k \log\frac{\mu_k^\star}{\mu_k} \\
    &= \KL(\bm\mu^\star, \bm\mu) + \sum_{k=1}^K \xi_k \left( \log\frac{\mu_k^\star}{\mu_0^\star} - \log\frac{\mu_k}{\mu_0} \right) \\
    &\overset{(*)}{=} \KL(\bm\mu^\star, \bm\mu) + \sum_{k=1}^K \xi_k \langle \vx_t, \bm\theta_\star^{(k)} - \bm\theta_t^{(k)} \rangle,
\end{align*}
where at $(*)$, we recall the definitions of $\bm\mu^\star$ and $\bm\mu$.
Then, with proper matrix notations, the statement follows.
\hfill\qedsymbol

\subsubsection{Proof of Lemma~\ref{lem:kl-bregman-multinomial}}
\label{app:lem-bregman}
Denote $\mu_k^{(i)} = \mu_k(\vz^{(i)})$ and $C_k^{(i)} := 1 + \sum_{j \neq k} e^{z_j^{(i)}}$.
Then we have the following conversion between $\mu, C$, and $z$:
\begin{equation*}
    \mu_k^{(i)} = \frac{e^{z_k^{(i)}}}{C_k^{(i)} + e^{z_k^{(i)}}}, \quad z_k^{(i)} = \frac{\mu_k^{(i)} C_k^{(i)}}{1 - \mu_k^{(i)}}.
\end{equation*}
The statement then follows from direct computation:
\begin{align*}
    &D_m(\vz^{(1)}, \vz^{(2)}) \\
    &= m(\vz^{(1)}) - m(\vz^{(2)}) - \nabla m(\vz^{(2)})^\intercal (\vz^{(1)} - \vz^{(2)}) \\
    &= \log\left( 1 + \sum_{k=1}^K e^{z^{(1)}_k} \right) - \log\left( 1 + \sum_{k=1}^K e^{z^{(2)}_k} \right) - \sum_{k=1}^K \frac{e^{z^{(2)}_k}}{1 + \sum_{k=1}^K e^{z^{(2)}_k}} (z^{(1)}_k - z^{(2)}_k) \\
    &= \log\frac{1 - \sum_{k=1}^K \mu_k^{(2)}}{1 - \sum_{k=1}^K \mu_k^{(1)}} - \sum_{k=1}^K \mu_k^{(2)} \log\frac{\mu_k^{(1)} (1 - \mu_k^{(2)}) C_k^{(1)}}{\mu_k^{(2)}(1 - \mu_k^{(1)}) C_k^{(2)}} \\
    &= \left( 1 - \sum_{k=1}^K \mu_k^{(2)} \right) \log\frac{1 - \sum_{k=1}^K \mu_k^{(2)}}{1 - \sum_{k=1}^K \mu_k^{(1)}} + \sum_{k=1}^K \mu_k^{(2)} \log\frac{\mu_k^{(2)}}{\mu_k^{(1)}}
    + \sum_{k=1}^K \mu_k^{(2)} \left\{ \log\frac{1 - \sum_{j=1}^K \mu_j^{(2)}}{1 - \sum_{j=1}^K \mu_j^{(1)}} - \log\frac{(1 - \mu_k^{(2)}) C_k^{(1)}}{(1 - \mu_k^{(1)}) C_k^{(2)}} \right\} \\
    &= \KL(\bm\mu(\vz^{(2)}), \bm\mu(\vz^{(1)})) + \sum_{k=1}^K \mu_k^{(2)} \left\{ \log\frac{\sum_{j=1}^K e^{z_j^{(1)}}}{\sum_{j=1}^K e^{z_j^{(2)}}} - \log\frac{C^{(1)}_k + e^{z_k^{(1)}}}{C^{(2)}_k + e^{z_k^{(2)}}} \right\} \\
    &= \KL(\bm\mu(\vz^{(2)}), \bm\mu(\vz^{(1)})).
\end{align*}
\hfill\qedsymbol

\subsubsection{Proof of Lemma~\ref{lem:generalized-self-concordant1}}
\label{app:lem-generalized}
By Proposition 8 of \cite{sun2019generalized}, we have that for any $\vz_1, \vz_2$,
\begin{equation*}
    \nabla^2 f(\vz_1 + v(\vz_2 - \vz_1)) \succeq e^{-M_f \lVert \vz_1 - \vz_2 \rVert_2 v} \nabla^2 f(\vz_1).
\end{equation*}
Multiplying both sides by $(1 - v)$ and integrating over $[0, 1]$ w.r.t. $v$, the statement follows:
\begin{align*}
    \int_0^1 (1 - v) \nabla^2 f(\vz_1 + v(\vz_2 - \vz_1)) dv &\succeq \int_0^1 (1 - v) e^{- M_f \lVert \vz_1 - \vz_2 \rVert_2 v} \nabla^2 f(\vz_1) dv \\
    &= \left( \frac{1}{M_f \lVert \vz_1 - \vz_2 \rVert_2} + \frac{\exp(-M_f \lVert \vz_1 - \vz_2 \rVert_2) - 1}{(M_f \lVert \vz_1 - \vz_2 \rVert_2)^2} \right) \nabla^2 f(\vz_1) \\
    &\succeq \frac{1}{2 + M_f \lVert \vz_1 - \vz_2 \rVert_2} \nabla^2 f(\vz_1),
\end{align*}
where the last inequality follows from the elementary inequality $\frac{1}{z} + \frac{e^{-z} - 1}{z^2} \geq \frac{1}{2 + z}$ for any $z \geq 0$.
\hfill\qedsymbol

\subsubsection{Proof of Lemma~\ref{lem:H-norm-multinomial}}
\label{app:H-norm-multinomial}

By Theorem~\ref{thm:confidence-multinomial}, we have that with probability at least $1 - \delta$, $\gL_t(\bm\Theta_\star) - \gL_t(\widehat{\bm\Theta}_t) \leq \beta_t(\delta)^2$, which we assume to be true throughout the proof.
Let $\bm\Theta \in \gC_t(\delta)$, and recall that $\bm\theta = \mathrm{vec}(\bm\Theta^\intercal)$.
Then, we first have that via second-order Taylor expansion of $\gL_t(\bm\theta)$ around $\bm\theta_\star$,
\begin{align}
    \lVert \bm\theta - \bm\theta_\star \rVert_{\widetilde{\mG}_t(\bm\theta_\star, \bm\theta)}^2 &= \gL_t(\bm\theta) - \gL_t(\bm\theta_\star) + \nabla \gL_t(\bm\theta_\star)^\intercal (\bm\theta_\star - \bm\theta) + \lambda \lVert \bm\theta - \bm\theta_\star \rVert_2^2 \nonumber \\
    &\leq \gL_t(\bm\theta) - \gL_t(\widehat{\bm\theta}_t) + \nabla \gL_t(\bm\theta_\star)^\intercal (\bm\theta_\star - \bm\theta) + \lambda \lVert \bm\theta - \bm\theta_\star \rVert_2^2 \nonumber \\
    &\leq K + \beta_t(\delta)^2 + \nabla \gL_t(\bm\theta_\star)^\intercal (\bm\theta_\star - \bm\theta), \quad \text{w.p. at least $1 - \delta$} \label{eqn:norm-H-multinomial},
\end{align}
where we chose $\lambda = \frac{K}{4S^2}$.

Now observe that
\begin{equation*}
    \nabla \gL_t(\bm\theta_\star)^\intercal \vv = \sum_{s=1}^t \left[ \left( \bm\mu(\vx_s, \bm\Theta_\star) - \vy_s \right) \otimes \vx_s \right]^\intercal \vv
    = \sum_{s=1}^t \bm\xi_s^\intercal \vectorize^{-1}(\vv) \vx_s
\end{equation*}
where $\mathrm{vec}^{-1}$ is the matricization operator, and we overload the notation and define $\bm\xi_s := \bm\mu(\vx_s, \bm\Theta_\star) - \vy_s$.

Let $\gB^{dK}(2S)$ be a $dK$-ball of radius $2S$, and $\vv \in \gB^{dK}(2S)$.
It can be easily checked that $\bm\xi_s^\intercal \mathrm{vec}^{-1}(\vv) \vx_s$ is also a martingale difference sequence that satisfies
\begin{align*}
    \left| \bm\xi_s^\intercal \left( \mathrm{vec}^{-1}(\vv) \vx_s \right) \right| &\leq 2S, \\
    \E\left[ \left( \bm\xi_s^\intercal \left( \mathrm{vec}^{-1}(\vv) \vx_s \right) \right)^2 \Big| \gF_{s-1} \right] &= \lVert \vectorize^{-1}(\vv) \vx_s \rVert_{\mA_\star(\vx_s)}^2.
\end{align*}
where for simplicity we denote $\mA_\star(\vx_s) := \mA(\vx_s, \bm\Theta_\star)$.
Thus, by Freedman's inequality (Lemma~\ref{lem:freedman}), for any $\eta \in \left[ 0, \frac{1}{2S} \right]$, the following holds:
\begin{equation}
    \sP\left[ \sum_{s=1}^t \bm\xi_s^\intercal \left( \mathrm{vec}^{-1}(\vv) \vx_s \right) \leq (e - 2) \eta \sum_{s=1}^t \lVert \vectorize^{-1}(\vv) \vx_s \rVert_{\mA_\star(\vx_s)}^2 + \frac{1}{\eta} \log\frac{1}{\delta} \right] \geq 1 - \delta.
\end{equation}
Then, via similar reasoning ($\eps$-net and union bound) as in the proof of Lemma~\ref{lem:H-norm}, we have the following:
for $\vv_t$ s.t. $\lVert \vv_t \rVert_2 \leq 2S$ and $\lVert (\bm\theta_\star - \bm\theta) - \vv_t \rVert_2 \leq \eps_t$,
\begin{align*}
    &\nabla \gL_t(\bm\theta_\star)^\intercal (\bm\theta_\star - \bm\theta) \\
    &= \sum_{s=1}^t \bm\xi_s^\intercal \left( \mathrm{vec}^{-1}(\vv_t) \vx_s \right) + \sum_{s=1}^t \bm\xi_s^\intercal \left( \mathrm{vec}^{-1}((\bm\theta_\star - \bm\theta) - \vv_t) \vx_s \right) \tag{linearity of $\vectorize^{-1}$} \\
    &\leq (e - 2)\eta \sum_{s=1}^t \lVert \vectorize^{-1}(\vv_t) \vx_s \rVert_{\mA_\star(\vx_s)}^2 + \frac{dK}{\eta} \log\frac{5S}{\eps_t} + \frac{1}{\eta} \log\frac{1}{\delta} + \eps_t t \tag{w.p. at least $1 - \delta$} \\
    &= (e - 2)\eta \left\{ \sum_{s=1}^t \lVert \vectorize^{-1}(\bm\theta_\star - \bm\theta) \vx_s \rVert_{\mA_\star(\vx_s)}^2 + \sum_{s=1}^t \left( \lVert \vectorize^{-1}(\vv_t) \vx_s \rVert_{\mA_\star(\vx_s)}^2 - \lVert \vectorize^{-1}\left( \bm\theta_\star - \bm\theta \right) \vx_s \rVert_{\mA_\star(\vx_s)}^2 \right) \right\} \\
    &\quad + \frac{dK}{\eta} \log\frac{5S}{\eps_t} + \frac{1}{\eta} \log\frac{1}{\delta} + \eps_t t \\
    &\overset{(*)}{\leq} (e - 2)\eta \sum_{s=1}^t \left\lVert (\bm\Theta_\star - \bm\Theta) \vx_s \right\rVert_{\mA_\star(\vx_s)}^2 + (e - 2)\eta L\left( 4S + \eps_t \right) \eps_t t + \frac{dK}{\eta} \log\frac{5S}{\eps_t} + \frac{1}{\eta} \log\frac{1}{\delta} + \eps_t t \\
    &\overset{(**)}{=} (e - 2)\eta \lVert \bm\theta_\star - \bm\theta \rVert_{\mH_t(\bm\theta_\star)}^2 + \frac{dK}{\eta} \log\frac{5S}{\eps_t} + \frac{1}{\eta} \log\frac{1}{\delta} + \left( (e-2)L \left( 4S\eta + \eps_t \eta \right) + 1 \right) \eps_t t \\
    &\leq (e - 2) (2 + 2\sqrt{6}S) \eta \lVert \bm\theta_\star - \bm\theta \rVert_{\widetilde{\mG}_t(\bm\theta_\star, \bm\theta)}^2 + \frac{dK}{\eta} \log\frac{5S}{\eps_t} + \frac{1}{\eta} \log\frac{1}{\delta} + \left( (e-2)L \left( 4S\eta + \eps_t \eta \right) + 1 \right) \eps_t t \tag{$\mH_t(\bm\theta_\star) \preceq (2 + 2\sqrt{6}S) \widetilde{\mG}_t(\bm\theta_\star, \bm\theta)$}, 
\end{align*}
where $(*)$ follows from the observation that
\begin{align*}
    \lVert \mC \vx_s \rVert_{\mA_\star(\vx_s)}^2 - \lVert \mD \vx_s \rVert_{\mA_\star(\vx_s)}^2 &= \lVert \mD \vx_s + (\mC - \mD) \vx_s \rVert_{\mA_\star(\vx_s)}^2 - \lVert \mD \vx_s \rVert_{\mA_\star(\vx_s)}^2 \\
    &= 2 \vx_s^\intercal \mD^\intercal \mA_\star(\vx_s) (\mC - \mD) \vx_s + \vx_s^\intercal (\mC - \mD)^\intercal \mA_\star(\vx_s) (\mC - \mD) \vx_s \\
    &\leq 2 \lVert \mD^\intercal \mA_\star(\vx_s) (\mC - \mD) \vx_s \rVert_2 + L \eps_t^2 \tag{Definition of $L$ (Eqn.~\eqref{eqn:L})} \\
    &\leq 2 \lVert \mD^\intercal \rVert_2 \lVert \mA_\star(\vx_s) \rVert_2 \lVert (\mC - \mD) \rVert_2 + L \eps_t^2 \\
    &\leq 2L \lVert \mD^\intercal \rVert_F \lVert (\mC - \mD) \rVert_F + L \eps_t^2 \tag{Definition of $L$ (Eqn.~\eqref{eqn:L})} \\
    &\leq L\left( 4S + \eps_t \right) \eps_t
\end{align*}
for any $\mC, \mD \in \sR^{d \times K}$ with $\lVert \mC \rVert_F, \lVert \mD \rVert_F \leq 2S$ and $\lVert \mC - \mD \rVert_F \leq \eps_t$.
$(**)$ follows from the observation that for $\bm\theta = \vectorize(\bm\Theta^\intercal)$,
\begin{align*}
    \bm\theta^\intercal (\mA \otimes \vx\vx^\intercal) \bm\theta &= \vectorize(\bm\Theta^\intercal)^\intercal (\mA \otimes \vx\vx^\intercal) \vectorize(\bm\Theta^\intercal) \\
    &\overset{(a)}{=} \vectorize(\bm\Theta^\intercal)^\intercal \vectorize\left( \vx \vx^\intercal \bm\Theta^\intercal \mA^\intercal \right) \\
    &\overset{(a)}{=} \vectorize(\bm\Theta^\intercal)^\intercal \left( \mA \bm\Theta \otimes \vx \right) \vx \\
    &\overset{(b)}{=} \vx^\intercal  \left( \bm\Theta^\intercal \mA^\intercal \otimes \vx^\intercal \right) \vectorize(\bm\Theta^\intercal) \\
    &\overset{(a)}{=} \vx^\intercal \vectorize(\vx^\intercal \bm\Theta^\intercal \mA \bm\Theta) \\
    &= \vx^\intercal \bm\Theta^\intercal \mA \bm\Theta \vx,
\end{align*}
where $(a)$ follows from the mixed Kronecker matrix-vector product property, $(\mC \otimes \mD) \vectorize(\mE) = \vectorize(\mD \mE \mC^\intercal)$, and $(b)$ follows from the tranpose property of the Kronecker product, $(\mC \otimes \mD)^\intercal = \mC^\intercal \otimes \mD^\intercal$.

Choosing $\eta = \frac{1}{2 (e - 2) (2 + 2\sqrt{6}S)} < \frac{1}{2S}$, $\eps_t = \frac{dK}{t}$, and rearranging Eqn.~\eqref{eqn:norm-H-multinomial} with Theorem~\ref{thm:confidence-multinomial}, we finally have that
\begin{align*}
    \lVert \bm\theta - \bm\theta_\star \rVert_{\widetilde{\mG}_t(\bm\theta_\star, \bm\theta)}^2 \lesssim d K S \log\left(e + \frac{St}{dK}\right) + \sqrt{K}S \log\frac{1}{\delta} + dKL.
\end{align*}
\hfill\qedsymbol

\subsubsection{Proof of Lemma~\ref{lem:EPL-generalized}}
\label{app:lem-generalized-elliptical1}

We follow the proof of the usual elliptical potential lemma as provided in \cite{abbasiyadkori2011linear}:
\begin{align*}
    \det(\mV_{t+1}) &= \det\left( \mV_t + \sum_{k=1}^K \vx_{t, k} \vx_{t, k}^\intercal \right) \\
    &= \det(\mV_t) \det\left( \mI + \sum_{k=1}^K \mV_t^{-\frac{1}{2}}\vx_{t, k} \left( \mV_t^{-\frac{1}{2}}\vx_{t, k} \right)^\intercal \right) \\
    &\overset{(*)}{\geq} \det(\mV_t) \left( 1 + \sum_{k=1}^K \left\lVert \vx_{t,k} \right\rVert_{\mV_t^{-1}}^2 \right),
\end{align*}
where $(*)$ follows from the following lemma:
\begin{lemma}
\label{lem:matrix}
    For $\mA = [\va_1 \cdots \va_K]$, $\det(\mI + \mA \mA^\intercal) \geq 1 + \tr(\mA \mA^\intercal) = 1 + \sum_{k=1}^K \lVert \va_k \rVert^2$.
\end{lemma}

Taking the log on both sides and using the inequality $2\log(1 + z) \geq z$ for $z \in [0, 1]$,
\begin{align*}
    \sum_{t=1}^T \min\left\{1, \sum_{k=1}^K \left\lVert \vx_{t, k} \right\rVert_{\mV_t^{-1}(\bm\theta_\star)}^2 \right\} &\leq 2 \sum_{t=1}^T \log\left( 1 + \sum_{k=1}^K \left\lVert \vx_{t, k} \right\rVert_{\mV_t^{-1}(\bm\theta_\star)}^2 \right) \\
    &\leq 2 \left( \log\det(\mV_T) - d \log\lambda \right) \\
    &\overset{(*)}{\leq} 2 d \log\left( 1 + \frac{KT}{d \lambda} \right),
\end{align*}
where $(*)$ follows from the fact that for $\lVert \vx_{t,k} \rVert \leq 1$, $\det(\mV_T) \leq \left( \frac{d\lambda + KT}{d} \right)^d$ by AM-GM inequality.

We conclude by proving Lemma~\ref{lem:matrix}: let the eigenvalues of $\mA \mA^\intercal$ be $\lambda_1 \geq \cdots \geq \lambda_m \geq 0$ with $m = \min\{d, K\}$.
Then, we have that
\begin{equation*}
    \det(\mI + \mA \mA^\intercal) = \prod_{k=1}^m (1 + \lambda_k) \geq 1 + \sum_{k=1}^m \lambda_k
    = 1 + \tr(\mA \mA^\intercal)
    = 1 + \tr\left( \sum_{k=1}^K \va_k \va_k^\intercal \right)
    = 1 + \sum_{k=1}^K \lVert \va_k \rVert^2.
\end{equation*}
\hfill\qedsymbol

\subsubsection{Proof of Lemma~\ref{lem:EPCL-generalized}}
\label{app:lem-generalized-elliptical2}

We follow the proof of the elliptical potential count lemma as provided in \cite{gales2022norm}.

Let $\mM_T := \lambda \mI + \sum_{t \in \gH_T} \sum_{k=1}^K \vx_{t,k} \vx_{t,k}^\intercal$, and let $\lambda_1 \geq \cdots \geq \lambda_d \geq 0$ be the eigenvalues of $\sum_{t \in \gH_T} \sum_{k=1}^K \vx_{t,k} \vx_{t,k}^\intercal$.
We first have that
\begin{align*}
    \det(\mM_T) &= \prod_{i=1}^d (\lambda + \lambda_i) \\
    &\leq \left( \sum_{i=1}^d \frac{\lambda + \lambda_i}{d} \right)^d \tag{AM-GM inequality} \\
    &\leq \left( \lambda + \frac{1}{d} \tr\left( \sum_{t \in \gH_T} \sum_{k=1}^K \vx_{t,k} \vx_{t,k}^\intercal \right) \right)^d \\
    &\leq \left( \lambda + \frac{K |\gH_T|}{d} \right)^d.
\end{align*}
Next, from the proof of our generalized elliptical potential lemma, we have that
\begin{equation*}
    \det(\mM_T) \geq \lambda^d \prod_{t \in \gH_T} \left( 1 + \sum_{k=1}^K \lVert \vx_{t,k} \rVert_{\mM_t^{-1}}^2 \right)
    \geq \lambda^d \prod_{t \in \gH_T} \left( 1 + \sum_{k=1}^K \lVert \vx_{t,k} \rVert_{\mV_t^{-1}}^2 \right)
    \geq \lambda^d 2^{|\gH_T|}.
\end{equation*}

Combining the two, we have
\begin{equation*}
    |\gH_T| \leq \frac{d}{\log(2)} \log\left( 1 + \frac{K |\gH_T|}{\lambda d} \right).
\end{equation*}

From here, we are done with the same algebraic computations as done in \cite{gales2022norm} using their Lemma 8.

\hfill\qedsymbol

\end{document}